\documentclass{article}

\usepackage{microtype}
\usepackage{graphicx}
\usepackage{subcaption}
\usepackage{booktabs}

\usepackage{hyperref}

\usepackage[preprint]{icml2026}

\usepackage{amsmath}
\usepackage{amssymb}
\usepackage{mathtools}
\usepackage{amsthm}

\usepackage[capitalize,noabbrev]{cleveref}

\usepackage{amssymb, amsthm, amsmath, graphicx, bm, paralist, pdfpages, appendix}
\usepackage{hyperref}
\usepackage{xcolor}
\usepackage{graphicx}

\newcommand{\target}{\textcolor{cyan!30}{\raisebox{-0.2ex}{\scalebox{1.5}{$\bullet$}}}}
\newcommand{\robot}{\textcolor{red}{\raisebox{0.05ex}{\scalebox{1.15}{$\blacktriangle$}}}}

\def\ve{\varepsilon}

\def\mc{\mathcal}
\def\mb{\mathbf}
\def\mbb{\mathbb}
\def\ra{\rightarrow}
\def\bmu{\boldsymbol{\mu}}
\def\bnu{\boldsymbol{\nu}}
\def\bpi{\boldsymbol{\pi}}

\def\argmin{\mathrm{argmin}}
\def\diag{\mathrm{diag}}

\makeatletter
\renewcommand{\@fnsymbol}[1]{}  
\makeatother

\newcommand{\mrm}[1]{{\mathrm#1}}

\theoremstyle{plain}
\newtheorem{theorem}{Theorem}[section]

\newtheorem{lem}[theorem]{Lemma}

\theoremstyle{definition}

\newtheorem{assumption}[theorem]{Assumption}
\theoremstyle{remark}

\usepackage[textsize=tiny]{todonotes}

\icmltitlerunning{Optimal and Scalable MAPF via Multi-Marginal Optimal Transport and Schr\"{o}dinger Bridges}

\begin{document}

\twocolumn[
  \icmltitle{Optimal and Scalable MAPF via\\ Multi-Marginal Optimal Transport and Schr\"{o}dinger Bridges}

  \icmlsetsymbol{equal}{*}

  \begin{icmlauthorlist}
    \icmlauthor{Usman A. Khan}{yyy,sch,thanks}
    \icmlauthor{Joseph W. Durham}{yyy}
  \end{icmlauthorlist}

  \icmlaffiliation{yyy}{Amazon Robotics, Boston, MA, USA.}
  \icmlaffiliation{sch}{Boston College, Chestnut Hill, MA, USA.}
  \icmlaffiliation{thanks}{UAK holds concurrent appointments as an Amazon Scholar with Amazon Robotics and as a Professor of Computer Science at Boston College}

  \icmlcorrespondingauthor{Usman A. Khan}{uakhan@amazon.com, usman.khan@bc.edu}

  \icmlkeywords{Optimal Transport, Sinkhorn Algorithm}

  \vskip 0.3in
]

\printAffiliationsAndNotice{}  

\begin{abstract}

We consider anonymous multi-agent path finding (MAPF) where a set of robots is tasked to travel to a set of targets on a finite, connected graph. We show that MAPF can be cast as a special class of multi-marginal optimal transport (MMOT) problems with an underlying Markovian structure, under which the exponentially large MMOT collapses to a linear program (LP) polynomial in size. Focusing on the anonymous setting, we establish conditions under which the corresponding LP is feasible, totally unimodular, and consequently, yields min-cost, integral~$(\{0,1\})$ transports that do not overlap in \textit{both} space and time. To adapt the approach to large-scale problems, we cast the MAPF-MMOT in a probabilistic framework via Schrödinger bridges. Under standard assumptions, we show that the Schrödinger bridge formulation reduces to an entropic regularization of the corresponding MMOT that admits an iterative Sinkhorn-type solution. The Schrödinger bridge, being a probabilistic framework, provides a shadow (fractional) transport that we use as a template to solve a reduced LP and demonstrate that it results in near-optimal, integral transports at a significant reduction in complexity. Extensive experiments highlight the optimality and scalability of the proposed approaches.
\end{abstract}

\section{Introduction}
Coordinating large teams of robots to reach target locations while avoiding collisions in space and time is a fundamental problem in robotics and automation. In multi-agent path finding (MAPF), robots are assigned to targets on a shared graph and must compute collision-free trajectories that are jointly optimal in both space and time. This coupling of assignment, path planning, and scheduling renders MAPF combinatorial in nature. In this paper, we show that MAPF can be cast as multi-marginal optimal transport (MMOT) over path spaces with an underlying Markovian structure. Focusing on the anonymous setting, where any robot may reach any target, we show that the resulting Markovian MMOT admits a polynomial-size linear program (LP) with strong optimality and integrality guarantees (extensions to non-anonymous settings are possible via more general MMOT formulations
). We then use ideas from the Schr\"odinger bridge framework and develop an entropic regularization of the corresponding MMOT to build scalable, probabilistic relaxations of the MAPF problem. The main contributions of this paper are as follows.
\begin{itemize}
    \item We show that MAPF is a special class of Markovian MMOT, which admits a polynomial-size LP with a totally unimodular constraint matrix in the anonymous setting; subsequently, all extreme points of the feasible polyhedron are integral~$(\{0,1\})$. We further derive the conditions under which the proposed min-cost LP yields min-move and min-makespan transports.

    \item We connect the MAPF-MMOT to Schr\"odinger bridges, which enable transports with desirable structural properties through appropriate reference distributions. By choosing the reference as a Gibbs kernel, the Schr\"odinger bridge reduces to an entropic regularization of the MAPF-MMOT enabling Sinkhorn-type fast iterations. The resulting Sinkhorn-MAPF provides a \emph{shadow transport}, a set of likely paths from robots to targets. As the Gibbs parameter~${\varepsilon \to 0}$, the Schr\"odinger bridge associated with a highly volatile reference Markov chain concentrates onto tight, min-cost geodesic corridors of the underlying graph.
    
    \item We use the shadow transport to guide a principled pruning of the underlying graph, resulting in an LP defined on a substantially reduced graph while preserving the totally unimodular structure of the feasible polyhedron. Because the Schrödinger bridge leads to a shadow cast on highly likely paths, the resulting pruned formulation is both scalable and integral. Experiments show that as problem size grows, the resulting LP incurs less than~$10\%$ cost degradation while eliminating approximately~$60–80\%$ of the edges of the original graph.
\end{itemize}

\textbf{Related Work: }A comprehensive survey on MAPF can be found in~\cite{Stern2019}; see also~\cite{Ma2016,Peng2023,Zain2023,Fine2023} for some recent results on anonymous MAPF. Representative algorithmic baselines for optimal MAPF include Conflict-Based Search, see e.g.,~\cite{Sharon2015,Felner2017}, as well as SAT-based formulations~\cite{Surynek2015}. Prior work via time-expanded flow networks and integer programs can be found in~\cite{YuLaValle2013NetworkFlow,Ma2020Thesis}; however, this line of work does not explicitly characterize the integrality of MAPF solutions as a polyhedral property of the associated LP. In contrast, our MAPF–MMOT formulation establishes integrality of the resulting LP from first principles via total unimodularity~\cite{schrijver1998,AhujaMagnantiOrlin1993}. The proposed MMOT viewpoint is novel and further enables extensions to probabilistic formulations. In particular, we generalize MAPF-MMOT to the Schrödinger bridge problem where the goal is to find a robot distribution~$\mb P$ over path spaces that is close to a given reference~$\mb G$. Key references in this direction include:~\cite{leonard2014survey}, which surveys the Schrödinger problem and its connections to OT;~\cite{pavonTicozzi2010discrete} on discrete-time Markovian bridges; and~\cite{HRCK2021} on tree-structured costs. MMOT and its computational aspects are discussed in~\cite{Pass2014,LinHoCuturiJordan2022,Haasler2021}. Related work on multi-marginal Sinkhorn methods can be found in~\cite{BCCNP2015,DiMarinoGerolin2020,Carlier2022}. Of relevance is also the body of work in~\cite{Chen2016,Chen2017,CGP2021} that casts Schrödinger bridges within a stochastic control framework.

\section{Background and Problem Formulation}\label{sec:pf}
In this section, we provide some preliminaries and formally write the multi-agent path finding (MAPF) problem.

\textbf{Multi-Marginal OT:} Multi-marginal optimal transport (MMOT) extends the classical two-marginal OT problem to multiple marginals. The transport here is not therefore between two measures but a coupling across more than two measures. Formally, consider a joint probability tensor~${\mb P\in\mbb R^{K\times \ldots \times K}}$, with~${T+1}$ modes, each with~$K$ supports, and a transport cost tensor~$\mb C$ of the same dimensions. Given~$T+1$ marginals~${\mb q_0,\mb q_1, \ldots, \mb  q_T}$, the MMOT problem is to find a min-cost transport plan (coupling), i.e.,
\begin{align*}
    &\min_{\mb P\geq 0} \langle \mb {P, C}\rangle \quad
    \text{subject to} \quad{\phi_t(\mb P)=\mb q_t},{\forall t\in 0,\ldots,T}, 
\end{align*}
such that all elements of~$\mb P$ sum to~$1$, where~$\langle\cdot,\cdot\rangle$ denotes the inner product and~$\phi_t$ is the (linear) projection of the joint distribution~$\mb P$ on its~$t$-th marginal, given to be~$\mb q_t$. A generalization to compactly supported Borel measures on smooth manifolds can be found in~\cite{Pass2014}. Clearly, the discrete MMOT problem is a linear program with~$K^{T+1}$ variables in the transport tensor~$\mb P$. It has been shown that for~$T\geq 2$, combinatorial
algorithms like the simplex methods no longer remain suitable~\cite{LinHoCuturiJordan2022}. Because of the exponential state-space, recent work has studied convex relaxations using e.g., Sinkhorn iterations~\cite{LinHoCuturiJordan2022,Haasler2021}.

\textbf{From MMOT to MAPF: }We now establish a connection between 
multi-agent path finding (MAPF) and MMOT problems. Consider a finite, bounded 
region of interest~${\Omega\subset\mbb R^2}$ in which~$N$ robots operate. We 
discretize~$\Omega$ into cells, each no smaller than a robot, and define a 
graph~${\mc G=(\mc V,\mc E)}$ whose vertices~$\mc V$ correspond to cells and 
whose edges~$\mc E$ connect adjacent cells between which a robot can 
transition in one time step. The resulting graph is finite, and the cell 
size ensures that each vertex is occupied by at most one robot at any time. Let~${|\mc V|=K}$ and let~$N$ robots in the set~$\mc N$ occupy some vertices 
in~$\mc V$ and travel on the edges~$\mc E$, over a time horizon~$T$. The 
goal for the robots is to reach a set~$\mc M$ of~$M$ distinct targets, also 
in~$\mc V$, such that~${M=N}$, while minimizing the travel cost and/or the travel 
time to reach those targets\footnote{For simplicity, we assume that the 
number of robots and targets are equal. Extensions to the unequal cases can 
be easily considered using unbalanced/partial OT but is beyond the scope of 
this exposition.}. This setup is standard in the discrete MAPF 
literature; see, e.g.,~\cite{Stern2019,YuLaValle2013NetworkFlow,Standley2010}. We further let~${\bmu=\{\mu_i\}}$ and~${\bnu=\{\nu_j\}}$, both 
in~$\mbb R^K$, denote the initial distributions of~$N$ robots and targets, 
respectively, on~$K$ vertices. 

Let~$(X_t)_{t=0}^T$ be a discrete-time stochastic process with state-space~$\{1,\ldots,K\}$, where each state corresponds to a vertex of the graph~$\mathcal G$. In other words,~${(X_t)_{t=0}^T}$ is a possible trajectory taken by a robot over the time horizon~$T$. Let~${\mathbf P \in \mathbb R_{\ge0}^{K\times\cdots\times K}}$ be a~${(T+1)}$-th order tensor representing a distribution on path space. In particular, each entry~$\mathbf P_{i_0,\ldots,i_T}$ denotes the probability or the amount of mass (normalized to sum to one) assigned to a complete trajectory~${X_0=i_0, X_1=i_1,\; \ldots, X_T=i_T},$
that is, the path ${i_0\to i_1\to\cdots\to i_T}$ over the time horizon~${t=0,\ldots,T}$, with~${i_t\in\mc V}$. Thus,~$\mathbf P$ assigns mass directly to entire space-time trajectories rather than to individual vertices or edges. The collective motion of all~$N$ robots, over the horizon~$T$, is encoded in this joint path-space distribution and its one-time marginals~$\mb q_t$ recover the robot locations at each time. We have~$\sum_{i_0,\ldots,i_T} \mathbf P_{i_0,\ldots,i_T}=1$ and
\begin{equation}
\sum_{i_0=1}^K \!\cdots\! \sum_{i_{t-1}=1}^K ~~
\sum_{i_{t+1}=1}^K \! \cdots \! \sum_{i_T=1}^K
\mathbf P_{i_0,\ldots,i_T}
=
\frac{1}{N}[\mathbf q_t]_{i_t},
\end{equation}
with fixed endpoint distributions~${\mathbf q_0=\boldsymbol\mu}$ and~${\mathbf q_T=\boldsymbol\nu}$. From the path-space tensor $\mathbf P$, the transport between consecutive time layers is obtained as a two-marginal. Specifically, the transport matrix
$\Pi_t$ from time ${t-1}$ to $t$ is defined by
\begin{equation*}
[\Pi_t]_{i_{t-1},i_t}
=
N
\sum_{i_0=1}^K \! \cdots \! \sum_{i_{t-2}=1}^K \sum_{i_{t+1}=1}^K \!\cdots\! \sum_{i_T=1}^K
\mathbf P_{i_0,\ldots,i_T},
\end{equation*}
which represents the total mass (or joint probability) at vertex~$i_{t-1}$ at time~${t-1}$ \textit{and} at vertex~$i_t$ at time~$t$, aggregated over all past (before~${t-1}$) and future (after~$t$) evolution. Clearly, since robot motion is causal in time, the process~${(X_t)_{t=0}^T}$ is Markovian, i.e., for all~$t>0$,
\begin{align*}
\mathbb P(X_t=u \mid X_{t-1}=v,\ldots,X_0=w)\\
=
\mathbb P(X_t=u \mid X_{t-1}=v),
\end{align*}
for all~${u,v,w\in\mathcal V}$, and therefore the joint path-space tensor~$\mathbf P$ admits the standard factorization:
\begin{align}\nonumber
    \mathbf{P}_{i_0, \dots, i_T}
    &= \mathbb P(i_T \mid i_{T-1}) \cdot \ldots\cdot \mathbb P(i_1\mid i_0)\cdot \mathbb P(i_0)\\
    \label{P_factor}
    &= \frac{1}{N} [\mb q_0]_{i_0} \prod_{t=1}^T \frac{\tfrac{1}{N}[\Pi_t]_{i_{t-1}, i_t}}{\tfrac{1}{N}[\mb q_{t-1}]_{i_{t-1}}}.
\end{align}

Given a cost tensor~$\mb C$, with costs on each possible trajectory, the MAPF problem is to find~$N$ robot trajectories to targets that minimize travel cost or travel time. In this paper, our interest is to exploit the aforementioned probabilistic interpretation of~$\mb P$ and deploy MMOT and its efficient relaxations for MAPF. To our advantage, the Markovian factorization of the joint tensor reduces the number of free variables from~${K^{T+1}}$, exponential in~$T$, to~$K^2T$, polynomial in~$T$. Additionally, the general MMOT relaxes from fixing~$T+1$ marginals to only two boundary marginals: the starting distribution~$\mb q_0$, which is where the robots are located, and the ending distribution~$\mb q_T$, where the targets are located. The intermediate marginals~$\mb q_1,\ldots,\mb q_{T-1}$, consequently, define the positions the robots take when traveling from their starting locations to their destinations and are free under appropriate Markovianity/causality constraints as we will explicitly capture in the next section. This approach however comes at a price as MMOT, being a probabilistic object, does not necessarily result into integral~${\frac{1}{N}\{0,1\}}$ transports. In other words, the transports may be fractional and the robots may split while traveling to the targets. 

Building on this problem formulation, the remainder of the paper is organized as follows. 
\begin{itemize}
    \item Section~\ref{sec:mapf} focuses on \textit{integral} and \textit{optimal} solutions of the proposed MAPF-MMOT. In particular, we show that the corresponding LP (\textbf{P1}) is totally unimodular under mild structural assumptions; consequently, all robots take non-conflicted, non-fractional paths to their targets. We further identify when these paths are minimum cost and/or minimum makespan. 

    \item Section~\ref{SchBridge} builds towards scalability by casting MAPF as a probabilistic Schrödinger bridge and derives the corresponding entropic formulation and Sinkhorn-MAPF iterations (Appendix \ref{sinkMAPF}). The main idea is to efficiently obtain a fractional (shadow) transport that concentrates mass on the most likely paths.
    
    \item Section~\ref{intSinkMAPF} then recovers integrality from the shadow transports by solving a variation of the base~\textbf{P1} LP on a reduced graph obtained from the Schrödinger bridge, while Section~\ref{sec:comp} and \ref{sec:exp}, respectively, provide complexity analysis and a comprehensive set of experiments. 
\end{itemize} 

\section{Integral and Optimal Solutions for MAPF}\label{sec:mapf}
Recall the problem formulation in Section~\ref{sec:pf} where~$\bmu$ and~$\bnu$ (robot and target locations) are such that
\begin{align*}
\mu_i = 
\begin{cases}
1,& i\in\mc N,\\
0, & \text{otw.}
\end{cases}
\qquad 
\nu_j = 
\begin{cases}
1,& j\in\mc M,\\
0, & \text{otw.}
\end{cases}
\end{align*}
with~${\|\bmu\|_1\!=\!\|\bnu\|_1\!=\!N}$. Note that the~${(t-1)\to t}$ transport plan~${\Pi_t=\{\pi_{ij,t}\}\in \mbb R^{K\times K}}$ contains the robot distributions at time~${t-1}$ and~$t$ as marginals, i.e.,
\begin{align*}
\mb q_{t-1} = \Pi_t\mb 1,\qquad \mb 1^\top \Pi_t = \mb q_{t}^\top,\qquad {t=1,\ldots T},
\end{align*}
where~${\pi_{ij,t}}$ is the amount of mass transported from vertex~$i$, at time~${t-1}$, to vertex~$j$, at time~$t$. Our goal is to find the optimal transport plans~$\{\Pi_t\}_{t=1}^T$ that moves the robots (mass at source distribution~$\bmu$) to the targets (mass at destination distribution~$\bnu$) over the time horizon~$T$. Recall that~$\mb C$ is the cost tensor and let~${C_t=\{c_{ij,t}\}}$ denote the cost matrix at time~$t$ such that the cost of traveling on an edge~${i\to j}$, starting at~${t-1}$ and arriving at~$t$, is~$c_{ij,t}$. We impose the following structural assumptions throughout this paper; see, e.g.,~\cite{Stern2019,YuLaValle2013NetworkFlow,Standley2010} for similar setups. 

\begin{assumption}\label{assump1} 
The graph~$\mc G$ and cost matrices~$C_t$ satisfy the following:
\begin{enumerate}[(i)]
\item Self-loops~${i\to i \in \mathcal E}$, for all~${i\in\mathcal V}$, are always present. Consequently, waiting at any vertex is always feasible.

\item \label{conf_dist} If two edges in~$\mc G$ do not share a common endpoint, the corresponding physical motions can be executed simultaneously without conflict. In addition, an edge~$i{\to j}$ is included in~$\mc G$ only if its traversal is independent of the occupancy of all vertices other than~$i$ and~$j$.\footnote{This abstraction is a standard safeguard in cooperative path finding literature; see, e.g.,~\cite{Standley2010}. If needed, one may refine the graph so that collision-relevant interactions occur only at shared vertices; or on grid graphs, diagonal motion is disallowed.}

\item For each~${i\to j\in\mc E}$, the cost~${c_{ij,t}<\infty},{\forall t}$, while~${c_{ij,t}=+\infty},{\forall (i,j) \notin \mc E}$.

\item \label{add_cost} We assume that~$\mb C_{i_0,\ldots,i_T} = \sum_{t=1}^T c_{i_{t-1}i_t,t}$, i.e., the path cost is additive across time. 

\item \label{conf_self} The cost matrix satisfies the following,~${\forall t=1,\ldots,T}$:
\[
0 = c_{jj,t},\ j\in\mathcal M 
~~ < ~~
c_{ii,t},\ i\notin\mathcal M 
~~ < ~~
c_{ij,t},\ i\neq j.
\]
\end{enumerate}
\end{assumption}
The graph structure imposed in Assumption~\ref{assump1}(i)-(ii) ensures 
that the discretization faithfully captures the physical constraints of 
robot motion: edges and vertices cannot be added arbitrarily if they 
result in conflicts. Violating (i) or (ii) would mean the graph does not 
correctly model the physical environment, and collisions invisible to the 
discretized formulation may arise. The cost Assumption~\ref{assump1}(\ref{conf_self}) is natural: a move expends strictly more energy than a wait, and waiting at a target is free. We emphasize that Assumption~\ref{assump1} does not require a grid graph or point-mass robots. Any finite, connected graph~$\mc G$ satisfying (i)-(ii) suffices, and the LP formulation and its guarantees developed in the sequel hold for any such~$\mc G$.

Under Assumption~\ref{assump1}(\ref{add_cost}), the tensor inner product~$\langle\mb{P,C}\rangle$ reduces to a sum over local transports and costs. Subsequently, because of the Markovian restriction on~$\mb P$, the multi-marginal optimal transport formulation of anonymous MAPF can be equivalently written as follows:
\begin{align*}
    &\qquad\textbf{P1:}\qquad\{\Pi_t^*\}_{t=1}^T = \argmin_{\{\Pi_t\}_{t=1}^T} \sum_{t=1}^T\langle\Pi_t, C_t\rangle\\
    &\text{subject to}\quad \mc F := 
    \begin{cases}
        \Pi_t \geq0,\\
        \Pi_t^\top \mb  1= \Pi_{t+1}\mb 1, ~~~\forall t=1,\ldots,T,\\
        {\Pi_1\mb 1=\bmu},~{\Pi_T^\top\mb 1=\bnu},\\
        0\leq \Pi_t^\top \mb 1 \leq \mb 1, ~~~\forall t=1,\ldots,{T-1},
    \end{cases}
\end{align*}
We explain this MMOT formulation next:
\begin{itemize}
    \item \textbf{P1} is a linear program with linear constraints, in which the transport plans~$\Pi_t$'s are real-valued nonnegative decision variables (not necessarily integral).
    
    \item Gluing constraints: The second set of equality constraints impose a Markovian restriction to the global MMOT problem. They ensure that the mass transported from~$t$ to~${t+1}$, must be the mass that arrives at~$t$ from~${t-1}$, i.e.,~${\Pi_{t+1}\mb 1 = \mb q_t}$ and also $\Pi_t^\top \mb 1 = \mb q_{t}$.
    
    \item Terminal constraints: The boundary marginals are fixed to enforce the robots initial and terminal positions.
    
    \item Vertex-capacity constraints: The last inequality constraints ensure that no location (vertex in~$\mc G$) receives more than one robot.
\end{itemize}
Clearly, if~$\Pi_t$'s are integral, i.e.,~${\pi_{ij,t}\in \{0,1\},\forall i,j,t}$, then we get executable robot paths to the targets. Assumption~\ref{assump1} and the constraint polyhedron~$\mc F$ further ensure that all targets are reached, while the robot trajectories are non-conflicting and min-cost. We characterize these results~next. 

\subsection{Main Results}
We now characterize some basic properties of~\textbf{P1} in the following lemmas.
\begin{lem}\label{lem1}
    Let~$\mc {G=(V,E)}$ be a finite, connected graph with~${|\mc V|=K}$. Consider~$N$ robots in~$\mc N$ and~$M$ targets in~$\mc M$, on distinct vertices in~$\mc G$, such that~${N=M\leq K/2}$. For each time ${t\in\mathbb N}$, let~${C_t=\{c_{i,j,t}\}}$ satisfy Assumption~\ref{assump1}. Then, there exist a finite ${\bar T\in\mathbb N}$ and a transport~${\{\Pi_t\}_{t=1}^{\bar T}}$ feasible for \textbf{P1}, i.e., satisfying all constraints in~$\mc F$, such that~$\sum_{t=1}^{\bar T} \langle \Pi_t, C_t\rangle < \infty.$
\end{lem}
The proof is provided in Appendix~\ref{p_lem1}. Lemma~\ref{lem1} establishes the feasibility of anonymous MAPF and the constraint set~$\mc F$ and does so purely at a structural level, without invoking any properties of the corresponding linear program (LP) beyond Assumption~\ref{assump1}. The following lemma now concretely establishes the properties of the LP in~\textbf{P1}.

\begin{lem}\label{lem:P1-LP}
Consider the settings of Lemma~\ref{lem1} and Assumption~\ref{assump1}, and fix a horizon~$\bar T$ such that~$\mc F$ is nonempty with~${\sum_{t=1}^{\bar T}\langle \Pi_t,C_t\rangle<\infty}$, for some feasible~$\{\Pi_t\}_t$. Then, the MMOT formulation \textbf{P1} over the horizon $\bar T$ satisfies:
\begin{enumerate}[(i)]
    \item \textbf{P1} admits an optimal solution~${\{\Pi^*_t\}_{t=1}^{\bar T}}$ that is integral, i.e.,~${\Pi^*_t \in\{0,1\}^{K\times K}}$, for all~${t=1,\ldots,\bar T}$.

    \item \textbf{P1} results in a transport~${\{\Pi^*_t\}_{t=1}^{\bar T}}$ that attains a minimum cost over the horizon $\bar T$. 

    \item The complexity of~\textbf{P1} is polynomial in $K$ and $\bar T$.
\end{enumerate}
\end{lem}

The proof is provided in Appendix~\ref{plem:P1-LP}, where we show that \textbf{P1} admits an optimal \textit{integral} basic solution, as a consequence of total unimodularity of the constraint matrix~\cite{schrijver1998}. The result can also be viewed through the lens of classical integrality arguments for network flows~\cite{FordFulkerson1962,AhujaMagnantiOrlin1993}. Regarding (iii), see Section~\ref{sec:comp} for precise complexity arguments. The next theorem uses the results of the previous two lemmas and applies them to the anonymous MAPF problem.

\begin{theorem}\label{lem:P1-MAPF}
Consider the settings of Lemmas~\ref{lem1},~\ref{lem:P1-LP}, and Assumption~\ref{assump1}. The optimal transport plan~$\{\Pi^*\}_{t=1}^{\bar T}$, returned by \textbf{P1}, satisfies:
\begin{enumerate}[(i)]
    \item No two robots collide at any time. 

    \item The robot trajectories do not overlap in \textit{both} space and time.
    
    \item All robots reach a distinct target.
\end{enumerate}
\end{theorem}

Theorem~\ref{lem:P1-MAPF}, proved in Appendix~\ref{p_lem:P1-MAPF}, is one of the central results of this paper. It establishes the conditions under which MMOT gives executable robot trajectories, i.e.,~$\{0,1\}$  transports, relying on the total unimodularity (TU) of~\textbf{P1}. TU in general is delicate, and adding arbitrary constraints to enforce a desirable behavior on robot trajectories typically breaks TU. It is therefore preferable to impose desirable traits in the trajectories through the cost matrices~$\{C_{t}\}_t$. We next describe a few of such desirable properties enforced via costs, in addition to Assumption~\ref{assump1}; it can be verified that they do not violate Lemma~\ref{lem:P1-LP} and Theorem~\ref{lem:P1-MAPF}.

\textbf{No oscillations}: For all distinct vertices~${i\neq j\in\mc V}$, and~all~$t$, let~${c_{ii,t} + c_{ii,t+1} < c_{ij,t} + c_{ji, t+1}}.$ In other words, leaving~$i$ and returning to it, now or later, costs more than staying at~$i$.\\
\textbf{Temporal urgency}: For all ${i\to j\in\mc E}$, all~$t$, and all feasible moves, let~${c_{ij,t} \le c_{ij,t+1}}.$ Thus, executing a move is never cheaper later. Combined with Assumption~\ref{assump1}(\ref{conf_self}), this implies that whenever a robot can reach a target earlier, there exists an optimal transport in which it does so and subsequently waits at the target.\\
\textbf{Temporal subadditivity}: For all~${i,j,k\in\mc V}$ and all~$t$, let~${c_{ij,t} \le c_{ik,t} + c_{kj,t+1}}.$ Thus, whenever a direct move from~$i$ to~$j$ is available, routing through an intermediate vertex over consecutive time steps is never cheaper. This rules out avoidable detours: if a robot can move directly to a vertex and then wait, it is never optimal to reach the same vertex via an unnecessary intermediate location.\\
The two temporal conditions described above should not be used in scenarios, e.g., where a toll road becomes cheaper later (within the horizon~$\bar T$) and the goal is to exploit that.\\
\textbf{Shortest-path costs}: Assume that~$\mc G$ is endowed with an edge-length metric~${d:\mc E \ra \mbb R_+}$ such that for all ${i\to j\in\mc E}$ and all $t$,~$    {c_{ij,t} = d(i,j)}.$ Under this cost structure, any minimum-cost transport minimizes the total traveled distance among all feasible transports. A canonical example is robots and targets embedded in~$\mbb R^2$, with $d(i,j)$ given by the Euclidean distance.

\textbf{Minimum Moves:} It can be shown that given a feasible horizon~$\bar T$, under transition costs~${c_{i\neq j,t}=1}$, for all~$t$,
and for sufficiently small waiting costs~$c_{ii,t}>0$, for~$i\notin\mc M$, a
min-cost solution of \textbf{P1} is also a min-move solution.

\textbf{Minimum makespan: }For any feasible transport $\{\Pi_t\}_{t=1}^{\bar T}$, its makespan is the largest time-index for which~${\pi_{ij,t}>0}$, for some ${i\neq j}$. The minimum makespan is the smallest such value over all feasible transports. We next describe obtaining a minimum makespan transport from~\textbf{P1} by tuning costs with the help of the following assumption and lemma.
\begin{assumption}\label{assump:time-sep}
Let~$\{\tilde c_{ij}\}_{i\to j\in\mc E}$ satisfy~$
{\tilde c_{jj}=0,\ j\in\mc M},$ and~${0<\tilde c_{ii,i\notin \mc M} <\tilde c_{ij,i\neq j}}$ and define~$\tilde c_{\min}$ to be the minimum over all~${\tilde c_{ij}}$, with~${i\neq j}$. Choose~${C_t=\{c_{ij,t}\}}$ to be such that~${c_{ij,t} := B^{\,t}\,\tilde c_{ij}}, {\forall\, i\to j\in\mc E},\ \forall\, t,$ where~$B$ is chosen to  satisfy~${(B-1)\tilde c_{\min}>\sum_{i\to j\in \mc E}\tilde c_{ij}}$. 

Note that Assumption~\ref{assump:time-sep} imposes a much stronger growth condition on the costs. It can be further verified that it satisfies Assumption~\ref{assump1} and also implies the aforementioned no-oscillations and temporal cost conditions.
\end{assumption}

\begin{lem}\label{lem:mm}
    Let~$\mc G$ be finite and connected and consider~$C_t$ such that it satisfies Assumption~\ref{assump:time-sep}. Suppose that~$\bar T$ is a feasible horizon for~\textbf{P1}. Then, for any solution~$\{\Pi_t^*\}_{t=1}^{\bar T}$ of~\textbf{P1}, all robot motions terminate by~$T^*$, where~$T^*$ is the minimum makespan over all feasible transports.
\end{lem}

The proof is provided in Appendix~\ref{p_lem:mm}. Note that Lemma~\ref{lem:mm} provides a transport over the entire horizon~$\bar T$ and achieving minimum makespan is implicit. In other words, the robot motion ceases at~$T^*$ because of rapidly growing time-dependent costs. Such costs may cause numerical instability; to avoid that the following result explicitly searches for the minimum feasible horizon~$T^*$.
\begin{lem}\label{lem:T*}
For a feasible anonymous MAPF instance over a finite, connected graph~$\mc G$ with~${|\mc V|=K}$, the minimum makespan satisfies~${T^* \le N + K - 1}$. A minimum makespan transport can be found in~$\mc O(\log K)$ calls to~\textbf{P1}.
\end{lem}
The~$T^*$ bound in this lemma can be found, e.g., in~\cite{YuLaValle2013NetworkFlow,Ma2020Thesis}. The rest of the lemma follows by performing a binary search over the horizon~${T \in [0,\, N+K-1]}$ and checking for the earliest feasibility of~\textbf{P1}. We thus obtain two complementary approaches for computing minimum makespan transports in Lemmas~\ref{lem:mm} and~\ref{lem:T*}. The exponential cost construction in Assumption~\ref{assump:time-sep} and consequently Lemma~\ref{lem:mm} implicitly encode the makespan optimality into the objective, at the expense of aggressive cost scaling. Alternatively, minimum makespan can be found explicitly by searching over the horizon, requiring~$\mc O(\log K)$ calls to~\textbf{P1} with simpler cost. 

An alternate min-makespan formulation can be achieved by minimizing~$z$, such that~${z \geq t\, \pi_{ij,t}}$, $\forall\, i\notin\mc M,\, j\in\mc V,\, t$. However, the resulting LP is not totally unimodular in general. In other words, explicit makespan minimization requires integer programs, and therefore the implicit mechanism described here may be preferable.

The results in this section yield a polynomial-time LP for anonymous MAPF; see also Section~\ref{sec:comp}. This complexity however may be impractical for very large-scale MAPF. In the subsequent sections, we develop scalable solutions for~\textbf{P1} by casting MAPF as a discrete Schr\"odinger bridge, which provides a principled probabilistic framework for formulating and analyzing MAPF (Section~\ref{SchBridge}). This formulation, under appropriate conditions, leads to a convex Problem~\textbf{P2}, which we show is an entropic relaxation of~\textbf{P1}. The resulting~\textbf{P2} admits highly efficient Sinkhorn-type iterations (Appendix~\ref{sinkMAPF}) but yields fractional (shadow) transports. Integrality is then enforced with a subsequent projection step to recover executable robot motions (Section~\ref{intSinkMAPF}).

\newpage
\section{MAPF and the Schr\"{o}dinger Bridge Problem}\label{SchBridge}
The classical Schr\"{o}dinger bridge is described as follows. Given a reference diffusion process~\(G\) on a topological state space~\(\mathcal X\), with arbitrary marginals, find a probability measure~\(P^*\) on the space of continuous trajectories that minimizes the relative entropy with respect to~\(G\):
\[
P^*
= \argmin_{P}\,
\mathrm{KL}(P\,\|\,G),
\]
such that the measure~\(P^*\) has fixed initial and terminal marginals~\cite{schrodinger1931umkehrung,leonard2014survey}. Here,~$G$ is typically chosen as Brownian motion and~$P^*$ is the most plausible stochastic evolution whose continuous-time trajectories interpolate between the given boundary marginals. Schrödinger bridge formulations are widely used in applied physics and are typically studied in continuous settings to model the evolution of particle systems, see e.g., the hot and lazy gas experiments in~\cite{Villani2009,Leonard2017}. Originally introduced by Erwin Schrödinger in the 1930s, the Schrödinger bridge problem characterizes the path-space trajectories of gas particles from empirical observations of their distributions at two time instants, and is closely connected to large deviation theory, where Schrödinger bridges arise as minimizers of associated rate functionals~\cite{Follmer1988}.
We now cast MAPF as a Schrödinger bridge and characterize the conditions under which this formulation reduces to the MAPF–MMOT problem~\textbf{P1}.

For the remainder of the paper, we assume that ${T<\infty}$ is a feasible horizon for \textbf{P1}. Recall the factorization of~$\mb P$ in~\eqref{P_factor} and let~$\mathbf{G}$ denote a reference Markovian tensor on~$\mc V$ with a similar factorization, 
where~$\mathbf{G}_t$ are the transports of the reference distribution~$\mb G$ with marginals~${{\mb G}_t\mb 1 = \mb g_{t-1}}$ and~${{\mb G}_t^\top\mb 1 = \mb g_{t}}$. The S\"chrodinger bridge problem corresponding to~\textbf{P1} seeks a joint distribution~$\mathbf{P}$, and the corresponding  sequence of transports~$\{\Pi_t\}_{t=1}^T$, that is closest to~$\mathbf{G}$ in the relative entropy sense, i.e.,
\begin{equation}\label{KLtensor}
    \min_{\mathbf{P} \in \mathcal{C}} \mathrm{KL}(\mathbf{P} \| \mathbf{G}),
\end{equation}
where $\mathcal{C}$ is the set of tensors satisfying the constraints of~\textbf{P1} and each element of~${\mb G}$ is such that the $\mathrm{KL}$ divergence is well-defined. 
\begin{lem}\label{lem_KL}
    The Schrödinger bridge in~\eqref{KLtensor} reduces to
    \begin{align}\nonumber
    \mathrm{KL}(\mathbf{P} \| \mathbf{G})
    &= \sum_{t=1}^T \mathrm{KL}(\tfrac{1}{N}\Pi_t \| \mathbf{G}_t)
    + \mathrm{KL}(\tfrac{1}{N}\mb q_0 \| \mb g_0)\\\label{fullKL}
    &- \frac{1}{N}\sum_{t=1}^T \mathrm{KL}(\tfrac{1}{N}\mb q_{t-1} \| \mb g_{t-1}).
\end{align}
\end{lem}
The proof is provided in Appendix~\ref{p_lem_KL}. Note that the full Schr\"odinger bridge formulation decomposes into transport and marginal $\mathrm{KL}$ terms~\eqref{fullKL}; a related decomposition is derived in~\cite{pavonTicozzi2010discrete} for the initial-final marginal problem using conditional transition probabilities. Eq.~\eqref{fullKL} represents the general form of the Schrödinger bridge, through which one may impose a desirable structure on the robot trajectories by appropriately choosing the reference transports~$\mb G_t$ and marginals~$\mb g_t$. Consequently,~\eqref{fullKL} returns a MAPF transport~$\mb P$ that is consistent with the initial and final robot locations while remaining close to~$\mb G$. Existence and uniqueness of Schrödinger bridges are studied in~\cite{pavonTicozzi2010discrete}, where a solution is derived using space–time harmonic functions under suitable assumptions on the reference processes. In the following, we restrict the reference distributions to the Gibbs form, which leads to Sinkhorn-type iterations.

\begin{lem}\label{lem_KL_Gibbs}
Let~${C_t=\{c_{ij,t}\}}$ be the cost matrix and let~${{\mb G}_t=\{\bar g_{ij,t}\}}$ be the normalized Gibbs kernel, i.e., $
{{\bar g}_{ij,t} \!:=\! \frac{g_{ij,t}}{z_t}}, {g_{ij,t} \!:=\! \exp\!\left(-\frac{c_{ij,t}}{\varepsilon}\right)},$ and $ {z_t:=\sum_{k,\ell} g_{k\ell,t}}.$ Then, for each~$t$,
\begin{align*}
\mathrm{KL}(\tfrac{1}{N}\Pi_t \| {\mb G}_t)
&=
\tfrac{1}{N}\sum_{i,j}\Pi_{ij,t}\log \Pi_{ij,t} \\
&+\tfrac{1}{N}\sum_{i,j}\Pi_{ij,t} \frac{c_{ij,t}}{\varepsilon}
+ \log \tfrac{1}{N}+\log z_t.
\end{align*}
\end{lem}

The proof is provided in Appendix~\ref{p_lem_KL_Gibbs}.
The above lemma leads to the following Schr\"odinger bridge formulation of MAPF, when the reference distribution~${\mb G}$ is the Gibbs kernel (after removing the constants that do not depend on the transport variables) and minimizing~${\varepsilon N\,\mathrm{KL}(\tfrac{1}{N}\Pi_t\|{\mb G}_t})$:
\begin{align*}
\textbf{P2:}\min_{\{\Pi_t\in\mc F\}_{t=1}^T}
\sum_{t=1}^T
\left(
\langle \Pi_t, C_t\rangle
+
\varepsilon\sum_{i,j}\pi_{ij,t}(\log \pi_{ij,t}-1)
\right)
\end{align*}
under the constraints of~\textbf{P1}. We let~$\{\widetilde\Pi_t\}_{t=1}^T$ denote the transport resulted by~\textbf{P2} and note that~\textbf{P2} is precisely the entropic regularization of~\textbf{P1}. 

To obtain efficient and scalable solutions,~\textbf{P2} imposes the non-negativity constraints on the transport variables and allows~${\Pi_t \ge 0}$. This relaxation yields a convex problem that can be solved efficiently using Sinkhorn-type iterations. Under this relaxation, the marginal distributions induced by~$\Pi_t$ may become fractional, and the marginal KL terms in the full
Schr\"odinger bridge objective (Eq.~\eqref{fullKL}) may not remain constant. Consequently,~\textbf{P2} may no longer remain exactly equivalent to the Schr\"odinger bridge but can be interpreted as a tractable relaxation thereof. Solving the convex relaxation~\textbf{P2} yields fractional (shadow) transports, over which integrality can be imposed, as we will develop in Section~\ref{intSinkMAPF}. We provide the corresponding multi-marginal Sinkhorn iterations to build the fractional shadow transports~$\widetilde\Pi$ in Appendix~\ref{sinkMAPF}. A formal discussion and analysis of Sinkhorn-MAPF is beyond the scope of this paper. 

\section{Integral Projection of Sinkhorn-MAPF}\label{intSinkMAPF}

To obtain an integral solution from the entropy-regularized transports~$\{\widetilde\Pi_t\}_{t=1}^T$, obtained from~\textbf{P2}, we project~$\widetilde\Pi_t$'s back on the totally unimodular polyhedron~$\mc F$. To this end, we minimize a modified objective~${\langle\Pi_t, C_t\rangle + \lambda\, \mathrm{KL}(\Pi_t\|\widetilde \Pi_t)}$ that penalizes the transport~$\Pi_t$ when it's far from~$\widetilde\Pi_t$. Adding the KL penalty however makes the objective nonlinear, which we address by linearizing around an operating point~$\pi^0_{ij,t}$, yielding
\begin{align*}
    \mathrm{KL}(\Pi_t\|\widetilde \Pi_t) &\approx 
    \sum_{i,j}\pi_{ij,t} \left(\log \pi_{ij,t}^0  - \log {\widetilde \pi_{ij,t}}\right)\\ 
    & + \sum_{i,j}\pi_{ij,t} 
    -\sum_{i,j} \pi_{ij,t}^0.
\end{align*}
Dropping constants and choosing~$\pi_{ij,t}^0$ to be a constant over all~$i,j$, we obtain 
\begin{align*}
    \textbf{P3:}~~
    \min_{\{\Pi_t\}_{t=1}^T} \sum_{t=1}^T&\left( \sum_{i,j}\pi_{ij,t}\left(c_{ij,t} - \lambda \log ({\widetilde \pi_{ij,t}}+\delta)\right)\right)\\
    \text{subject to}\quad &{\Pi_t\!\in\!\mc F},~ {\Pi_t\subseteq[\widetilde\Pi_t]_\eta}, \forall t,
\end{align*}
where~${\delta\geq 0}$ ensures the logarithm is well-defined, and~$[\widetilde\Pi_t]_\eta$ is the regularized transport~$\widetilde\Pi_t$ with zeros for all elements that are at most~$\eta$. We let~$\{\widehat\Pi_t\}_{t=1}^T$ denote the transport resulted by~\textbf{P3}, which is effectively built from the \textit{shadow} transport~\textbf{P2}, i.e, the entropic regularization of~\textbf{P1}. We provide a few important remarks regarding \textbf{P3} next:

\begin{itemize}
    \item The constants~$\lambda,\delta$ must be chosen carefully to ensure that the modified cost~${c_{ij,t} - \lambda \log ({\widetilde \pi_{ij,t}}+\delta)}$ lies in the purview of Assumption~\ref{assump1}. Clearly,~$\{\widehat\Pi_t\}_t$ is integral~$\{0,1\}$, because~\textbf{P3} is an LP under the same feasibility polyhedron~$\mc F$ of~\textbf{P1}. However, the pruned graph or~$T$ may no longer remain feasible and~$\eta$ may need to be adjusted accordingly. 

    \item The interplay between the three constants~$\ve,\lambda,\eta$ straddle the spectrum of transports from optimal to highly scalable. Note that~${\lambda=\eta=0}$, disconnects \textbf{P2} (regardless of~$\ve$) and recovers \textbf{P1}.

    \item Choosing~${\ve,\lambda>0}$, independent of~$\eta$, biases \textbf{P3} to place mass on edges with large~$\widetilde\pi_{i,j,t}$, since the modified objective in \textbf{P3} arises from a linearized KL divergence between~$\widehat\pi_{i,j,t}$ and~$\widetilde\pi_{i,j,t}$. Consequently,~$\widehat\pi_{i,j,t}$ may inherit the smoothing effect induced by~\textbf{P2} as~$\ve\uparrow$.

    \item A scalable recipe is apparent: solve convex~\textbf{P2} fast, build a shadow, and subsequently solve the~\textbf{P3} LP over a pruned graph (with appropriate choices of~$\ve,\eta,\lambda,\delta$). A detailed breakdown of these parameters, based on~$260$ experiments on a~$1.5$M-variable problem, is provided in Appendix~\ref{sensitivity}.

\end{itemize}

\section{Computation Complexity}\label{sec:comp}
The LP in Problem~\textbf{P1} returns~$T$ transport matrices~${\{\Pi_t\}_{t=1}^T}$, each supported on~$|\mc E|$ edges, resulting in~${n := |\mc E|T}$ decision variables and~$\mc O(n)$ constraints. On physical graphs with bounded-degree connectivity (e.g., nearest-neighbor motion), ${|\mc E| = \mc O(K)}$ and thus~${n = \mc O(KT)}$. Using classical interior-point methods for linear programming, Problem~\textbf{P1} admits worst-case bit-complexity bounds of order $\mc O(n^{3}L)$, where~$L$ is the encoding length of the input data. The corresponding methods may not return a basic optimal solution (a vertex of the constraint polyhedron); in such cases, standard crossover techniques can be used to recover an extreme-point solution~\cite{Wright1997,PotraWright2000}. While these methods provide polynomial-time guarantees, practical implementations often rely on simplex-based methods, which perform significantly faster on large-scale instances despite lacking polynomial worst-case guarantees.

To improve scalability,~\textbf{P2} solves an entropic relaxation via Sinkhorn iterations (Appendix~\ref{sinkMAPF}) that are specialized to the MAPF problem and its constraints. The convergence and complexity of multi-marginal Sinkhorn are studied, e.g., in~\cite{DiMarinoGerolin2020,Carlier2022}, where linear convergence is established when the iterations are viewed as block coordinate descent on a convex dual objective. In practice, only a small number of Sinkhorn iterations suffices to construct the pruned graph used later in~\textbf{P3}, as demonstrated in Section~\ref{sec:exp} and in the detailed experiments (Appendix~\ref{extra_sim}). Finally, Problem~\textbf{P3} solves the original linear program over the pruned graph and therefore has the same worst-case complexity bound as the aforementioned interior-point methods. However, due to the effective graph pruning based on the shadow transport of~\textbf{P2}, the number of variables is significantly reduced, with~${n\approx \zeta |\mc E|T}$, where in practice~$\zeta$ typically lies in the range~$[0.2,0.4]$ for large-scale instances, as demonstrated in the experiments.

The above complexity bounds are worst-case limits. In practice, modern LP solvers such as Gurobi are highly optimized and we observe empirical solve time scaling as~${O(K^{1.68})}$ for~\textbf{P1} and as~${O(K^{1.15})}$ for the~\textbf{P2}+\textbf{P3} pipeline; see Fig.~\ref{fig:scaling_runtime} and Appendix~\ref{scaling} for a detailed scaling study.

\begin{figure*}
    \centering
    \includegraphics[width=0.95\textwidth]{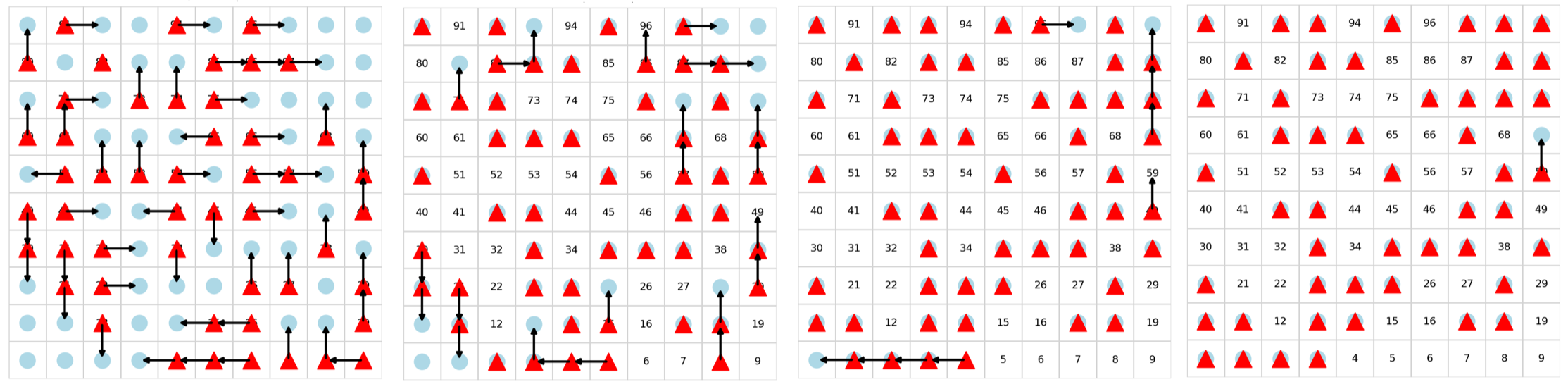}
    \includegraphics[width=0.95\textwidth]{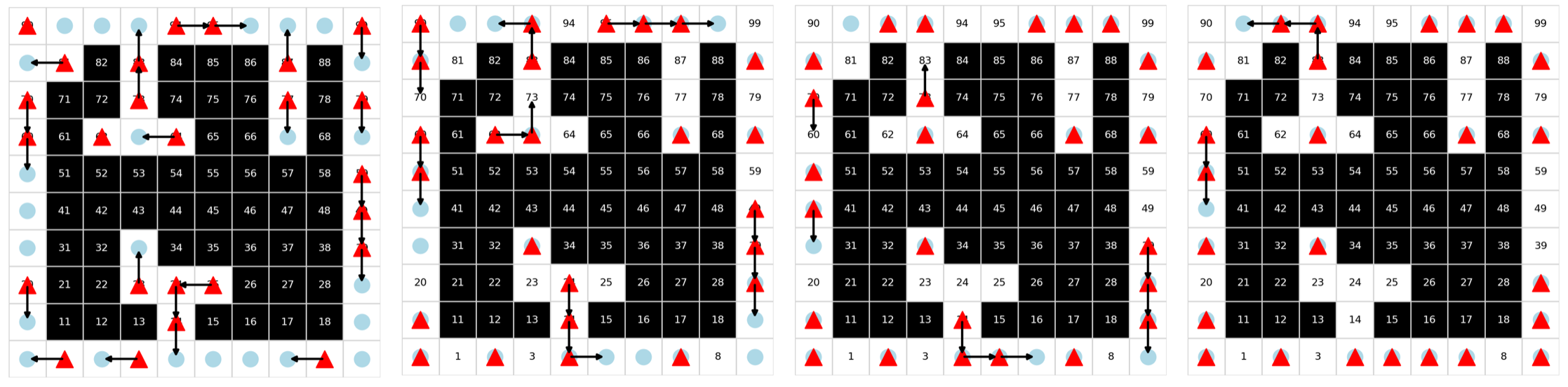}
    \caption{A~${10\times 10}$ grid with~${\{50,25\}}$ robots~$\robot$ and targets~$\target$, and~${\{0,50\}}$ obstacles.}
    \label{fig1}
\end{figure*}

\section{Experiments}\label{sec:exp}
We consider a grid of size~${K\!=\!{W\!\times\! H}}$ with potential obstacles; while the framework applies to arbitrary graphs, grids are chosen for simplicity.
Motion is allowed in the cardinal directions, while diagonal motion is prohibited. Waiting costs~$0$ at targets and~$0.5$ at non-targets; move cost is~${c_{i\neq j}=1}$ for all~${{i\to j \in\mc E}}$. The experiments are conducted using the HiGHS solver from the SciPy’s linprog library. Fig.~\ref{fig1} shows the min-cost~\textbf{P1} transport for two arbitrary configurations over~${T=3}$; we note that arbitrarily placed obstacles alter the connectivity and diameter of the underlying graph, and it no longer remains a regular grid. We next demonstrate the Schrödinger shadow transport (\textbf{P2}) in Fig.~\ref{fig3} on a~${40\times 40}$ grid with~$\{20,80\}$ robots and targets; all trajectories are superimposed. The left figures show the optimal integral transports~\textbf{P1}; the middle figures show the (fractional) Sinkhorn-MAPF with~${\varepsilon=50}$; while the right figures show integral~\textbf{P3} transports obtained on pruned graphs with nominal cost degradation. Finally, Fig.~\ref{fig4} demonstrates the cost degradation versus edges kept from the shadow transport on a~${K=W^2}$ grid with~$2W$ robots. We observe that shadow-based pruning is much more effective and feasibility is obtained with smaller number of edges as~$K\uparrow$ with a nominal cost degradation. 

To evaluate the runtime scaling, we switch to the Gurobi LP solver and conduct~$162$ independent runs on square grids ranging from~${K=2{,}500}$ to~${K=22{,}500}$ vertices at~$5\%$ robot density~(${N=0.05K}$) and~${T=30}$.
Fig.~\ref{fig:scaling_runtime} shows that~\textbf{P1} solve time
grows as~${O(K^{1.68})}$ while the~\textbf{P2}+\textbf{P3} pipeline scales nearly linearly as~${O(K^{1.15})}$, yielding speedups from~$3.6\times$ to~$7.1\times$ with cost gap consistently below~$10\%$. Every solution across all runs is verified integral. See Appendix~\ref{scaling} for the full scaling study.

Table~\ref{tab:sensitivity} reports the average cost gap~(\%) of
the~\textbf{P2}+\textbf{P3} pipeline over~$260$ runs at~${K=10{,}000}$
for varying~$(\varepsilon,\lambda)$. The regularization
parameter~$\varepsilon$ is the dominant factor: small~$\varepsilon$
produces a concentrated shadow close to the~\textbf{P1} optimum,
while~$\lambda$ has a milder effect. A robust default
is~${\varepsilon=0.2}$,~${\lambda=0}$: a~$4.3\%$ gap at~$5\times$ speedup.
See Appendix~\ref{sensitivity} for the full sensitivity analysis.

Appendix~\ref{nu_costs} validates the proposed approaches under non-uniform costs, and Appendix~\ref{comps} compares against the CBM baseline of~\cite{Ma2016}. 

\begin{table}[t]
\centering
\caption{Average cost gap~(\%) relative to~\textbf{P1} for
each~$(\varepsilon,\lambda)$ pair, over~$13$ instances
at~${K=10{,}000}$,~${T=30}$.}
\label{tab:sensitivity}
\vspace{0.3em}
\small
\begin{tabular}{@{}c c c c c@{}}
\toprule
$\varepsilon$ \textbackslash{} $\lambda$ & $0$ & $0.5$ & $1.0$ & $5.0$ \\
\midrule
$0.1$ & 2.3 & 2.5 & 2.7 & 3.0 \\
$0.2$ & 4.3 & 5.0 & 5.8 & 7.3 \\
$0.5$ & 11.1 & 12.7 & 14.0 & 16.5 \\
$1.0$ & 17.3 & 18.9 & 20.1 & 23.1 \\
$5.0$ & 17.1 & 18.0 & 18.7 & 19.9 \\
\bottomrule
\end{tabular}
\end{table}

\section{Conclusions}
In this paper, we develop a principled framework for multi-agent path finding (MAPF) that bridges multi-marginal optimal transport, entropy-regularized relaxations, and linear programming. By showing total unimodularity of the feasibility polyhedron, we obtain integral transports in polynomial time without explicitly enforcing integrality. We extend the methodology to Schrödinger bridges and entropic formulations that provide a novel probabilistic viewpoint of MAPF, yielding scalable approximations and structural guidance for reducing problem size. Building on this structure, the proposed projection and pruning strategies enable efficient recovery of executable transports, offering a flexible trade-off between tractability, smoothness, and optimality.

\clearpage
\begin{figure*}
    \centering
    \includegraphics[width=0.7\textwidth]{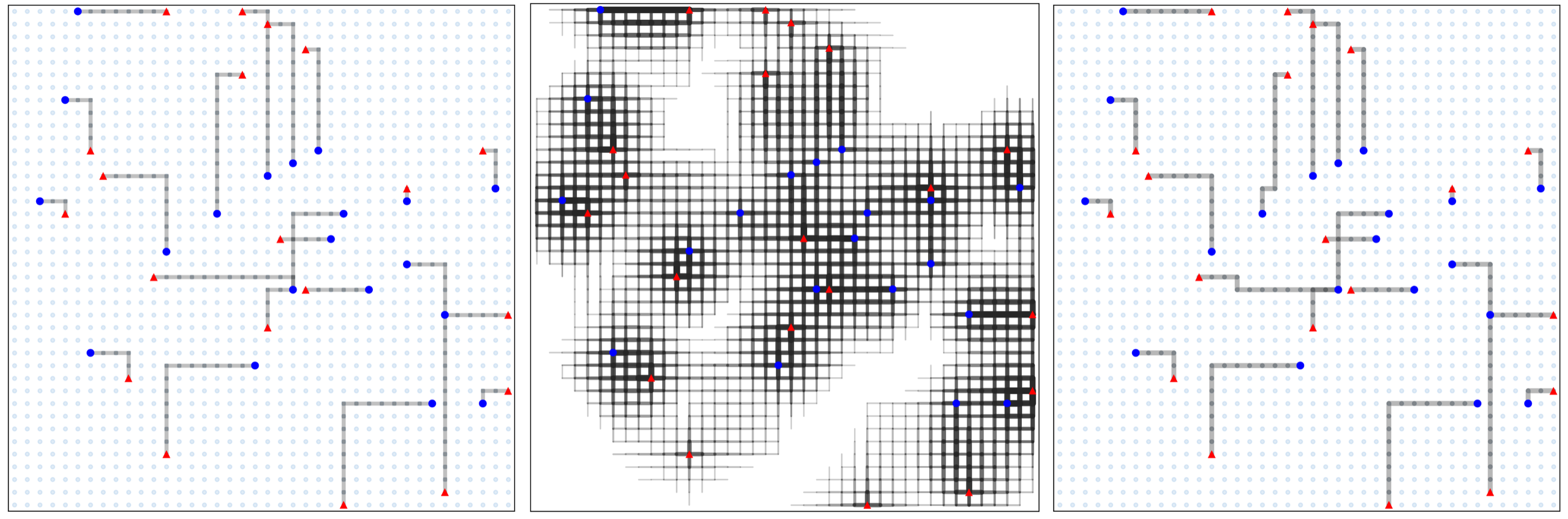}
    \includegraphics[width=0.7\textwidth]{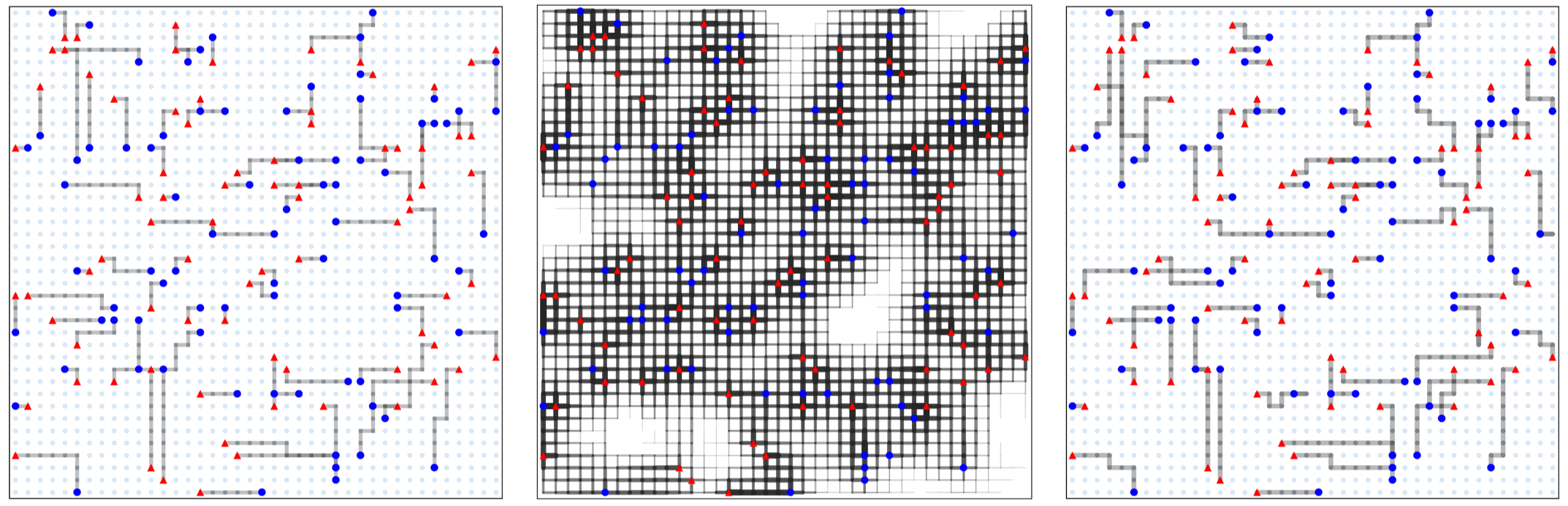}
    \caption{(Left) Optimal~\textbf{P1}; (Middle) Schr\"odinger shadow~\textbf{P2}; (Right) Integral projection~\textbf{P3}.
    Top (${N=20,T=15}$):~\textbf{P1} cost~$181$;~\textbf{P2} cost~$1053$;~\textbf{P3} with~$23\%$ edges retained at cost~$181$, i.e.,~$0\%$ degradation. 
    Bottom (${N=80,T=10}$):~\textbf{P1} cost~$402$;~\textbf{P2} cost~$3160$;~\textbf{P3} with~$22\%$ edges retained at cost~$436$, i.e.,~$8.5\%$ degradation.}
    \label{fig3}
\end{figure*}

\begin{figure*}
    \centering
    \includegraphics[width=0.9\textwidth]{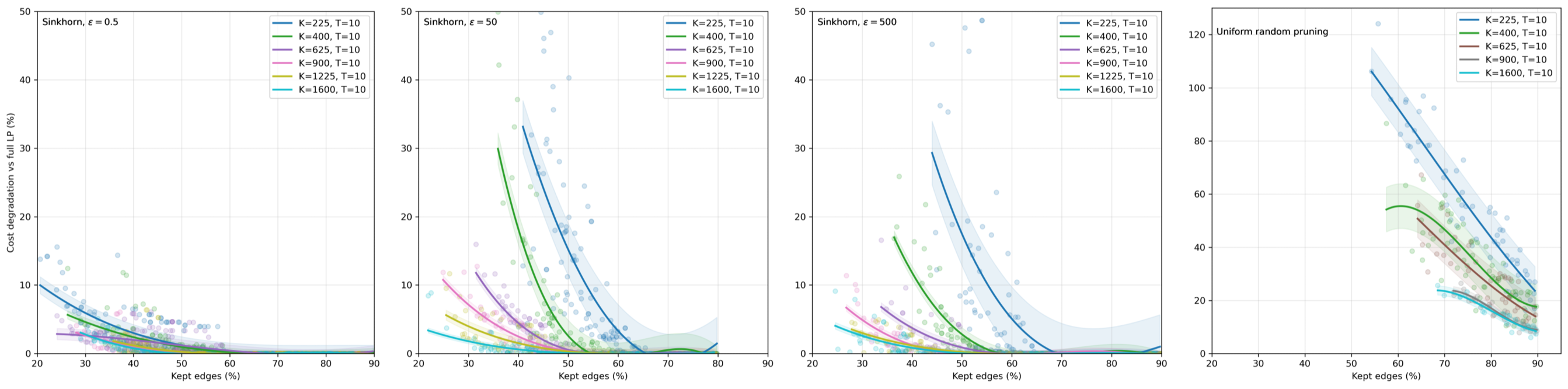}
    \caption{Scalable MAPF: The vertical axis plots the cost degradation~$x$, i.e., the cost of transport obtained from~\textbf{P3} is~$(1+x)c_{opt}$, where~$c_{opt}$ is the optimal min-cost from~\textbf{P1}; the horizontal axis plots the~$\%$ of edges retained from the full~\textbf{P1} transport.}
    \label{fig4}
\end{figure*}

\begin{figure*}[t]
    \centering
    \includegraphics[width=0.75\textwidth]{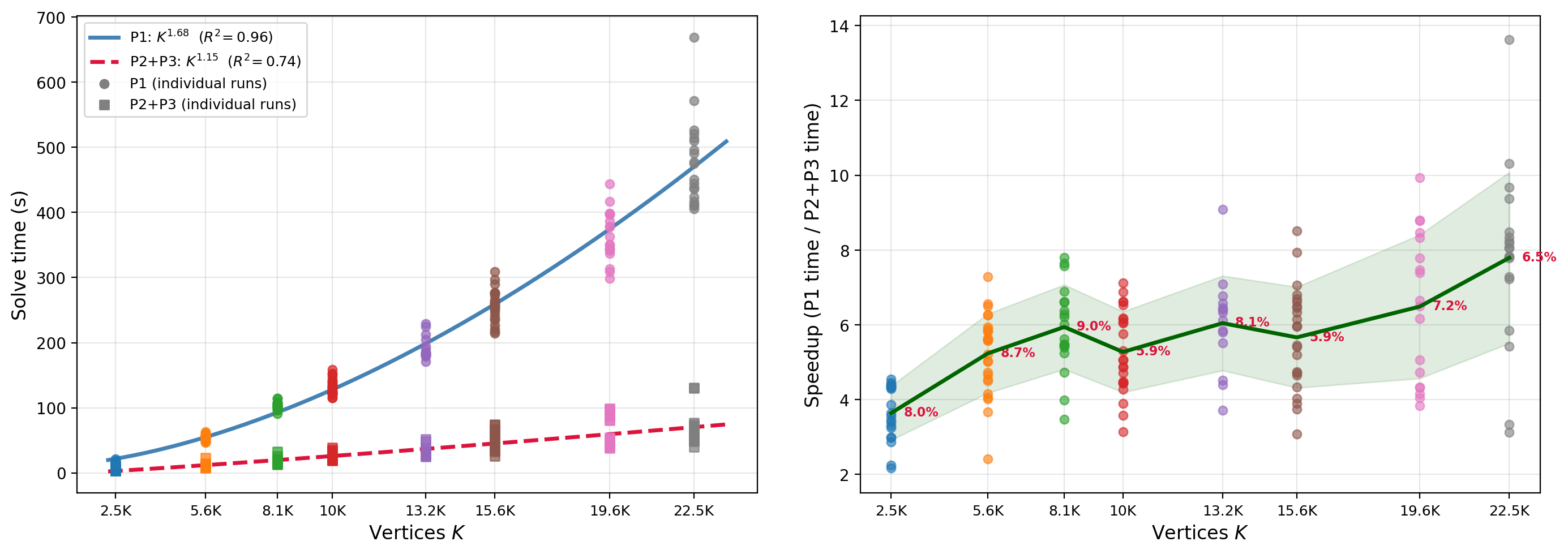}
    \caption{Runtime scaling across~$162$ runs at~$5\%$ robot density,~${T=30}$. (Left)~Solve time of~\textbf{P1} (circles) and~\textbf{P2}+\textbf{P3} (squares) versus~$K$; curves show power-law fits~${aK^p+b}$. (Right)~Speedup versus~$K$; the green line connects averages, the shaded band shows~${\pm 1}$ std.\ dev., and red annotations indicate the average cost gap.}
    \label{fig:scaling_runtime}
\end{figure*}

\clearpage
\section*{Acknowledgments}
This paper has been accepted for publication at the \textit{43rd International Conference on Machine Learning (ICML'26)}, July 2026, as a spotlight paper. 

\bibliographystyle{icml2026}
\bibliography{Refs.bib}

\appendix
\onecolumn

\section*{Appendix Overview}
The appendix is organized as follows.

\begin{table}[h]
\centering
\begin{tabular}{@{}l l@{}}
\toprule
Section & Content \\
\midrule
Appendix A & Proof of Lemma~\ref{lem1} (Feasibility) \\
Appendix B & Proof of Lemma~\ref{lem:P1-LP} (Integrality and TU) \\
Appendix C & Proof of Theorem~\ref{lem:P1-MAPF} (Collision-free transports) \\
Appendix D & Proof of Lemma~\ref{lem:mm} (Minimum makespan) \\
Appendix E & Proof of Lemma~\ref{lem_KL} (KL decomposition) \\
Appendix F & Proof of Lemma~\ref{lem_KL_Gibbs} (Gibbs kernel reduction) \\
Appendix G & Sinkhorn-MAPF algorithm \\
Appendix H & Detailed Experiments \\
\quad \ref{illustrative} & Illustrative Experiments (Figs.~\ref{a_fig1}--\ref{a_fig6}) \\
\quad \ref{scaling} & Scaling Study (162 runs, Table~\ref{tab:scaling}, Figs.~\ref{fig:scaling_tradeoff}--\ref{fig:scaling_breakdown}) \\
\quad \ref{sensitivity} & Parameter Sensitivity (260 runs, Table~\ref{tab:sinkhorn_time}, Fig.~\ref{fig:sensitivity}) \\
\quad \ref{nu_costs} & Non-uniform Costs (24 runs, Fig.~\ref{fig:nonuniform}) \\
\quad \ref{comps} & Baseline Comparison (15 runs, Table~\ref{tab:baseline}) \\
\bottomrule
\end{tabular}
\end{table}

\section{Proof of Lemma~\ref{lem1}}\label{p_lem1}
\begin{proof}
Consider a Markov chain~$\mathrm M$ whose state-space consists of all $K$-dimensional~$\{0,1\}$ vectors with \textit{exactly}~$N$ ones, representing the configuration space of the robots on the graph~$\mc G$. A transition from state~${\mb q\in \mrm M}$ to~${\mb q'\in\mrm M}$ is allowed (and given a non-zero probability), whenever~$\mb q'$ is obtained from~$\mb q$ by moving a single robot along a traversable edge~${(i,j)\in\mc E}$ to a neighboring location~$j$ that is unoccupied by a robot, or by keeping all robots at their current vertices. In other words, the state~$\mb q$ has outgoing (positive probability) edges to every state that can be achieved with exactly one robot's valid move and also to itself (no robot moves). Because ${\mc G}$ is connected, the robots are indistinguishable, 
a standard result from the pebble motion literature states that the configuration space of~$\mrm M$ is connected~\cite{KMS84}. Additionally, since~$\mc G$ is finite, under the imposed transition probabilities,~$\mrm M$ is irreducible and every state in~$\mrm M$ is recurrent, i.e., from any state~${\mb q_1 \in \mrm  M}$, one can reach any other state~${\mb q_2\in\mrm M}$, by a finite sequence of single-robot moves along the edges of~$\mc G$, with probability~$1$. Therefore, the terminal configuration~$\bnu$ is reachable from~$\bmu$ in a finite number of steps~$\bar T$. Clearly, a transport sequence~$\{\Pi_t\}_{t=1}^{\bar T}$ that encodes these Markovian transitions satisfies all feasibility constraints in~$\mc F$, and the theorem follows, since~$C_t$ is finite on~$\mc E$.
\end{proof}

\section{Proof of Lemma~\ref{lem:P1-LP}}\label{plem:P1-LP}
\begin{proof}
We first show (i): In order to establish the integral solution guarantee by the LP in~\textbf{P1}, we show that all constraints in~$\mc F$ can be written in the form of a node-arc incidence matrix of an augmented directed graph~$\bar{\mc G}$ that we construct as follows. Consider the time-expanded graph~$\mc G_{0}^{\bar T}$ over the horizon~$\bar T$, i.e., with vertices~$(i,t)$, for each~${i\in\mc V}$ and time~${t=0,\dots,\bar T}$, and arcs~${(i,t-1) \to (j,t)}$, for~${t=1,\ldots,\bar T}$, whenever the move~$i\to j$ is allowed by~$\mc G$. We further apply the classical node-splitting construction of Ford and Fulkerson~\cite{FordFulkerson1962,AhujaMagnantiOrlin1993} to augment~$\mc G_{0}^{\bar T}$ as follows. For every intermediate vertex~${(i,t),{t=1,\ldots,\bar {T}-1}}$, split it into two vertices
\[
(i,t)'' \quad\text{(in-node)}, \qquad (i,t)' \quad\text{(out-node)},
\]
and add an internal arc
\[
(i,t)'' \to (i,t)'. 
\]
All arcs that originally entered~$(i,t)$, now enter~$(i,t)''$, traverse the internal arc~${(i,t)''\to (i,t)'}$; and all arcs that originally left~$(i,t)$ now emanate from~$(i,t)'$. Thus, at each time~${t=1,\ldots,\bar {T}-1}$, a transport entering any vertex~$i$ is forced over~$(i,t)''\ra(i,t)'$ before it can move on to the next time layer~$t+1$. Let~$\bar{\mc G}$ denote the resulting augmented graph with~$(\bar T+1)$ layers of~$\mc G$, with a total of~${2K\bar T (= K+2K(\bar T-1)+K)}$ vertices due to node-splitting, and with~$K(\bar T-1)$ internal arcs added to the arcs in~$\mc G_0^{\bar T}$. 

Let $\bar\bpi$ stack all arc transport variables on $\bar{\mc G}$. It can be verified that all constraints in~$\mathbf{P1}$ can now be written as~${\bar A\bar\bpi\leq \bar {\mb b}}$, ${-\bar A\bar\bpi\leq -\bar {\mb b}}$, ${I\bar\bpi\leq \mb 1}$, and~${-I \bar\bpi\leq 0}$, where~${\bar{\mb b}\in\{0,\pm 1\}^{2K\bar T}}$ is such that it is~${-\mu_i}$ for each node~$(i,0)$ in layer~$0$,~$\nu_j$ for each node~$(j,\bar T)$ in the last layer~$\bar T$, and~$0$ on every intermediate split node. Define
\[
\hat A :=
\left[
\begin{array}{r}
\bar A\\ -\bar A\\ I\\ -I
\end{array}
\right],
\qquad
\hat{\mb b} :=
\left[
\begin{array}{r}
\bar {\mb b}\\ -\bar {\mb b}\\ \mb 1\\ \mb 0
\end{array}
\right].
\]
Since incidence matrices of directed graphs are totally unimodular (TU),~$\bar A$ is TU; see Chapter 19 in~\cite{schrijver1998}. Consequently, the overall constraint matrix~$\hat A$ is TU, because TU is preserved under row sign changes and appending rows of~$\pm I$. Since~$\hat A$ is TU, the polyhedron~${\{\bar\bpi:\hat A\bar\bpi\leq\hat{\mb b}\}}$ has integral extreme points. By the TUM theorem (Theorem 19.1 and Corollary 19.1a in~\cite{schrijver1998}), solving the LP corresponding to~\textbf{P1} on the augmented~$\bar{\mc G}$ admits an optimal basic solution that is integral. From standard network flow arguments~\cite{FordFulkerson1962,AhujaMagnantiOrlin1993}, any feasible flow in the augmented network yields an equivalent feasible flow in the original time-expanded network~$\mc G_0^{\bar T}$ that respects the node-capacity constraints, and conversely. Thus,~$\{\Pi^*_t\}_{t=1}^{\bar T}$ can be chosen as an optimal basic solution and is therefore integral, and~(i) follows. The rest of the lemma follows from the standard arguments in linear programming; see also Section~\ref{sec:comp} on precise complexity arguments.
\end{proof}

\section{Proof of Theorem~\ref{lem:P1-MAPF}}\label{p_lem:P1-MAPF}
\begin{proof}
We first show~(i). Robots may collide in the following scenarios: 
\begin{inparaenum}[(a)]
    \item at two intersecting edges~${i\to m}$ and~${j\to k}$, for distinct~$i,j,k,m$; or,
    \item a robot traversing~${i\to j}$ may collide with stationary robots at nearby vertices; or,
    \item when two robots travel to the same destination; or,
    \item at a bidirectional edge~${i\leftrightarrow j}$; or,
    \item in a flow cycle~${i_1\to i_2\to\cdots\to i_k\to i_1}$, for~${k\geq 3}$.
\end{inparaenum}
Clearly,~(a) and~(b) are ruled out because of Assumption~\ref{assump1}(\ref{conf_dist}), and~(c) is ruled out by the vertex-capacity constraint. To show that~(d) does not appear in the transport, we proceed as follows. Suppose, on the contrary, that there exists a min-cost transport~$\{\Pi_t\}_{t=1}^{\bar T}$, such that at time~$t$ and two distinct vertices~${i\neq j}$, we have~${\pi_{ij,t}=1}$ and~${\pi_{ji,t}=1}$, i.e., two robots simultaneously traverse the edges ${i\to j}$ and ${j\to i}$. Since self-loops ${i\to i}$ and ${j\to j}$ are feasible by Assumption~\ref{assump1}(i), define an alternate plan~$\{\tilde\Pi_t\}_{t=1}^{\bar T}$ that is exactly the same as the optimal~$\{\Pi_t\}_{t=1}^{\bar T}$, except for these swaps, which are replaced by waiting moves, i.e.,
\[
\tilde\pi_{ii,t}=1,\quad \tilde\pi_{jj,t}=1,\quad 
\tilde\pi_{ij,t}=0,\quad \tilde\pi_{ji,t}=0.
\]
This modification preserves feasibility, since the row and column sums of the~$t$-th slice~$\Pi_t$ remain unchanged and hence the distributions~$\mb q_{t-1}$ and~$\mb q_t$ are identical for both~$\{\Pi_t\}_{t=1}^{\bar T}$ and~$\{\tilde\Pi_t\}_{t=1}^{\bar T}$. By Assumption~\ref{assump1}(\ref{conf_self}),
\[
c_{ii,t}+c_{jj,t} < c_{ij,t}+c_{ji,t},
\]
and therefore
\[
\langle \tilde\Pi_t, C_t\rangle < \langle \Pi_t, C_t\rangle,
\]
contradicting the optimality of~$\{\Pi_t\}_{t=1}^{\bar T}$. With a 
similar argument, a~$k$-cycle leaves the configuration unchanged (because the robots are anonymous) at the price of~$k$ moves and is therefore also suboptimal. Hence, no min-cost solution contains a collision and~(i) follows;~(ii) follows consequently, and~(iii) is guaranteed by the terminal feasibility of~\textbf{P1}. 
\end{proof}

\section{Proof of Lemma~\ref{lem:mm}}\label{p_lem:mm}
\begin{proof}
That~\textbf{P1} results into integral transports is already established in Lemma~\ref{lem:P1-LP}. Let~${\Pi_t'\in\{0,1\}^{K\times K}},{t=1,\ldots,\bar T}$, be an optimal solution of~\textbf{P1} obtained with the cost structure~$C_t$ described in Assumption~\ref{assump:time-sep}, i.e.,~$\{\Pi_t'\}_{t=1}^{\bar T}$ minimizes the transport cost~$\sum_t\langle \Pi_t',C_t,\rangle$, over the horizon~${t=0,1,\ldots,\bar T}$, and let~${T'\le\bar T}$ be its makespan, i.e., there is no motion after~$T'$. Let~$T^*$ denote the minimum makespan over all feasible transports, and fix any feasible integral transport $\{\Pi_t^*\}_{t=1}^{\bar T}$, whose makespan is $T^*$
(extended to the horizon $\bar T$ by waits at targets). Clearly, we have that~$T'\ge T^*$ and we need to show that~${T'=T^*}$.

Suppose, on the contrary, that $T'>T^*$. Since $T'$ is the makespan of $\{\Pi_t'\}_t$, there exists at least one non-wait transport at time $T'$, i.e.,~${\pi'_{ij,T'}=1}$, for some ${i\neq j}$, with cost
\[
c_{ij,T'} = B^{T'}\tilde c_{ij} \;\ge\; B^{T'}\,\tilde c_{\min},
\]
from Assumption~\ref{assump:time-sep}, and therefore the min-cost transport~$\Pi_t'$ satisfies:
\[
\sum_{t=1}^{\bar T}\langle \Pi_t',C_t\rangle
\;\ge\; \langle \Pi_{T'}',C_{T'}\rangle
\;\ge\; c_{ij,T'} \;\ge\; B^{T'}\,\tilde c_{\min} \geq B^{T^*+1}\,\tilde c_{\min},
\]
since~${T'\ge T^*+1}$. Similarly,~$\{\Pi_t^*\}_t$ has no non-wait motion after time $T^*$, so its cost is supported only on times~$t\le T^*$, and using~$\pi^*_{ij,t}\le 1$, we have
\[
\sum_{t=1}^{\bar T}\langle \Pi_t^*,C_t\rangle
=
\sum_{t=1}^{T^*}\langle \Pi_t^*,C_t\rangle
\;\le\;
\sum_{t=1}^{T^*}\sum_{i\to j\in\mc E} c_{ij,t}
=
\sum_{t=1}^{T^*} B^t \sum_{i\to j\in\mc E}\tilde c_{ij}.
\]
By Assumption~\ref{assump:time-sep}, it follows that
\[
B^{T^*+1}\,\tilde c_{\min}
\;>\;
\frac{B^{T^*+1}}{B-1}\sum_{i\to j\in\mc E}\tilde c_{ij}
\;\ge\;
\sum_{t=1}^{T^*} B^t \sum_{i\to j\in\mc E}\tilde c_{ij}.
\]
Hence,~$\sum_{t=1}^{\bar T}\langle \Pi_t',C_t\rangle
>
\sum_{t=1}^{\bar T}\langle \Pi_t^*,C_t\rangle$, contradicting the optimality
of~$\{\Pi_t'\}_t$ and we conclude that~$T'=T^*$.
\end{proof}

\section{Proof of Lemma~\ref{lem_KL}}\label{p_lem_KL}

\begin{proof}
Since~${\mb G}$ is Markovian, it admits the following factorization:
\begin{equation*}
    \mathbf{G}_{i_0, \dots, i_T} = [\mb g_0]_{i_0} \prod_{t=1}^T \frac{[\mathbf{G}_t]_{i_{t-1}, i_t}}{[\mb g_{t-1}]_{i_{t-1}}},
\end{equation*}

Recalling the definition of~$\mathrm{KL}$ divergence, write
\begin{equation*}
    \mathrm{KL}(\mathbf{P} \| \mathbf{G}) = \sum_{i_0, \dots, i_T} \mathbf{P}_{i_0, \dots, i_T} \log \frac{\mathbf{P}_{i_0, \dots, i_T}}{\mathbf{G}_{i_0, \dots, i_T}}.
\end{equation*}
Substituting the Markov factorizations in~\eqref{KLtensor} yields
\begin{align*}
    \mathrm{KL}(\mathbf{P} \| \mathbf{G})
    &= \sum_{i_0, \dots, i_T} \mathbf{P}_{i_0, \dots, i_T}
    \Bigg[
    \log \frac{\tfrac{1}{N}[\mb q_0]_{i_0}}{[\mb g_0]_{i_0}}
    + \sum_{t=1}^T
    \left(
    \log \frac{\tfrac{1}{N}[\Pi_t]_{i_{t-1},i_t}}{[\mathbf{G}_t]_{i_{t-1},i_t}}
    - \log \frac{\tfrac{1}{N}[\mb q_{t-1}]_{i_{t-1}}}{[\mb g_{t-1}]_{i_{t-1}}}
    \right)
    \Bigg].
\end{align*}
Distributing the sums and marginalizing $\mathbf{P}$ yields~\eqref{fullKL} and the proof follows.
\end{proof}

\section{Proof of Lemma~\ref{lem_KL_Gibbs}}\label{p_lem_KL_Gibbs}

\begin{proof}
Recall that~$\mb q_t$, for any~${t=0,1,\dots,T}$, is a~$\{0,1\}$ vector with exactly~$N$ ones. Therefore, 
\begin{align*}
    \mathrm{KL}(\tfrac{1}{N}\mb q_t \| \mb g_t) = \sum_{i=1}^K \tfrac{1}{N}[\mb q_t]_i \log\frac{\tfrac{1}{N}[\mb q_t]_i}{[\mb g_t]_i} = \sum_{i\in \mc V_t} \frac{1}{N} \log\frac{\frac{1}{N}}{[\mb g_t]_i} 
    := \kappa_t,
\end{align*}
for all~$t$, since~$\mb g_t$ are given fixed references. From~\eqref{fullKL}, we have
\begin{align*}
    \mathrm{KL}(\mathbf{P} \| \mathbf{G})
    &= \sum_{t=1}^T \mathrm{KL}(\tfrac{1}{N}\Pi_t \| \mathbf{G}_t)
    + \kappa,
\end{align*}
where~$\kappa$ encodes all~$\kappa_t$'s and is independent of the transport variables. 
Fixing~$t$ and expanding the KL term, we get:
\begin{align*}
\mathrm{KL}(\tfrac{1}{N}\Pi_t\|{\mb G}_t)
&=
\sum_{i,j}\tfrac{1}{N}\Pi_{ij,t}\log\frac{\tfrac{1}{N}\Pi_{ij,t}}{{\mb G}_{ij,t}}\\
&=
\sum_{i,j}\tfrac{1}{N}\Pi_{ij,t}\log \tfrac{1}{N}\Pi_{ij,t}
-\sum_{i,j}\tfrac{1}{N}\Pi_{ij,t}\log \frac{g_{ij,t}}{z_t}\\
&=\sum_{i,j}\tfrac{1}{N}\Pi_{ij,t}\log \tfrac{1}{N}\Pi_{ij,t}
+\sum_{i,j}\tfrac{1}{N}\Pi_{ij,t} \frac{c_{ij,t}}{\varepsilon}
+\sum_{i,j}\tfrac{1}{N}\Pi_{ij,t}\log z_t,\\
&=\tfrac{1}{N}\sum_{i,j}\Pi_{ij,t}\log \Pi_{ij,t} + \log \tfrac{1}{N}
+\tfrac{1}{N}\sum_{i,j}\Pi_{ij,t} \frac{c_{ij,t}}{\varepsilon}
+\log z_t,
\end{align*}
which yields the desired result after noting that~$\sum_{i,j}\Pi_{ij,t}=N$.
\end{proof}

\section{Sinkhorn-MAPF}\label{sinkMAPF}
Recall that~\textbf{P2} is an entropic regularization of the LP in~\textbf{P1}. This regularization enables the use of efficient Sinkhorn-type algorithms, which scale to problem sizes where the original LP may become computationally intractable. We leverage these ideas to design an algorithm for the time-expanded transport formulation. Recall~${\mb G}_t$ from Lemma~\ref{lem_KL}. It is well known (see e.g., Lemma~2 in~\cite{cuturi2013sinkhorn}) that the minimizer of~\textbf{P2} is obtained at 
\begin{align*}
    \widetilde\Pi_t = \diag(\mb u_t)\, {\mb G}_t\, \diag(\mb v_t),\qquad t=1,\ldots,T,
\end{align*}
for some positive scaling vectors~${\mb u_t,\mb v_t \in \mbb R^K_+}$, where~$\diag(\mb u_t)$ is the diagonal matrix formed by the vector~$\mb u_t$. The Sinkhorn algorithm is an iterative scheme to find the scaling vectors~$\mb u_t,\mb v_t$ such that~$\widetilde\Pi_t$, as written above, satisfies the constraint set~$\mc F$. We describe this procedure next. Given the element-wise expansion of the above, i.e.,
\begin{align*}
    \widetilde\pi_{i,j,t} = \mb u_{i,t}\, {\mb G}_{ij,t}\, \mb v_{j,t},\qquad t=1,\ldots,T,~i,j=1,\ldots,K, 
\end{align*}
we note that, for each~$i,t$, 
\begin{align*}
    [\widetilde\Pi_t\mb 1]_i = \sum_{j}\widetilde\pi_{ij,t} = \mb u_{i,t}\, \sum_{j}{\mb G}_{ij,t} \mb v_{j,t} = \mb u_{i,t}\, [{\mb G}_{t} \mb v_t]_i,
\end{align*}
and similarly, for each~$j,t$,
\begin{align*}
    [\widetilde\Pi_t^\top\mb 1]_j = \sum_{i}\widetilde\pi_{ij,t} = \mb v_{j,t}\sum_{i} \mb u_{i,t}\, {\mb G}_{ij,t}  = \mb v_{j,t}\, [{\mb G}_{t}^\top \mb u_t]_j.
\end{align*}
We can now write all constraints in~$\mc F$ in terms of the scaling vectors as follows:
\begin{align*}
\mb q_{t-1} &= \widetilde\Pi_t\mb 1 = \mb u_t \odot [{\mb G}_{t}\mb v_t],\\
\mb q_{t} &= \widetilde\Pi_t^\top \mb 1 = \mb v_t \odot [{\mb G}_{t}^\top \mb u_t],
\end{align*}
where~$\odot$ denotes the element-wise product. Since~${\widetilde\Pi_t^\top \mb 1 = \mb q_t = \widetilde\Pi_{t+1} \mb 1}$, we have the dynamic consistency equation:
\begin{align}
\mb v_t \odot [G_t^\top \mb u_t] = \mb u_{t+1} \odot [G_{t+1}\mb v_{t+1}],\qquad t=1,\ldots,T-1.
\end{align}
The boundary conditions are given by
\begin{align}
\mb q_{0} = \widetilde\Pi_1\mb 1 &= \mb u_1 \odot [{\mb G}_1\mb v_1] = \bmu,\\
\mb q_{T} = \widetilde\Pi_T^\top \mb 1 &= \mb v_T \odot [{\mb G}_T^\top \mb u_T] = \bnu.
\end{align}
We now define the Sinkhorn-MAPF algorithm below\footnote{It is not uncommon to implement the resulting equations in the log domain to account for the numerical instabilities when~$\varepsilon\to0$, see e.g.,~\cite{Schmitzer2019Stabilized}.}:

\textbf{Initialization:} Set ${\mb u_t^{(0)} = \mb v_t^{(0)} = \mathbf{1}}$ for all ${t=1,\ldots,T}$ and then normalize the boundary slices only on the prescribed supports, i.e.,
\begin{align*}
[\mb u_1^{(0)}]_i \gets \frac{[\bmu]_i}{[G_1 \mb v_1^{(0)}]_i},~\text{for all } i \in \mathrm{supp}(\bmu),\qquad
[\mb v_T^{(0)}]_j \gets \frac{[\bnu]_j}{[G_T^\top \mb u_T^{(0)}]_j},~\text{for all } j \in \mathrm{supp}(\bnu);
\end{align*}
where we adopt the convention throughout that the ratio is set to~$0$ if the denominator is zero unless otherwise stated.

\textbf{Sinkhorn sweeps:} For~$\tau=0,\dots,{\bar\tau-1}$, repeat the following: 
\begin{enumerate}
\item \textit{Starting marginal projection: }Given~$\mb v_1^{(\tau)}$, update~$\mb u_1^{(\tau+1)}$ to conform with the starting marginal~$\bmu$, using
\begin{align*}
    [\mb u_1^{(\tau+1)}]_i = 
    \begin{cases}
        [\bmu]_i/[G_1\mb v_1^{(\tau)}]_i,& \mbox{if }[\bmu]_i=1,\\
        [\mb u_1^{(\tau)}]_i, & \mbox{if }[\bmu]_i=0,
    \end{cases}
\end{align*}
and~$\mb u_t^{(\tau+1)}=\mb u_t^{(\tau)},$ for~${t>1}$.

\item \textit{Forward consistency projection:} 
For each~$t=1,\ldots,{T-1}$, enforce the dynamic consistency between~$t$ and~$t+1$. 
First compute the current scaling constants
\begin{align*}
    [\mb q_{t}^{\text{out}}]_i &= [\mb v_t^{(\tau)}]_i \cdot [G_t^\top \mb u_t^{(\tau+1)}]_i,\\
    [\mb q_{t}^{\text{in}}]_i  &= [\mb u_{t+1}^{(\tau+1)}]_i \cdot [G_{t+1}\,\mb v_{t+1}^{(\tau)}]_i,
\end{align*}
and define a correction factor
\begin{align*}
    [\boldsymbol\gamma_t]_i = 
    \begin{cases}
        \sqrt{ \dfrac{[\mb q_{t}^{\text{in}}]_i}{[\mb q_{t}^{\text{out}}]_i} },& \text{if }[\mb q_{t}^{\text{out}}]_i>0,\\[6pt]
        1,& \text{otw.}
    \end{cases}
\end{align*}
to update
\begin{align*}
    [\mb v_t^{(\tau+1)}]_i     &= [\mb v_t^{(\tau)}]_i\, [\boldsymbol\gamma_t]_i^\eta,\\
    [\mb u_{t+1}^{(\tau+1)}]_i &= [\mb u_{t+1}^{(\tau+1)}]_i\, [\boldsymbol\gamma_t]_i^{-\eta}.
\end{align*}
The multiplicative factor~$[\boldsymbol\gamma_t]_i^\eta$ moves 
$[\mb q_{t}^{\text{out}}]_i$ and $[\mb q_{t}^{\text{in}}]_i$ closer at each node~$i$.
If ${\eta=1}$ (and no other constraints are active), 
then the update enforces $[\mb q_{t}^{\text{out}}]_i = [\mb q_{t}^{\text{in}}]_i$ in one step; 
for $\eta\in(0,1)$, it is a damped projection.

\item \textit{Terminal marginal projection: }Given~$\mb u_T^{(\tau+1)}$, update~$\mb v_T^{(\tau+1)}$ to conform with the terminal marginal~$\bnu$, using
\begin{align*}
    [\mb v_T^{(\tau+1)}]_j = 
    \begin{cases}
        [\bnu]_j/[G_T^\top\mb u_T^{(\tau+1)}]_j,& \mbox{if }[\bnu]_j=1,\\
        [\mb v_T^{(\tau)}]_j, & \mbox{if }[\bnu]_j=0.
    \end{cases}
\end{align*}
For~${t<T}$, we already have~$\mb v_t^{(\tau+1)}$ from the consistency update above.

\item \textit{Backward consistency projection:} For each $t=T-1,\ldots,1$, enforce the dynamic consistency between~${t+1}$ and~$t$, i.e., in reverse time.
Compute
\begin{align*}
    [\bar{\mb q}_{t}^{\mathrm{out}}]_i &= 
        [\mb u_{t+1}^{(\tau+1)}]_i \cdot 
        [G_{t+1}\,\mb v_{t+1}^{(\tau+1)}]_i,\\
    [\bar{\mb q}_{t}^{\mathrm{in}}]_i &= 
        [\mb v_{t}^{(\tau+1)}]_i \cdot 
        [G_t^\top \mb u_{t}^{(\tau+1)}]_i,
\end{align*}
and
\begin{align*}
    [\bar\gamma_t]_i = 
    \begin{cases}
        \sqrt{\dfrac{[\bar{\mb q}_{t}^{\mathrm{in}}]_i}{[\bar{\mb q}_{t}^{\mathrm{out}}]_i}}, 
        & [\bar{\mb q}_{t}^{\mathrm{out}}]_i>0,\\[6pt]
        1, & \text{otw.}
    \end{cases}
\end{align*}
with updates
\begin{align*}
    [\mb u_{t+1}^{(\tau+1)}]_i &= 
        [\mb u_{t+1}^{(\tau+1)}]_i\, [\bar\gamma_t]_i^\eta,\\
    [\mb v_{t}^{(\tau+1)}]_i &= 
        [\mb v_{t}^{(\tau+1)}]_i\, [\bar\gamma_t]_i^{-\eta}.
\end{align*}
As in the forward sweep, the multiplicative factor~$[\boldsymbol\gamma_t]_i^\eta$  reduces the discrepancy between $[\mb q_{t}^{\text{out}}]_i$ and~$[\mb q_{t}^{\text{in}}]_i$ at each node~$i$, but now propagates corrections backward in time. 
\end{enumerate}

\textbf{Final regularized transports:} The algorithm is terminated at~${\tau=\bar\tau-1}$, resulting in
\begin{align*}
    \widetilde\Pi_t = \diag(\mb u_t^{(\bar\tau)})\, G_{t}\, \diag(\mb v_t^{(\bar\tau)}),\qquad t=1,\ldots,T,
\end{align*}

\section{Detailed Experiments}\label{extra_sim}
In this section, we provide a detailed set of experiments. We consider square grids of size~${K=W^2}$ with~$N$ robots,~$M$ targets, and~$O$ obstacles, such that~${N=M}$. While the proposed framework applies to arbitrary graphs, grids are chosen for ease of visualization and reproducibility. The time horizons are chosen such that the corresponding LPs are feasible; smaller values of~$T$ are only chosen for convenience as it is easier to display the corresponding trajectories. Unless stated otherwise, the cost structure is such that waiting cost at targets is~${c_{ii}=0,i\in\mc M}$, waiting at non-target is~${c_{jj}}=0.5,j\notin\mc M$, while the move cost is~${c_{i\neq j}=1},$ for all~${{i\to j \in\mc E}}$. 

\subsection{Illustrative Experiments}\label{illustrative}

We first provide some basic experiments to demonstrate the main ideas. The experiment setup is described in the caption of each figure and the corresponding LPs are solved using the standard linear programming suite in the HiGHS solver from the SciPy's linprog library; all variables are continuous by default. Figs.~\ref{a_fig1},~\ref{a_fig2}, and~\ref{a_fig3} solve~\textbf{P1} over a given time horizon~$T$; we note that arbitrarily placed obstacles alter the connectivity and diameter of the underlying graph, and it no longer remains a regular grid. Figs.~\ref{ms_fig1}-\ref{ms_fig4} elaborate the min-cost versus minimum makespan nature of the transports returned by~\textbf{P1}. Figs.~\ref{a_fig4},~\ref{a_fig5}, and~\ref{a_fig6} solve~\textbf{P2} with the Sinkhorn iterations of Appendix~\ref{sinkMAPF} for the following~${(\varepsilon,\bar\tau)}$ pairs, where~$\bar\tau$ are the total number of Sinkhorn iterations: $\{(0.2,50), (0.5,10), (50,5)\}$; then prune the resulting graph and project on the integral~\textbf{P3} LP. As known for Sinkhorn iterations,~\textbf{P2} may require a larger~$\bar\tau$ as~$\varepsilon$ decreases; however, a useful shadow transport is typically obtained within a small number of iterations.

In the following subsections (Sections~\ref{scaling}--\ref{comps}), we conduct a large-scale study of the proposed approaches. We ran over~$460$ experiments on graphs ranging from~${K=2{,}500}$ to~${K=22{,}500}$ vertices, i.e., from~$369$K to~$3.4$M LP variables, using the Gurobi LP solver with continuous variables; all \textbf{P1} and \textbf{P3} solutions are verified integral as guaranteed by the TU of the underlying LP.

\subsection{Scaling Study}\label{scaling}

We evaluate the runtime scaling of~\textbf{P1} and the~\textbf{P2}+\textbf{P3} pipeline on a 2022 MacBook; all LPs use continuous variables. We consider square grids of size~${K=W^2}$ with no obstacles, at~$5\%$ robot density~(${N=0.05K}$), over a horizon~${T=30}$. The cost structure follows the convention adopted in Assumption~\ref{assump1}. The Sinkhorn parameters are~${\varepsilon=0.2}$ and~${\bar\tau=150}$ sweeps. For each grid size, we generate~$14$--$25$ random instances and report averaged results. On grid graphs, each vertex has at most~$5$ neighbors ($4$ cardinal directions plus a self-loop), so the number of variables in~\textbf{P1} is~${|\mc E|T \approx 5KT}$. For example, at~${K=22{,}500}$ with~${T=30}$:~${|\mc E|T = 3{,}357{,}000}$. 

Table~\ref{tab:scaling} summarizes the results across~$8$ grid sizes, totaling~$162$ runs. For each grid, we report the number of~\textbf{P1} variables~(${=|\mc E|T}$), the average~\textbf{P1} solve time, the average~\textbf{P2}+\textbf{P3} solve times (Sinkhorn + feasible~\textbf{P3}), the resulting speedup, cost gap relative to~\textbf{P1}, and the percentage of edges retained from the shadow transport (which also provides the average number of~\textbf{P3} variables, kept after pruning).

\begin{table}[!htb]
\centering
\caption{Scaling of~\textbf{P1} and~\textbf{P2}+\textbf{P3} across grid sizes at~$5\%$ robot density~(${N=0.05K}$),~${T=30}$,~${\varepsilon=0.2}$. All times are averaged and are reported in seconds~(s). Every solution across all~$162$ runs is verified integral.}
\label{tab:scaling}
\vspace{0.3em}
\small
\begin{tabular}{@{}c c c c c c c c c@{}}
\toprule
$K$ & $N$ & Runs & \textbf{P1} vars & \textbf{P1} (s) & \textbf{P2}+\textbf{P3} (s) & Speedup & Gap (\%) & Kept (\%) \\
\midrule
$2{,}500$   & $125$   & 20 & $369$K   &   15 &    4 & $3.6\times$ & 8.0 & 32 \\
$5{,}625$   & $281$   & 25 & $835$K   &   55 &   11 & $5.0\times$ & 8.7 & 35 \\
$8{,}100$   & $405$   & 19 & $1.2$M   &  103 &   18 & $5.7\times$ & 9.0 & 37 \\
$10{,}000$  & $500$   & 24 & $1.5$M   &  132 &   26 & $5.1\times$ & 5.9 & 41 \\
$13{,}225$  & $661$   & 14 & $2.0$M   &  193 &   33 & $5.8\times$ & 8.1 & 39 \\
$15{,}625$  & $781$   & 23 & $2.3$M   &  257 &   48 & $5.3\times$ & 5.9 & 43 \\
$19{,}600$  & $980$   & 18 & $2.9$M   &  364 &   62 & $5.8\times$ & 7.2 & 42 \\
$22{,}500$  & $1{,}125$ & 19 & $3.4$M &  478 &   67 & $7.1\times$ & 6.5 & 40 \\
\bottomrule
\end{tabular}
\end{table}

Fig.~\ref{fig:scaling_runtime} (in Section~\ref{sec:exp}, left panel) plots the~\textbf{P1} and~\textbf{P2}+\textbf{P3} solve times against the number of grid vertices~$K$, with power-law fits of the form~${aK^p + b}$ over all~$162$ individual runs. The fitted models are
\[
\textbf{P1:}\quad \textrm{solve time} = 2.26\!\times\!10^{-5}\, K^{1.68} + 10.2 \quad (R^2 = 0.96),
\]
\[
\textbf{P2}{+}\textbf{P3:}\quad \textrm{solve time} = 7.40\!\times\!10^{-4}\, K^{1.15} - 2.97 \quad (R^2 = 0.74).
\]
We note that~\textbf{P1} time grows sub-quadratically in~$K$, while the~\textbf{P2}+\textbf{P3} oracle time scales almost linearly. The speedup, shown in Fig.~\ref{fig:scaling_runtime} (right panel), grows with problem size from~$3.6\times$ at~${K=2{,}500}$ to~$7.1\times$ at~${K=22{,}500}$, while the cost gap remains consistently below~$10\%$ (median~$6.4\%$). Since~\textbf{P1} already provides the exact optimum, the~\textbf{P2}+\textbf{P3} pipeline offers a practical speed-quality tradeoff: a~$5$--$7\times$ speedup at under~$10\%$ cost degradation.

Fig.~\ref{fig:scaling_tradeoff} examines the cost-gap-versus-speedup tradeoff from two perspectives. The left panel plots the cost gap against the speedup for each of the~$162$ individual runs, colored by the number of vertices~$K$. The cluster structure confirms that larger instances achieve higher speedups at comparable or lower cost gaps, i.e., the shadow-based pruning becomes more effective as~$K$ grows. The right panel plots the average cost gap and speedup jointly against~$K$ on a dual axis. The cost gap remains stable between~$5$--$9\%$ across all grid sizes, while the speedup increases steadily from~$3.6\times$ to~$7.1\times$.

Fig.~\ref{fig:scaling_breakdown} provides a complementary view of the pipeline. The left panel decomposes the average solve time into~\textbf{P1} (gray), Sinkhorn (blue), and~\textbf{P3} LP (red) components. The connected dots trace the scaling shape: the~\textbf{P1} curve grows subquadratically while the~\textbf{P2}+\textbf{P3} curve remains nearly flat. The italic percentages above each stacked bar indicate the Sinkhorn share of the~\textbf{P2}+\textbf{P3} oracle time, which decreases from~$63\%$ at~${K=2{,}500}$ to~$34\%$ at~${K=22{,}500}$; at large scale, the LP solve dominates. The right panel shows the variable reduction achieved by shadow-based pruning:~\textbf{P3} consistently operates on~$32$--$43\%$ of the~\textbf{P1} variables.

\begin{figure}[!htb]
    \centering
    \includegraphics[width=0.75\textwidth]{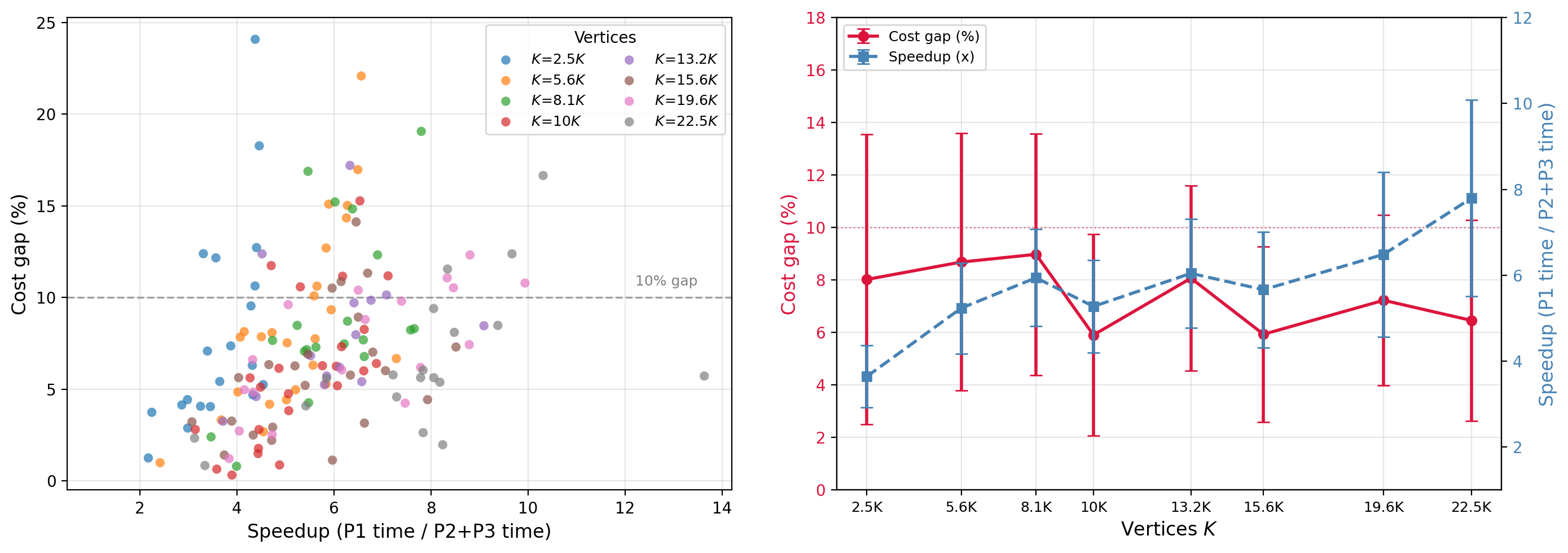}
    \caption{Cost gap and speedup tradeoff across~$162$ scaling runs at~$5\%$ robot density,~${T=30}$. (Left)~Cost gap~(\%) versus speedup for each individual run, colored by the number of vertices~$K$; the dashed line marks the~$10\%$ gap threshold. (Right)~Average cost gap~(\%, left axis, red) and speedup~(right axis, blue) versus~$K$, with~${\pm 1}$ standard deviation error bars; the dotted line marks the~$10\%$ gap threshold.}
    \label{fig:scaling_tradeoff}
\end{figure}

\begin{figure}[!htb]
    \centering
    \includegraphics[width=0.75\textwidth]{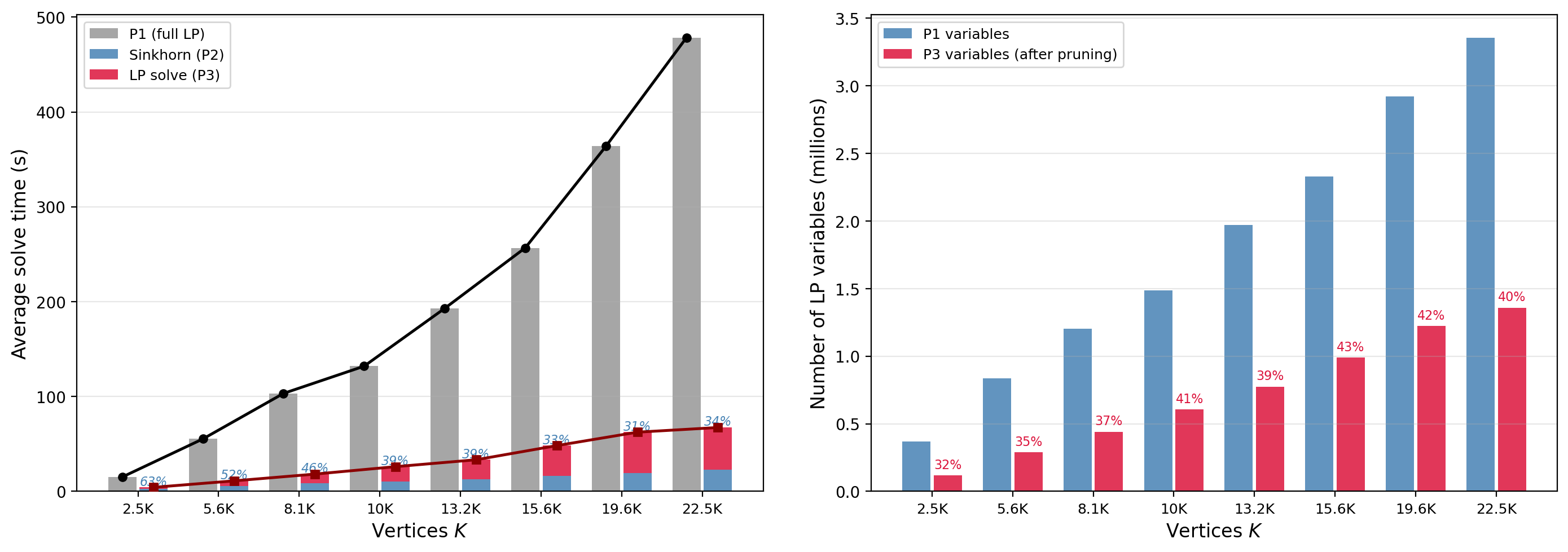}
    \caption{Pipeline decomposition across~$8$ grid sizes at~$5\%$ robot density,~${T=30}$; all values are averages over~$14$--$25$ independent instances. (Left)~Solve time of~\textbf{P1} (gray bars) versus the Sinkhorn~(\textbf{P2}, blue) and LP solve~(\textbf{P3}, red) components of the pipeline; connected dots trace the scaling shape of each; italic percentages show the Sinkhorn share of the~\textbf{P2}+\textbf{P3} time. (Right)~Number of LP variables in~\textbf{P1} (blue) versus~\textbf{P3} after pruning (red); percentages indicate the fraction of~\textbf{P1} variables retained.}
    \label{fig:scaling_breakdown}
\end{figure}

\subsection{Sinkhorn and Pruning Parameter Sensitivity}\label{sensitivity}
In this section, we study the sensitivity of the scalable~\textbf{P2}+\textbf{P3} pipeline to the corresponding parameters~$\varepsilon$ and~$\lambda$. We sweep~${\varepsilon\in\{0.1,0.2,0.5,1.0,5.0\}}$ and~${\lambda\in\{0,0.5,1.0,5.0\}}$ over~$13$ independent random instances on~${K=10{,}000}$ vertices~(${W=H=100}$,~${N=500}$,~${T=30}$,~$1.5$M variables in~\textbf{P1}), for a total of~$260$ runs ($20$ parameter combinations per instance). The~\textbf{P1} baseline solves on average in~$130$s. Of the~$260$ runs,~$233$ are feasible; the~$27$ infeasible cases arise at the extremes~${\varepsilon=0.1}$ ($19$ runs across~$5$ instances) and~${\varepsilon=5.0}$ ($8$ runs across~$2$ instances). These infeasibilities are artifacts of the fixed pruning threshold~(${\approx 40\text{--}45\%}$ of edges retained); relaxing the threshold to retain more edges recovers feasibility in all cases. Every feasible solution is verified integral.

Table~\ref{tab:sensitivity} (in Section~\ref{sec:exp}) reports the average cost gap~(\%) relative to~\textbf{P1} for each~$(\varepsilon,\lambda)$ combination. The parameter~$\varepsilon$ is the dominant factor: small~$\varepsilon$~(${\leq 0.2}$) produces a concentrated shadow close to the~\textbf{P1} optimum, enabling aggressive pruning with~$2$--$5\%$ gap, while large~$\varepsilon$ smooths the shadow and increases the gap to~$17$--$20\%$. The parameter~$\lambda$ has a milder effect, adding roughly~$1$--$6\%$ to the gap depending on~$\varepsilon$. A plausible time-quality tradeoff is at~${\varepsilon=0.2}$,~${\lambda=0}$: a~$4.3\%$ gap in~$26$s ($5.0\times$ speedup over~\textbf{P1}).

Table~\ref{tab:sinkhorn_time} reports the Sinkhorn convergence and \textbf{P2} + \textbf{P3} timing per~$\varepsilon$ (at~${\lambda=0}$), where Sinkhorn iterations are terminated at convergence. The Sinkhorn convergence scales inversely with~$\varepsilon$, as expected from entropic regularization:~$307$ sweeps at~${\varepsilon=0.1}$ versus~$39$ at~${\varepsilon=5.0}$. However, the cost gap increases as~$\varepsilon$ increases (from~$\sim\!2\%$ at~${\varepsilon=0.1}$ to~$\sim\!17\%$ at~${\varepsilon=5.0}$), reflecting the smoothing effect that makes the shadow less discriminative and the pruned graph less targeted. The total~\textbf{P2}+\textbf{P3} time ranges from~$48$s at~${\varepsilon=0.1}$ to~$22$--$26$s at~${\varepsilon\geq 0.2}$, all well below the~\textbf{P1} baseline of~$130$s. At~${\varepsilon=5.0}$, the runtime increase is because of the diffuse shadow that makes pruning ineffective, resulting in more variables retained in~\textbf{P3}.

\begin{table}[!htb]
\centering
\caption{Sinkhorn convergence and~\textbf{P2}+\textbf{P3} pipeline timing per~$\varepsilon$ at~${\lambda=0}$, averaged over~$13$ instances at~${K=10{,}000}$,~${T=30}$.}
\label{tab:sinkhorn_time}
\vspace{0.3em}
\small
\begin{tabular}{@{}c c c c c c@{}}
\toprule
$\varepsilon$ & Sweeps & Sinkhorn (s) & \textbf{P2}+\textbf{P3} (s) & Gap (\%) & Kept (\%) \\
\midrule
$0.1$ & 307 & 20 & 48 & 2.3 & 47 \\
$0.2$ & 132 &  9 & 26 & 4.3 & 42 \\
$0.5$ &  97 &  7 & 22 & 11.1 & 40 \\
$1.0$ &  75 &  5 & 23 & 17.3 & 43 \\
$5.0$ &  39 &  3 & 25 & 17.1 & 47 \\
\bottomrule
\end{tabular}
\end{table}

\begin{figure}[!htb]
    \centering
    \includegraphics[width=0.95\textwidth]{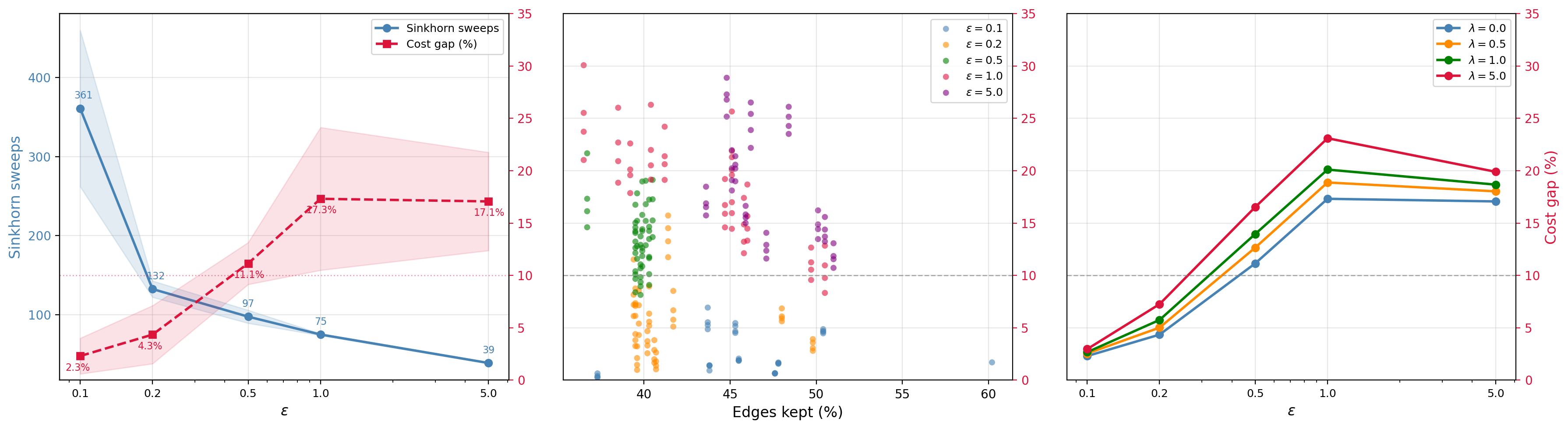}
    \caption{Parameter sensitivity across~$260$ runs ($13$ instances,~$20$ combinations each) at~${K=10{,}000}$,~${T=30}$. (Left)~Dual-axis plot at~${\lambda=0}$: effective Sinkhorn sweeps (left axis, blue) and cost gap~(\%, right axis, red) versus~$\varepsilon$; shaded bands show~${\pm 1}$ standard deviation across instances. (Middle)~Cost gap versus edges kept~(\%) for all~$233$ feasible runs, colored by~$\varepsilon$. (Right)~Average cost gap versus~$\varepsilon$ for each~$\lambda$.}
    \label{fig:sensitivity}
\end{figure}

\textbf{Parameter roles.} The parameter~$\varepsilon$ controls the sharpness of the shadow transport: as~${\varepsilon\to 0}$, the Schr\"odinger bridge concentrates onto minimum-cost geodesic corridors, producing a sharp shadow that enables aggressive pruning at low cost gap; as~$\varepsilon$ increases, the shadow becomes diffuse and the pruned graph retains more edges with less discriminative structure. The parameter~$\lambda$ controls the bias toward the shadow in the~\textbf{P3} cost: increasing~$\lambda$ penalizes edges with small shadow flow, effectively forcing~\textbf{P3} to follow the shadow transport more closely. When costs are uniform,~$\lambda$ and~$\varepsilon$ together provide a beneficial tiebreaker among edges that are otherwise equivalent in cost, favoring those with a darker shadow. The parameter~$\delta$ is a small numerical safeguard for the logarithm in the modified~\textbf{P3} cost and has negligible effect on the solution. Across the experiments reported in this paper, the following is a robust default that works without much fine-tuning:~${\delta=10^{-6}}$,~${\varepsilon=0.2}$, stop Sinkhorn when the last~$20$ iterates stabilize, and~${\lambda=0}$. For the pruning threshold~$\eta$, retaining~$40$--$48\%$ of edges (depending on~$\varepsilon$) consistently yields feasible~\textbf{P3} solutions.

Fig.~\ref{fig:sensitivity} visualizes the sensitivity structure from three perspectives. The left panel is a dual-axis plot: the left axis shows the effective number of Sinkhorn sweeps (blue) and the right axis shows the cost gap (red), both versus~$\varepsilon$ at~${\lambda=0}$, with~${\pm 1}$ standard deviation shaded across the~$13$ instances. The sweeps decrease inversely with~$\varepsilon$, while the gap increases nearly monotonically. The middle panel plots the gap against the fraction of edges retained for all~$233$ feasible runs, colored by~$\varepsilon$, where lower~$\varepsilon$ achieves lower gaps at comparable pruning levels. For a fixed~$\varepsilon$, the cost variation is also due to varying~$\lambda$. The right panel isolates the effect of~$\lambda$ by plotting the gap versus~$\varepsilon$ for each~$\lambda$ value; the curves confirm that~$\lambda$ shifts the gap upward by a roughly constant offset, with a mild effect relative to~$\varepsilon$.

\subsection{Non-uniform Costs}\label{nu_costs}
In this section, we validate the proposed~\textbf{P1} and the~\textbf{P2}+\textbf{P3} pipeline under non-uniform costs. We assign each cell a random arrival cost~${\sim\mathrm{Uniform}[0.6,1]}$ and wait cost~${\sim\mathrm{Uniform}[0.1,0.5]}$ (wait at target~${=0}$), preserving Assumption~\ref{assump1}. Choosing costs like this reflects e.g., uneven terrains; see Fig.~\ref{fig:nonuniform} (left) for a candidate scenario. We run~$24$ instances at~${K=10{,}000}$~(${W=H=100}$,~${N=500}$ robots,~${T=30}$) with~${\varepsilon=0.2}$,~${\lambda=0}$. The~\textbf{P1} baseline averages~$138$s; the~\textbf{P2}+\textbf{P3} pipeline averages~$25$s at~$5.1\%$ gap and~$5.4\times$ speedup. For reference, the uniform-cost baseline (from the sensitivity study) gives~$4.3\%$ gap and~$5.0\times$ speedup. The gap and speedup under non-uniform costs are comparable, confirming that the pipeline adapts to the cost landscape without degradation. Fig.~\ref{fig:nonuniform} (middle) shows the per-instance gap and speedup, and Fig.~\ref{fig:nonuniform} (right) compares the uniform and non-uniform averages. Every solution is verified integral. 

\begin{figure}[!htb]
    \centering
    {\includegraphics[width=0.3\textwidth]{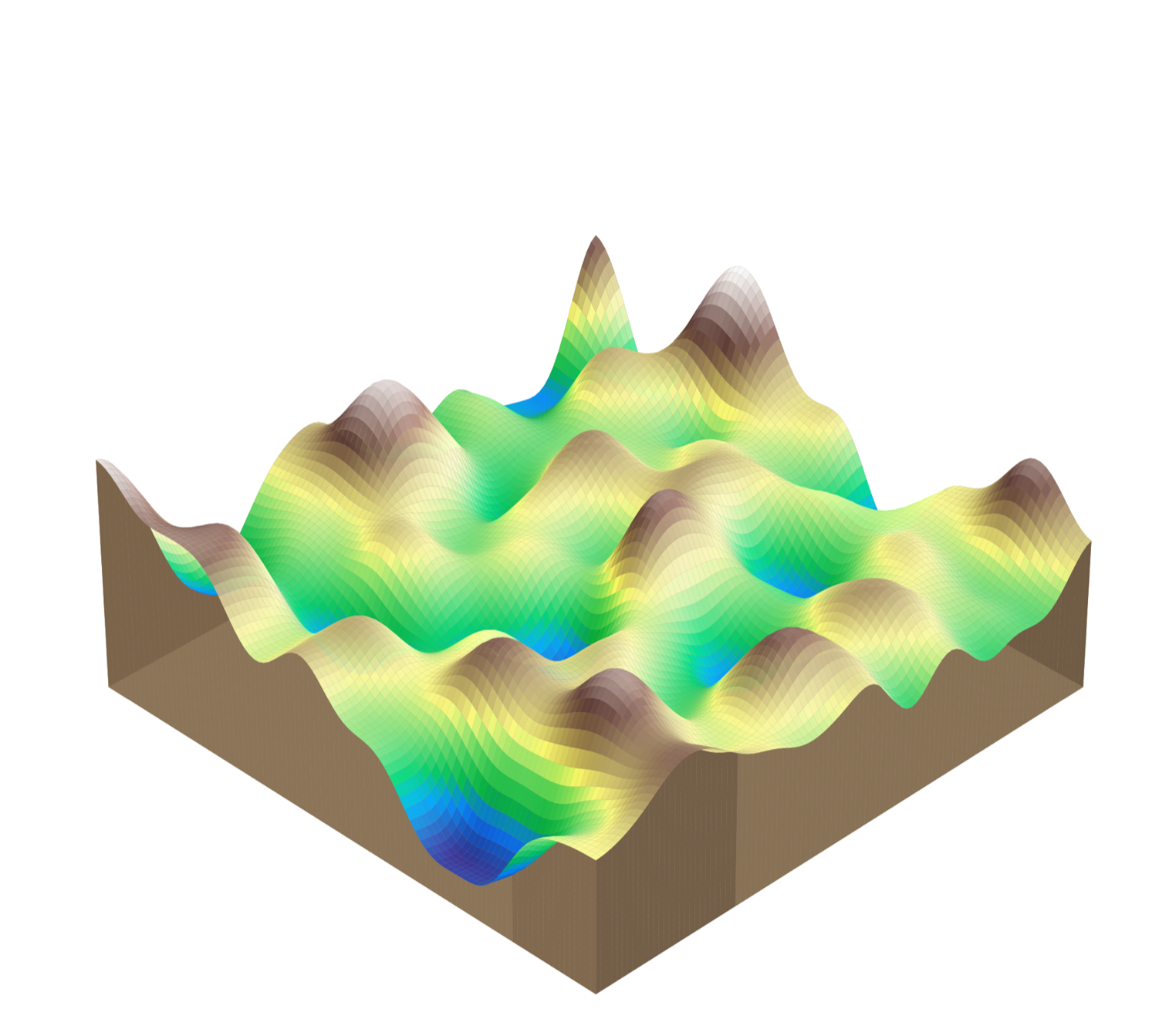}}
    \hspace{0.1cm}
    {\includegraphics[width=0.3\textwidth]{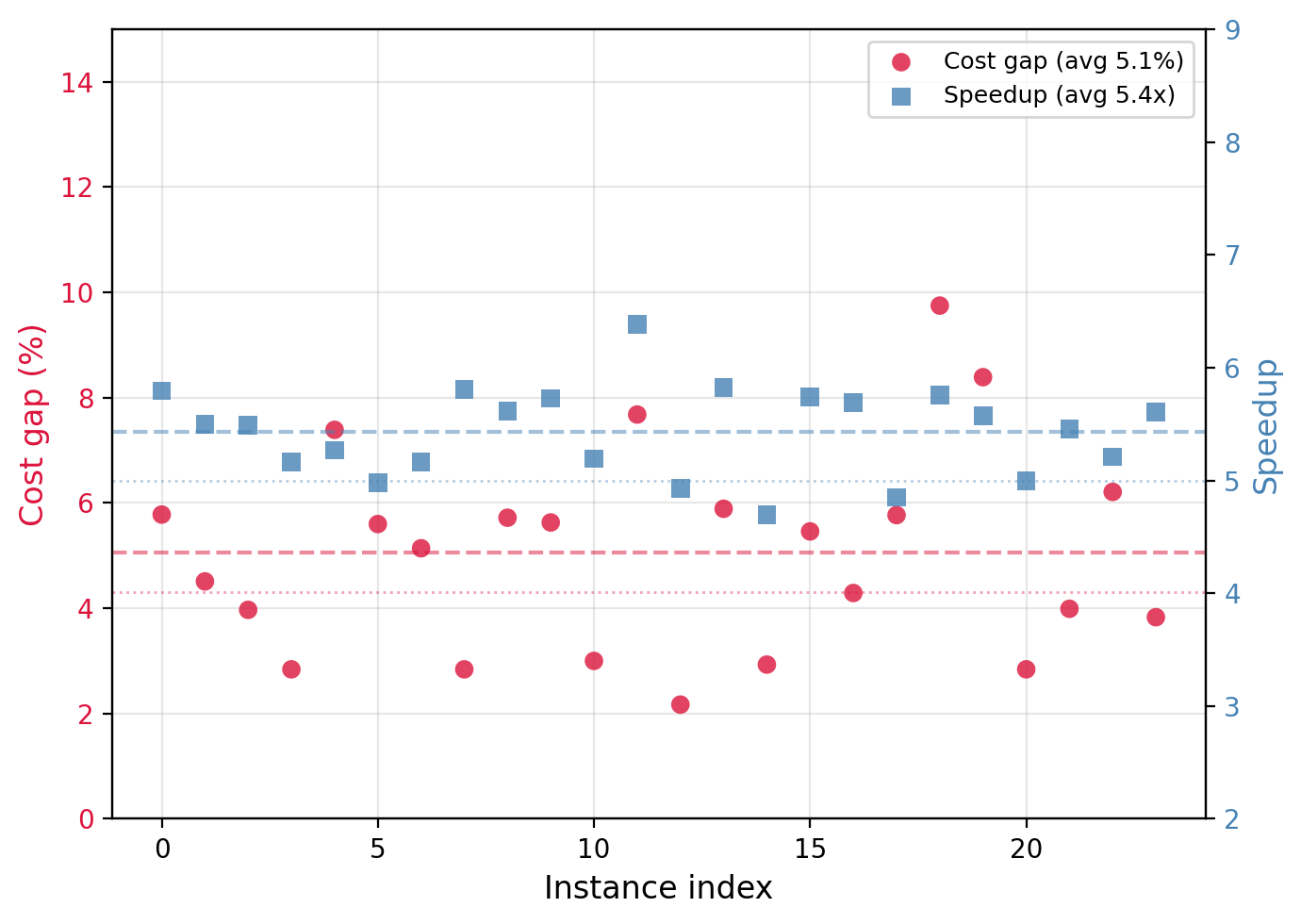}}
    \hspace{0.1cm}
    {\includegraphics[width=0.3\textwidth]{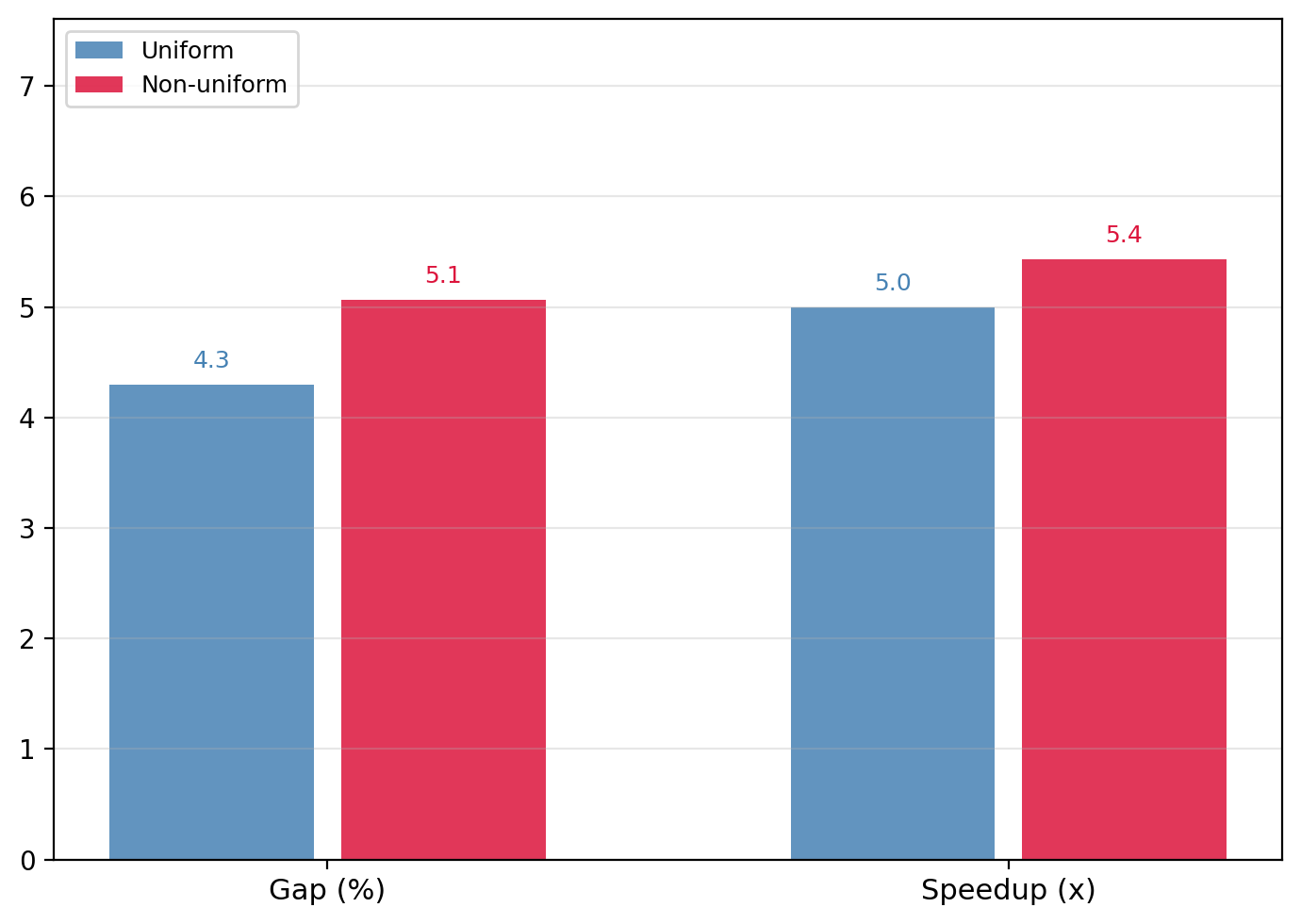}}
    \caption{Non-uniform cost experiments across~$24$ instances at~${K=10{,}000}$,~${T=30}$. (Left)~Illustrative terrain on a~${100\times 100}$ grid; each cell has a random arrival cost~${\sim\mathrm{Uniform}[0.6,1]}$; higher elevation corresponds to higher move cost. (Middle)~Cost gap~(\%, left axis, red) and speedup~(right axis, blue) per instance; dashed lines show non-uniform averages ($5.1\%$ gap,~$5.4\times$ speedup), dotted lines show the uniform-cost reference ($4.3\%$ gap,~$5.0\times$ speedup). (Right)~Uniform versus non-uniform cost comparison, averaged over all instances.}
    \label{fig:nonuniform}
\end{figure}

\subsection{Baseline Comparison}\label{comps}
In this section, we compare with a close prior anonymous MAPF formulation~\cite{Ma2016}, which studies TAPF (combined target-assignment and path-finding). In TAPF, a generalization of anonymous MAPF, agents are partitioned into teams, each team is given the same number of targets as agents, and the goal is to jointly assign agents to targets and plan collision-free paths that minimize makespan. TAPF with a single team (all agents exchangeable) is the anonymous MAPF problem that we solve. Their method CBM (Conflict-Based Min-Cost-Flow) uses a min-cost max-flow solver on a time-expanded network on the low level (for within-team assignment and routing) and conflict-based search on the high level (for inter-team collision resolution). We note that the network-flow formulation in~\cite{Ma2016} (or in~\cite{YuLaValle2013NetworkFlow}) does not establish total unimodularity of the constraint matrix. Consequently, even for a single team, the LP relaxation is not formulated and is not guaranteed to produce integer solutions; one must therefore resort to integer linear programming (ILP) or the hierarchical CBM search for integrality. 

In contrast, our~\textbf{P1} formulation guarantees integral solutions from the continuous LP via TU, without branch-and-bound or conflict resolution. The comparison next provides a useful reference point on the same grid setting. We adopt the same experimental setting as~\cite{Ma2016}, Table~1: a~${30\times 30}$ grid with~$10\%$ randomly blocked cells and~$4$-neighbor connectivity. Table~\ref{tab:baseline} compares the results. CBM solves up to~$50$ agents in~$5.32$s; the reported ILP-based solver handles~$50$ agents in~$162$s with only~$4\%$ success rate within a~$5$-minute timeout. We note that restricting the formulation in~\cite{Ma2016} to a single team may improve the performance, as the multi-team structure introduces additional complexity. In contrast, the proposed~\textbf{P1} solves~$300$ agents (same grid, $6\times$ more agents) in~$0.54$s on average, and the~\textbf{P2}+\textbf{P3} pipeline solves in~$0.49$s at~$0.63\%$ gap. Our framework further scales to~${K=22{,}500}$ vertices ($1{,}125$ agents) where~\textbf{P2}+\textbf{P3} solve time is~$67$s on average.

\begin{table}[!htb]
\centering
\caption{Comparison on~${30\times 30}$ grids with~$10\%$ obstacles. CBM and ILP results are from~\cite{Ma2016}, Table~1. Our results are averaged over~$15$ instances.}
\label{tab:baseline}
\vspace{0.3em}
\small
\begin{tabular}{@{}l c c c@{}}
\toprule
Method & Agents & Time (s) & Success \\
\midrule
CBM~\cite{Ma2016} & 50 & 5.32 & 100\% \\
ILP~\cite{Ma2016} & 50 & 162 & 4\% \\
ILP~\cite{Ma2016} & 40 & 153 & 14\% \\
\midrule
\textbf{P1} (ours) & 300 & 0.54 & 100\% \\
\textbf{P2}+\textbf{P3} (ours) & 300 & 0.49 & 100\% \\
\bottomrule
\end{tabular}
\end{table}

\newpage
\begin{figure}[!htb]
    \centering
    {\includegraphics[width=0.3\textwidth]{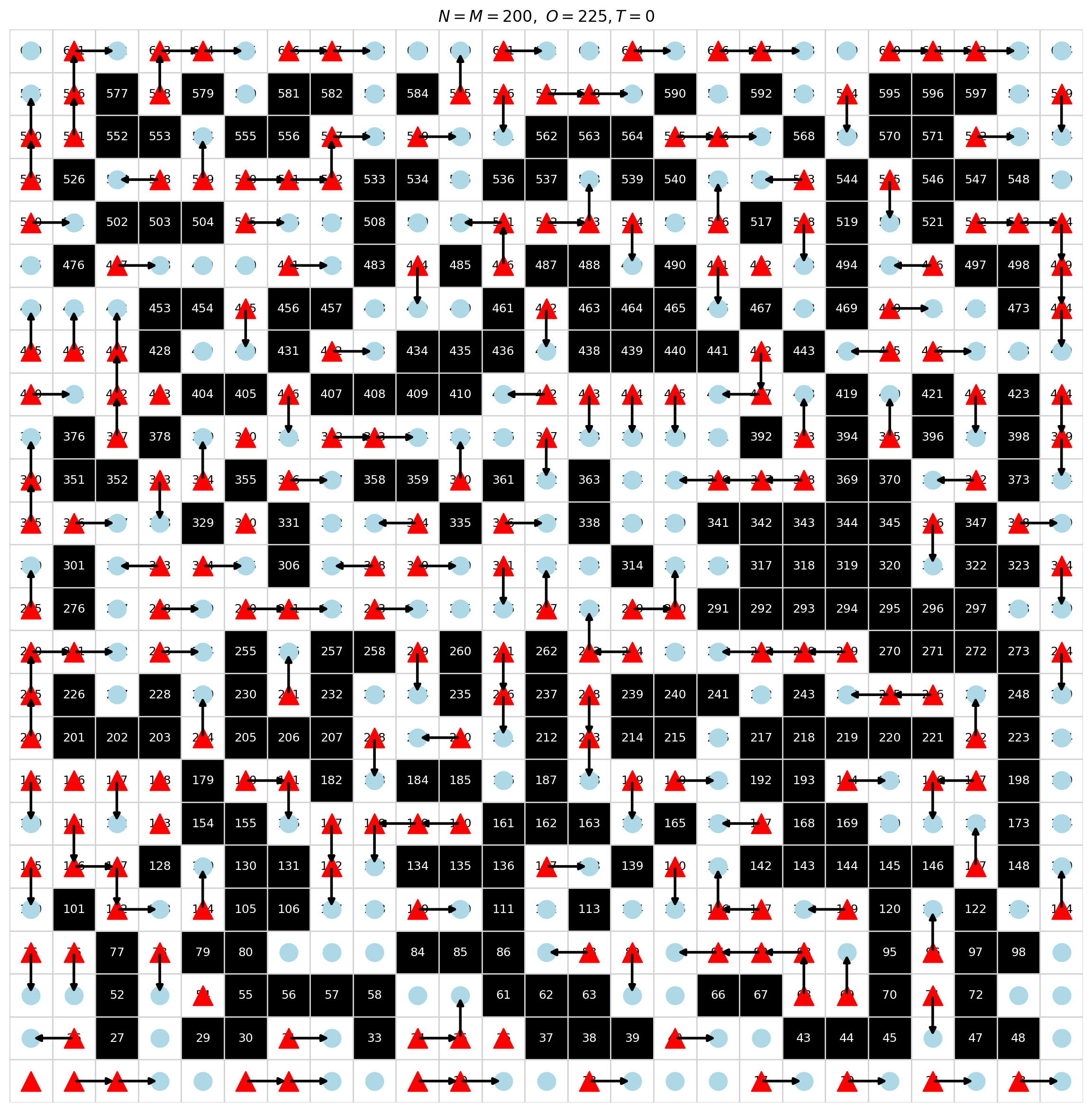}}
    \hspace{0cm}
    {\includegraphics[width=0.3\textwidth]{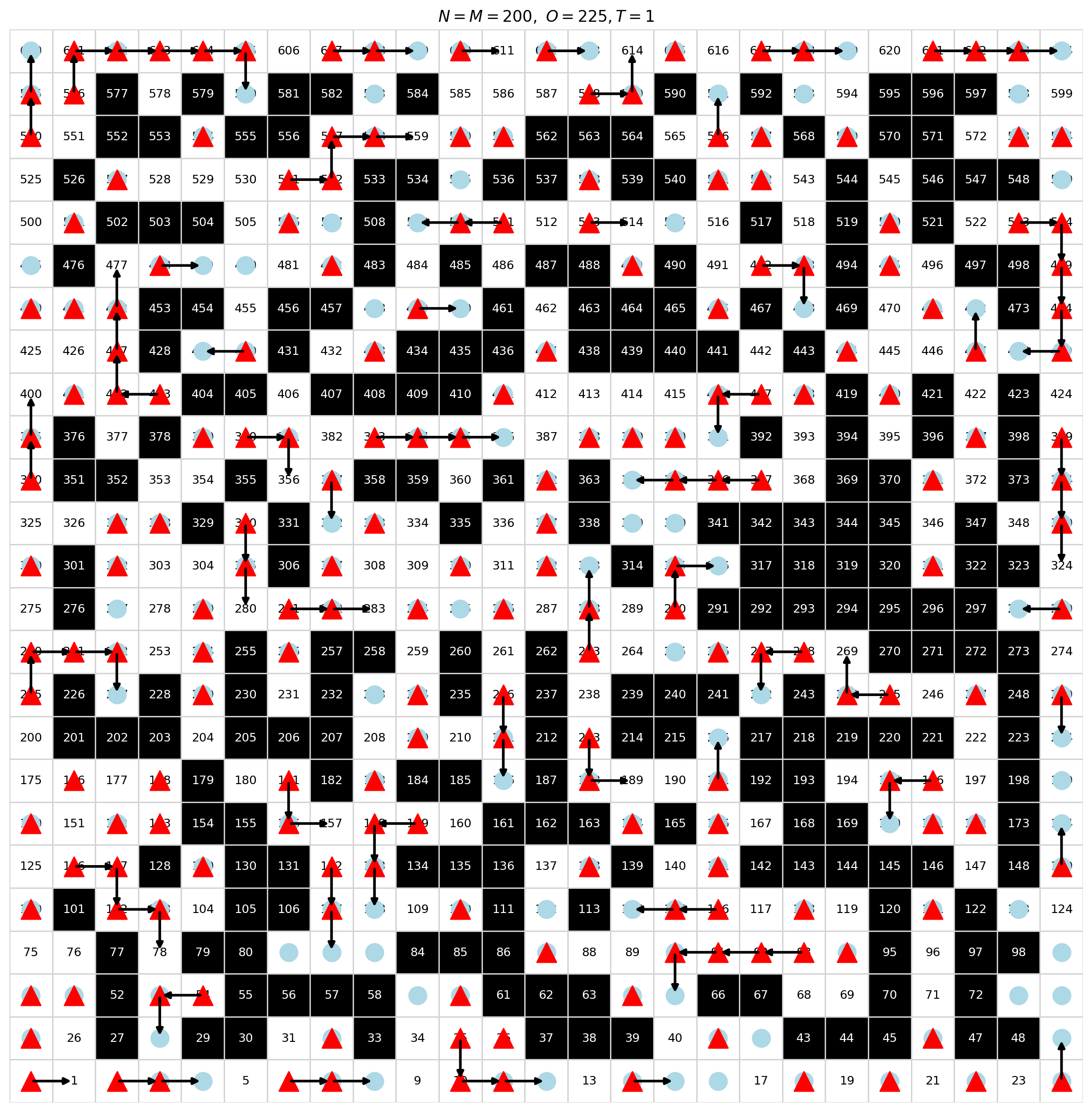}}
    \hspace{0cm}
    {\includegraphics[width=0.3\textwidth]{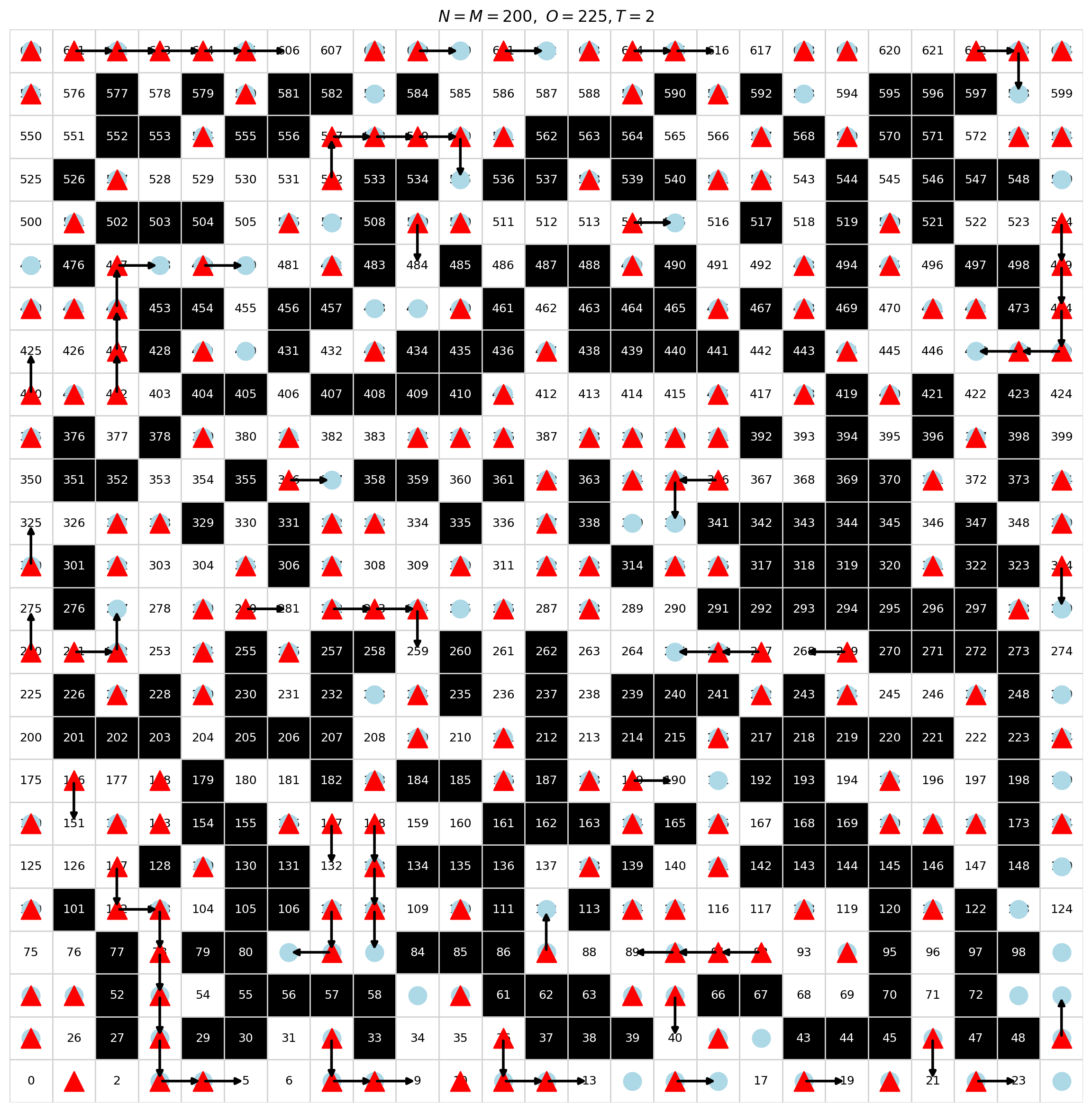}}
    \hspace{0cm}
    {\includegraphics[width=0.3\textwidth]{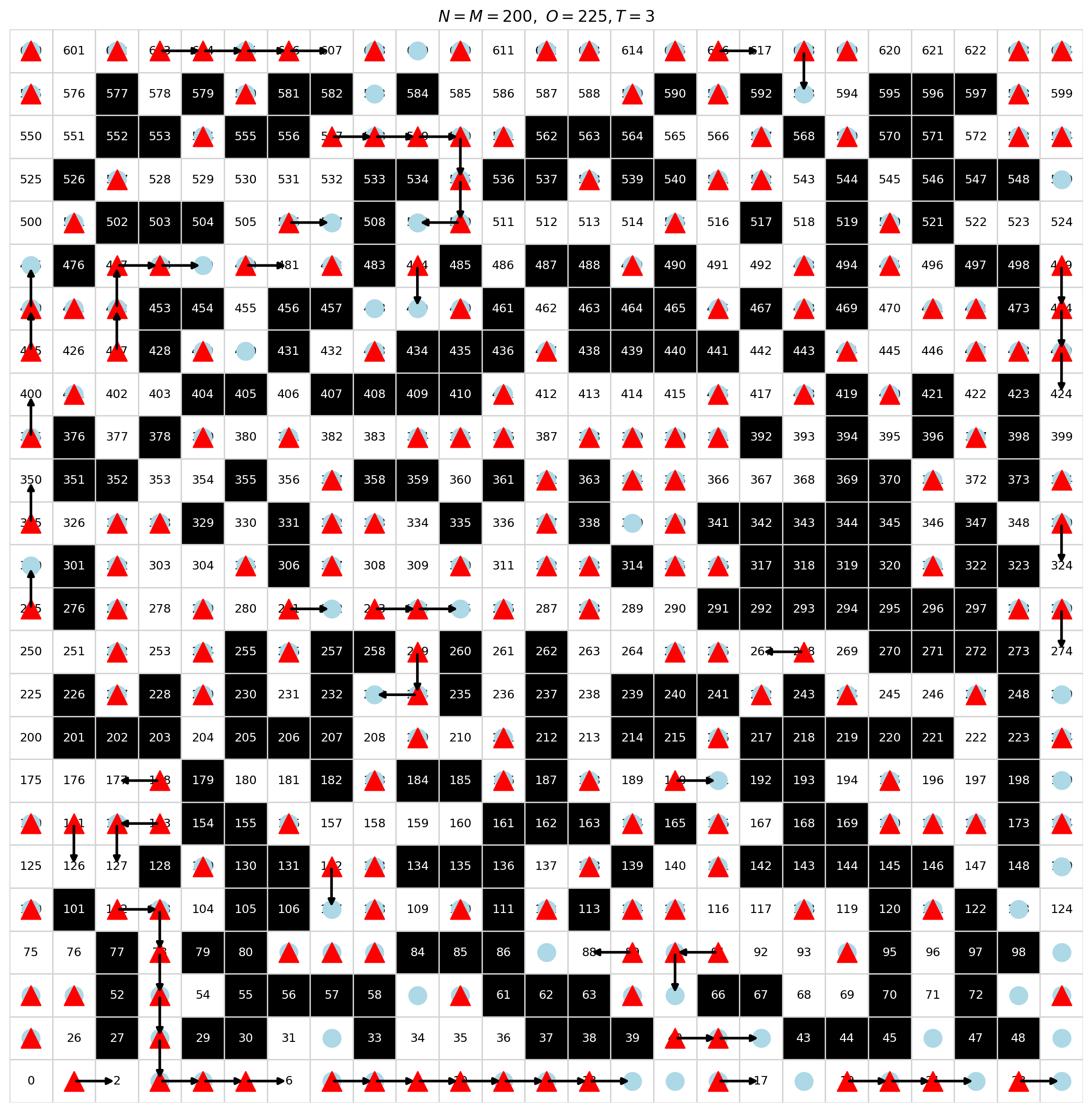}}
    {\includegraphics[width=0.3\textwidth]{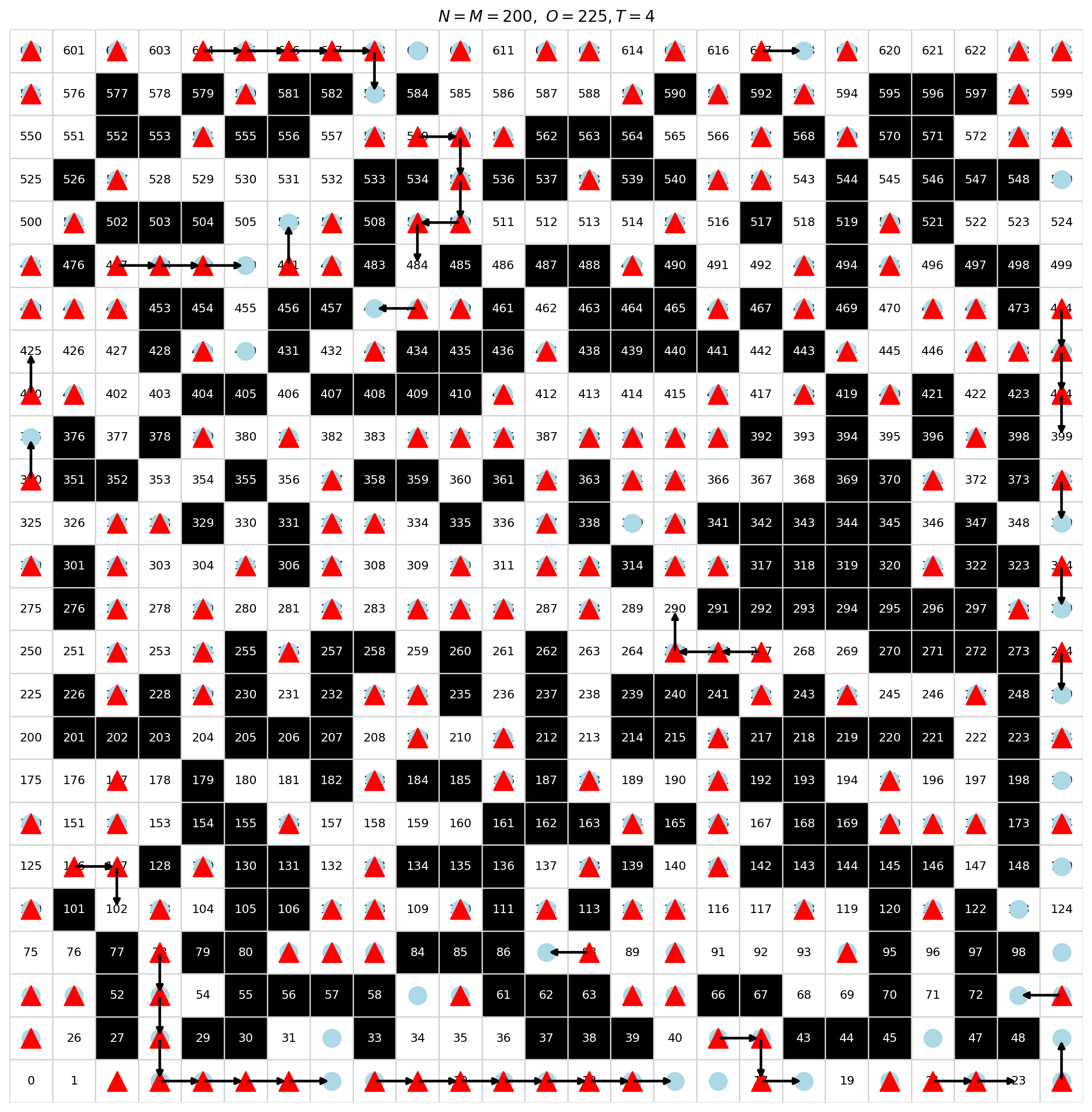}}
    \hspace{0cm}
    {\includegraphics[width=0.3\textwidth]{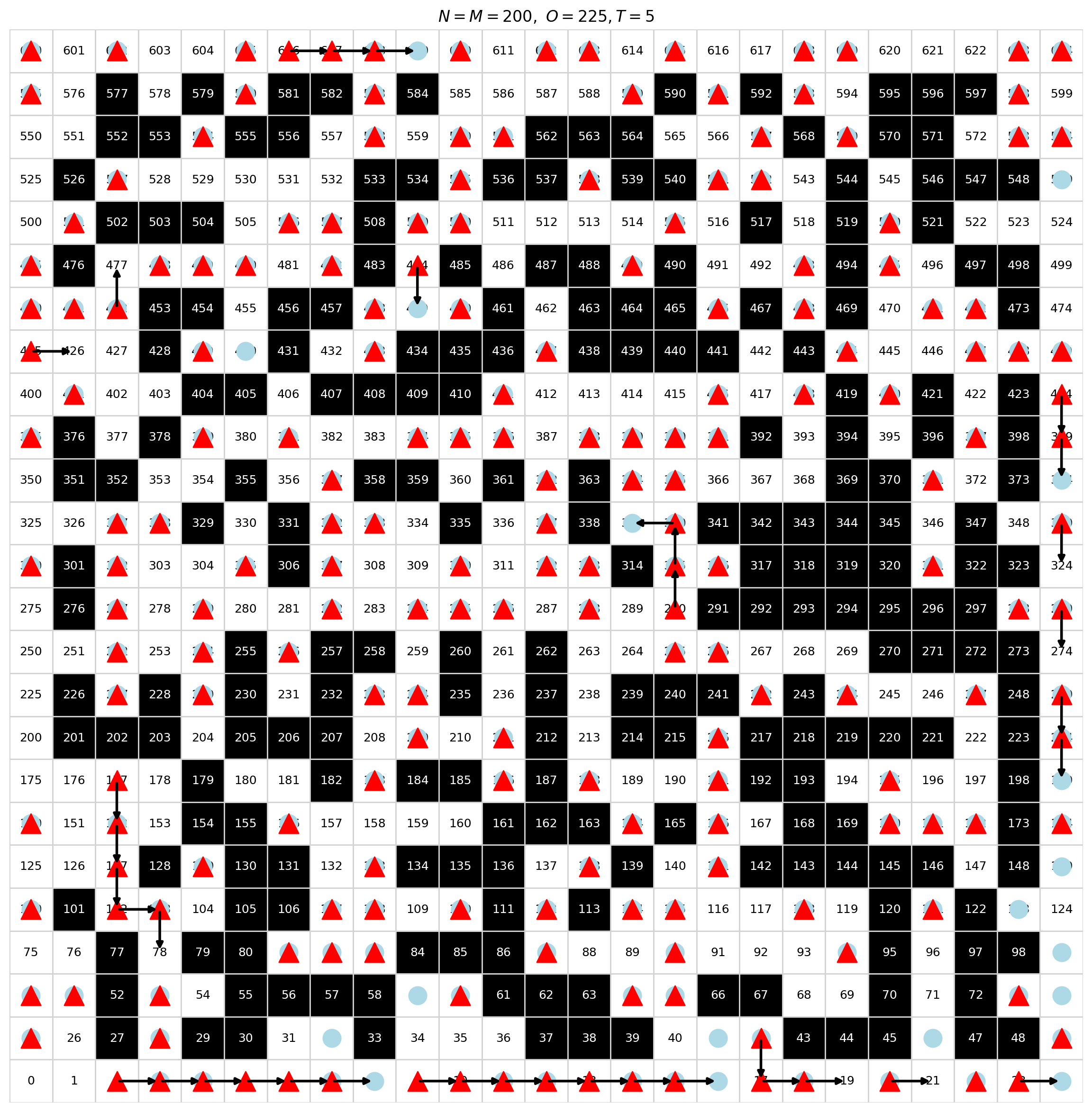}}
    \hspace{0cm}
    {\includegraphics[width=0.3\textwidth]{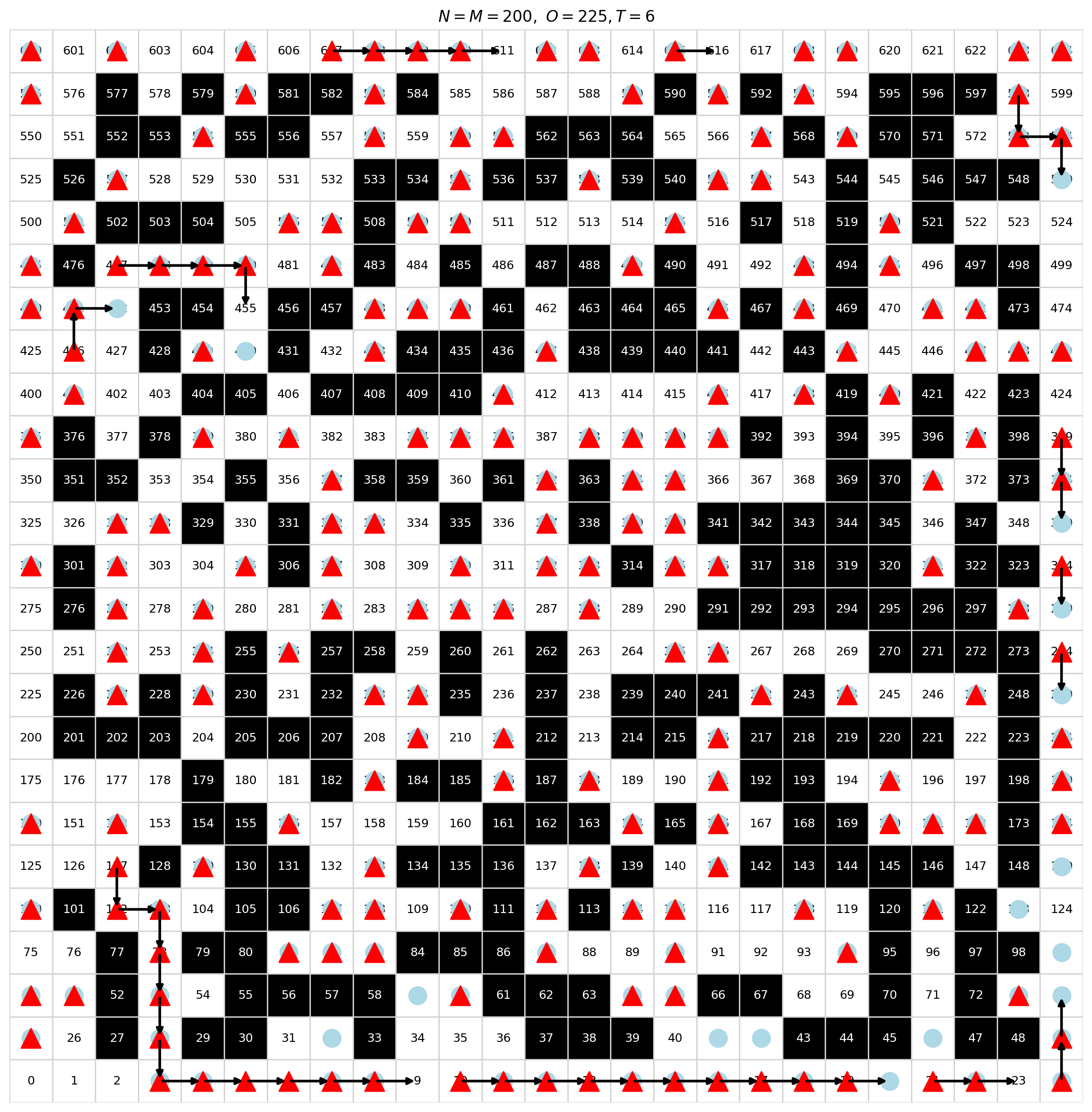}}
    \hspace{0cm}
    {\includegraphics[width=0.3\textwidth]{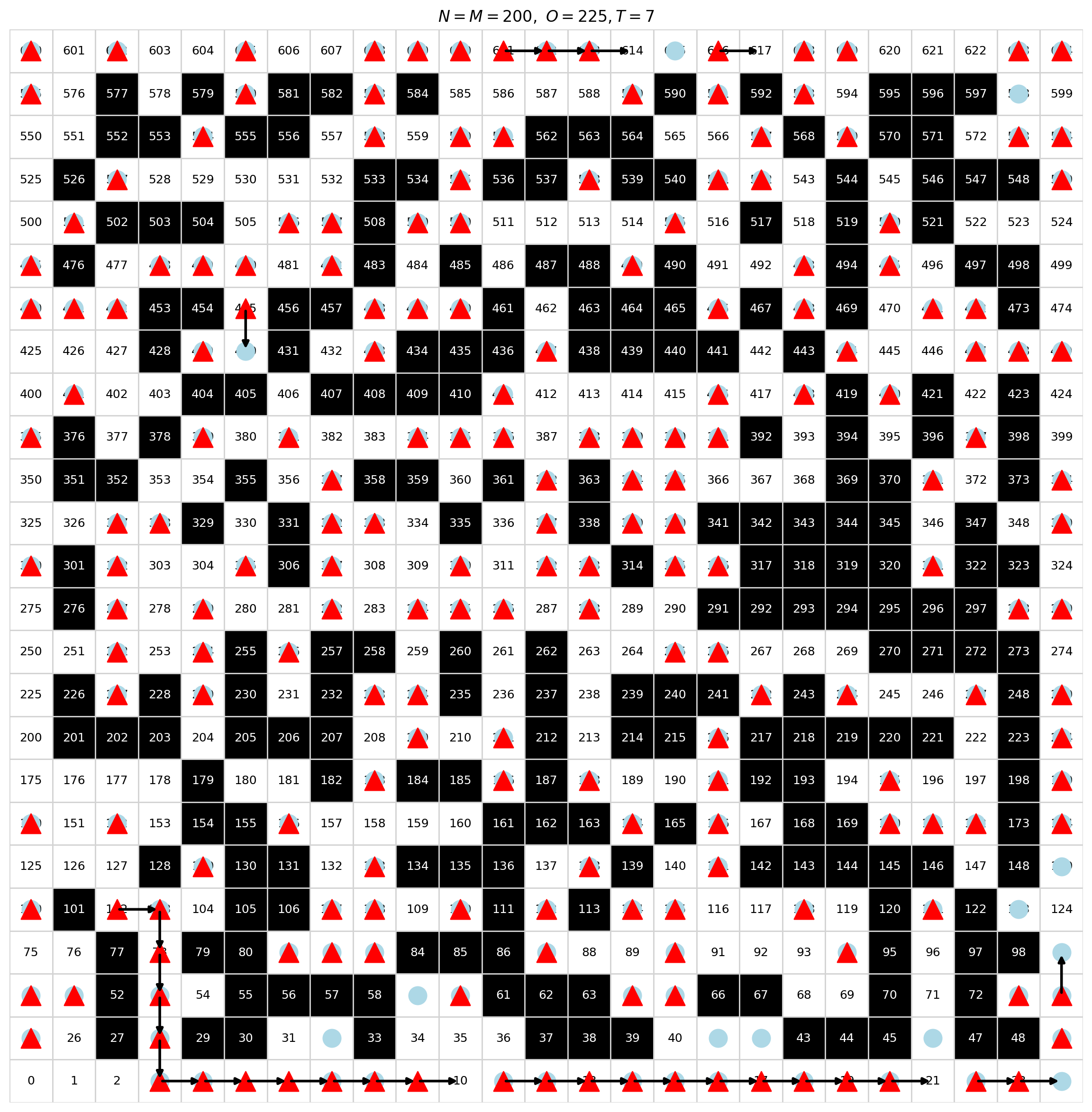}}
    \hspace{0cm}
    {\includegraphics[width=0.3\textwidth]{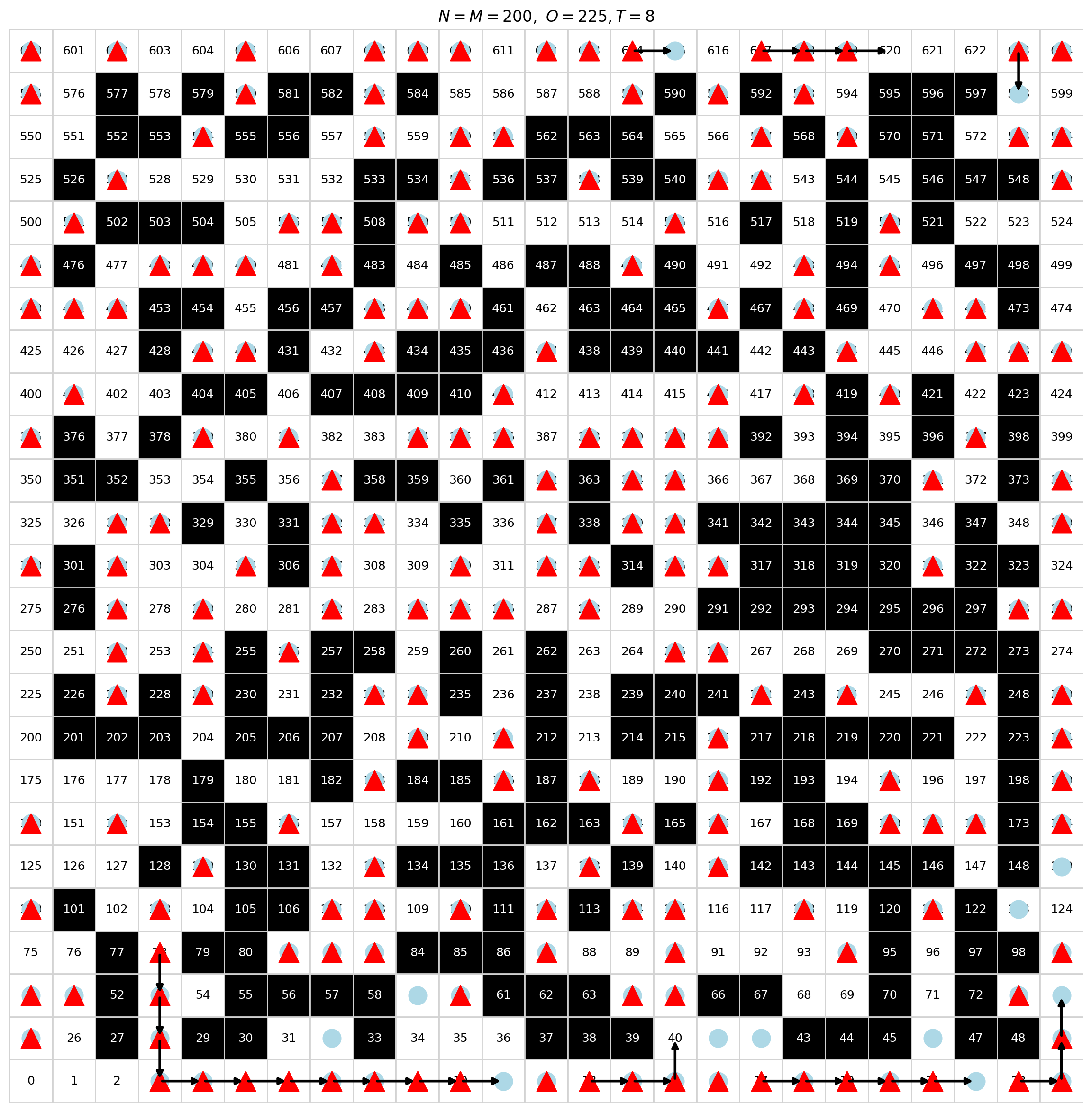}}
    \hspace{0cm}
    {\includegraphics[width=0.3\textwidth]{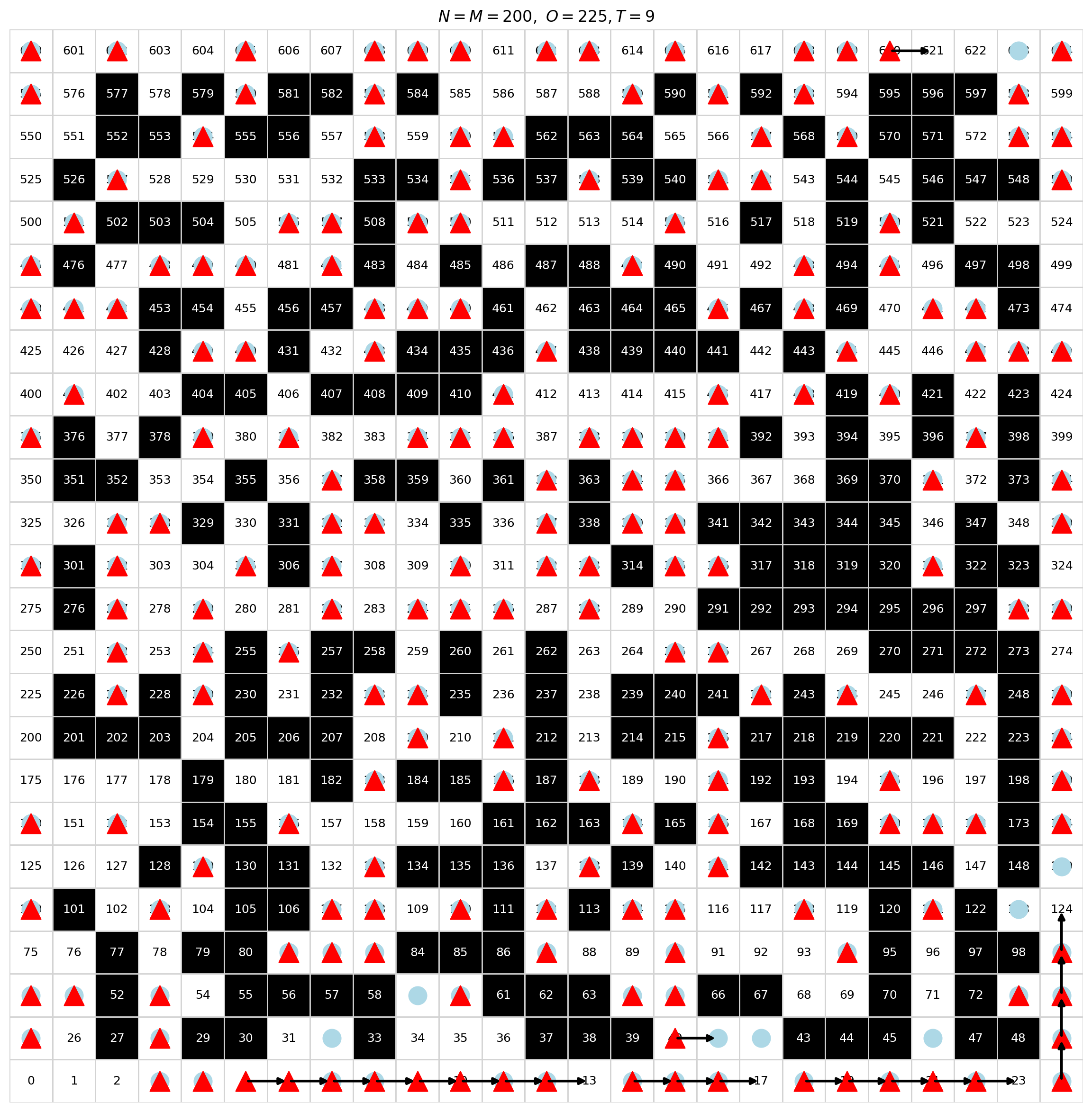}}
    \hspace{0cm}
    {\includegraphics[width=0.3\textwidth]{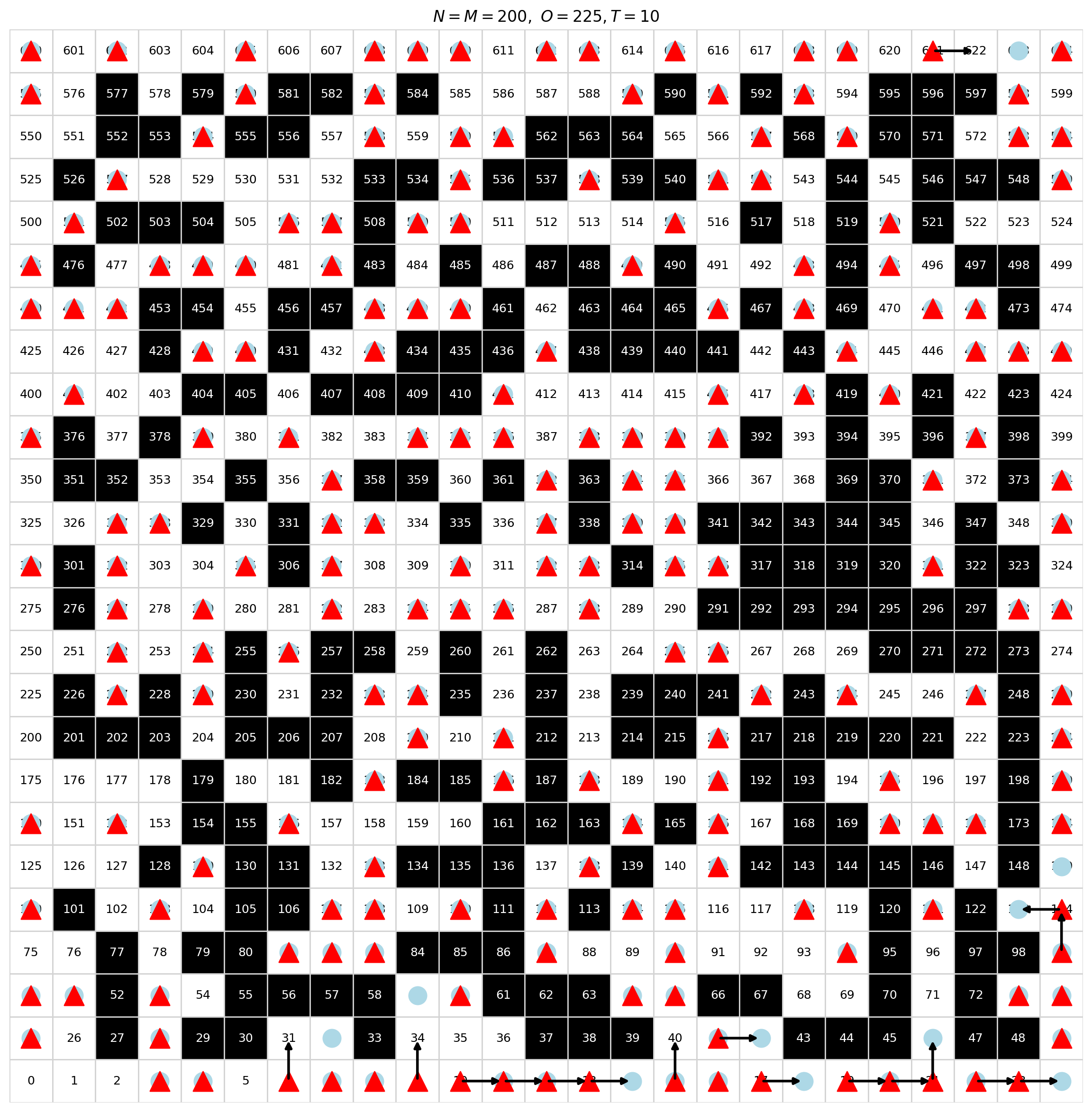}}
    \hspace{0cm}
    {\includegraphics[width=0.3\textwidth]{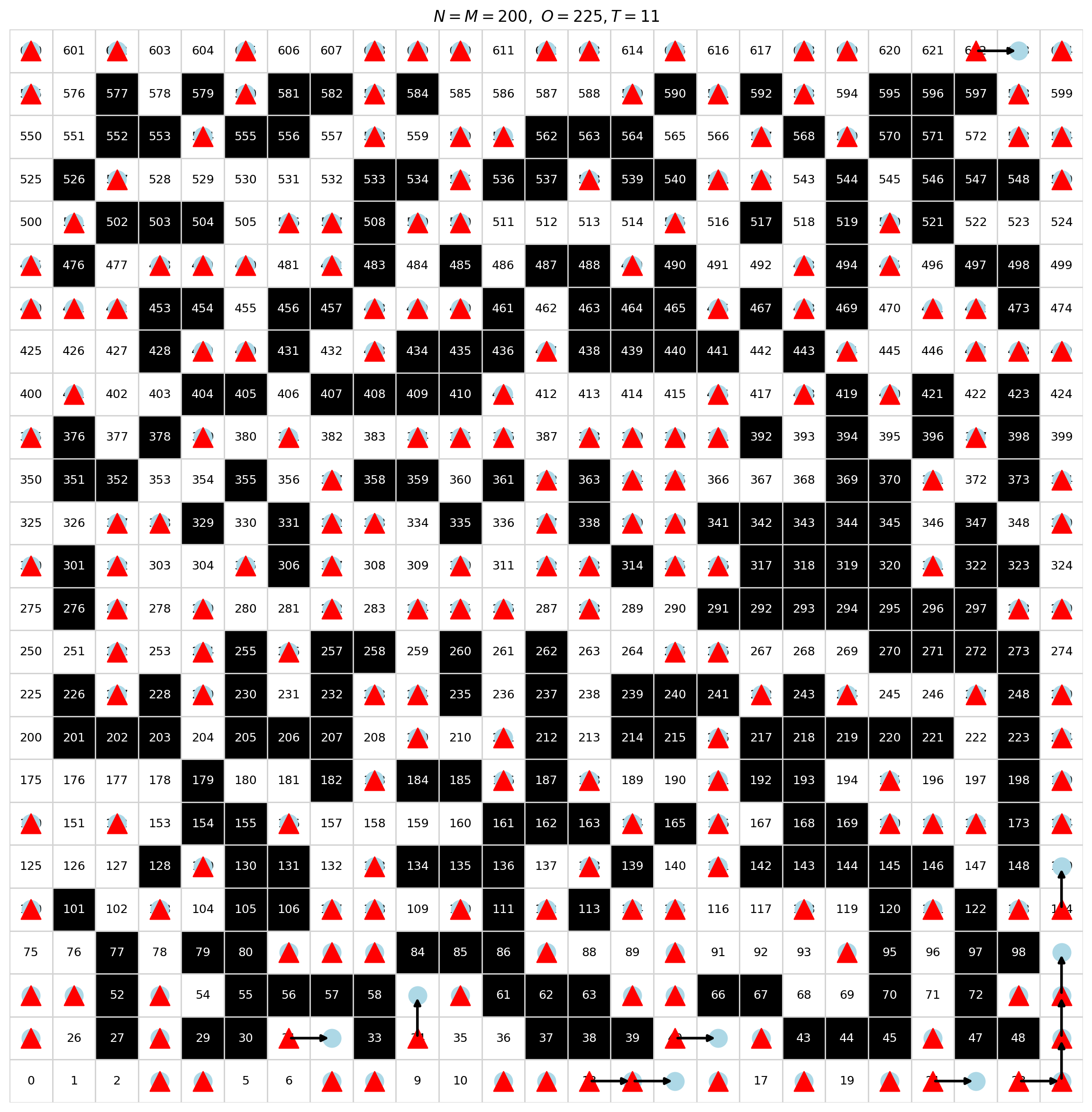}}
    \caption{A~${25\times 25}$ grid with~${\{200\}}$ robots~$\robot$,~${\{200\}}$ targets~$\target$, and~${\{225\}}$ obstacles. Every cell is either a robot, a target, or an obstacle. The robot trajectories come from the min-cost transport obtained by solving~\textbf{P1} over~${T=12}$.}
    \label{a_fig1}
\end{figure}

\newpage
\begin{figure}[!htb]
    \centering
    {\includegraphics[width=0.4\textwidth]{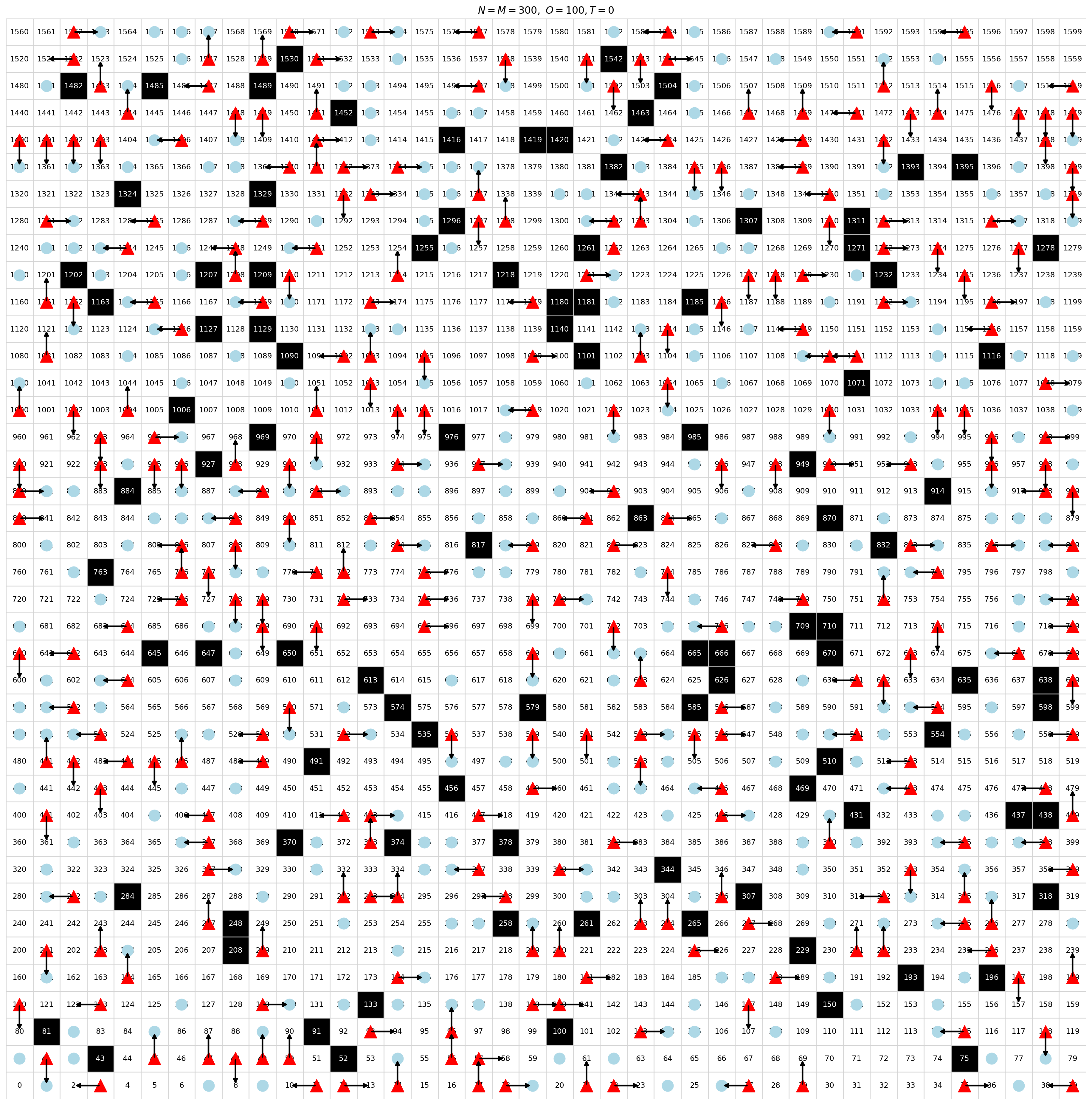}}
    \hspace{0cm}
    {\includegraphics[width=0.4\textwidth]{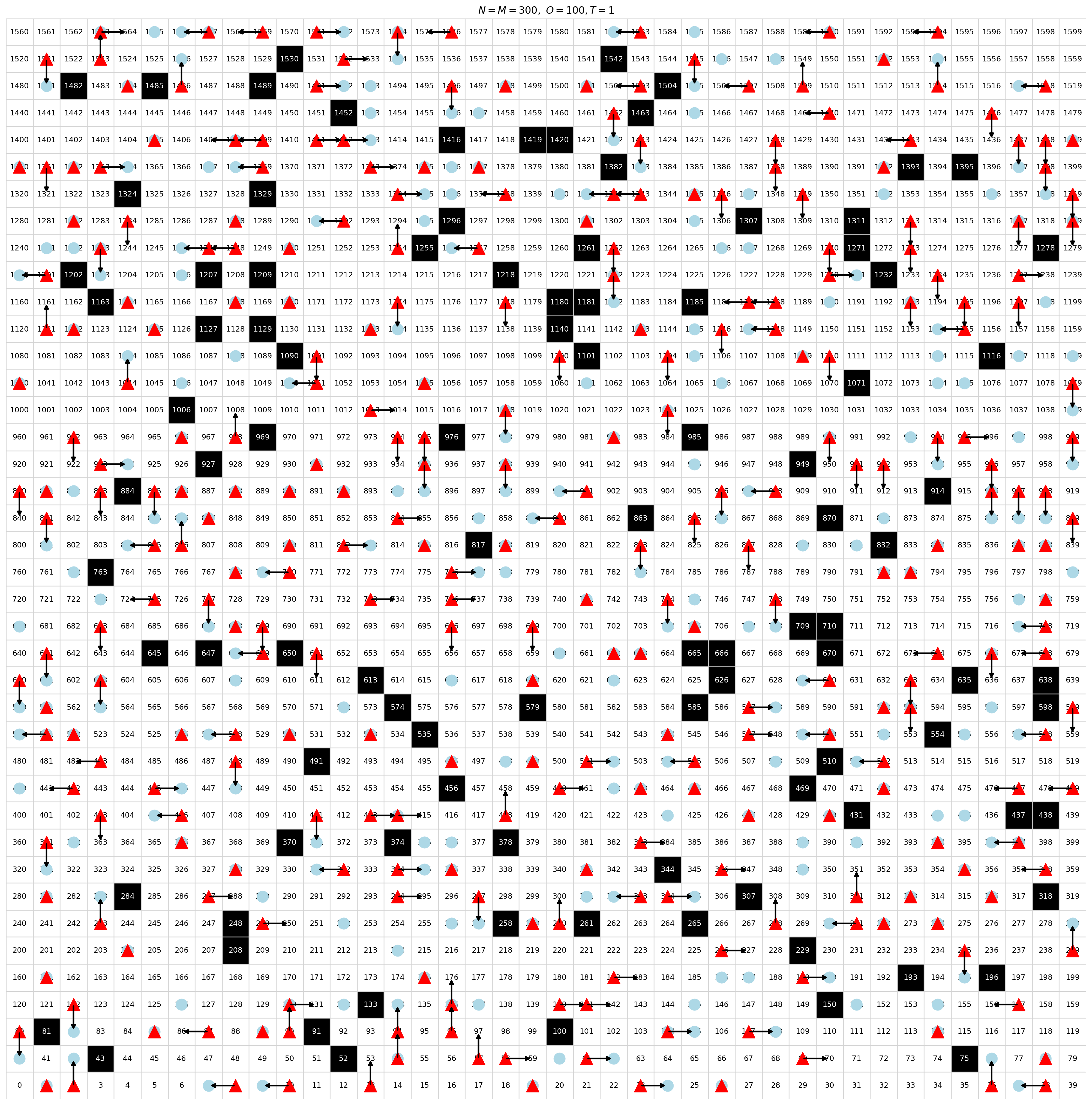}}
    \hspace{0cm}
    {\includegraphics[width=0.4\textwidth]{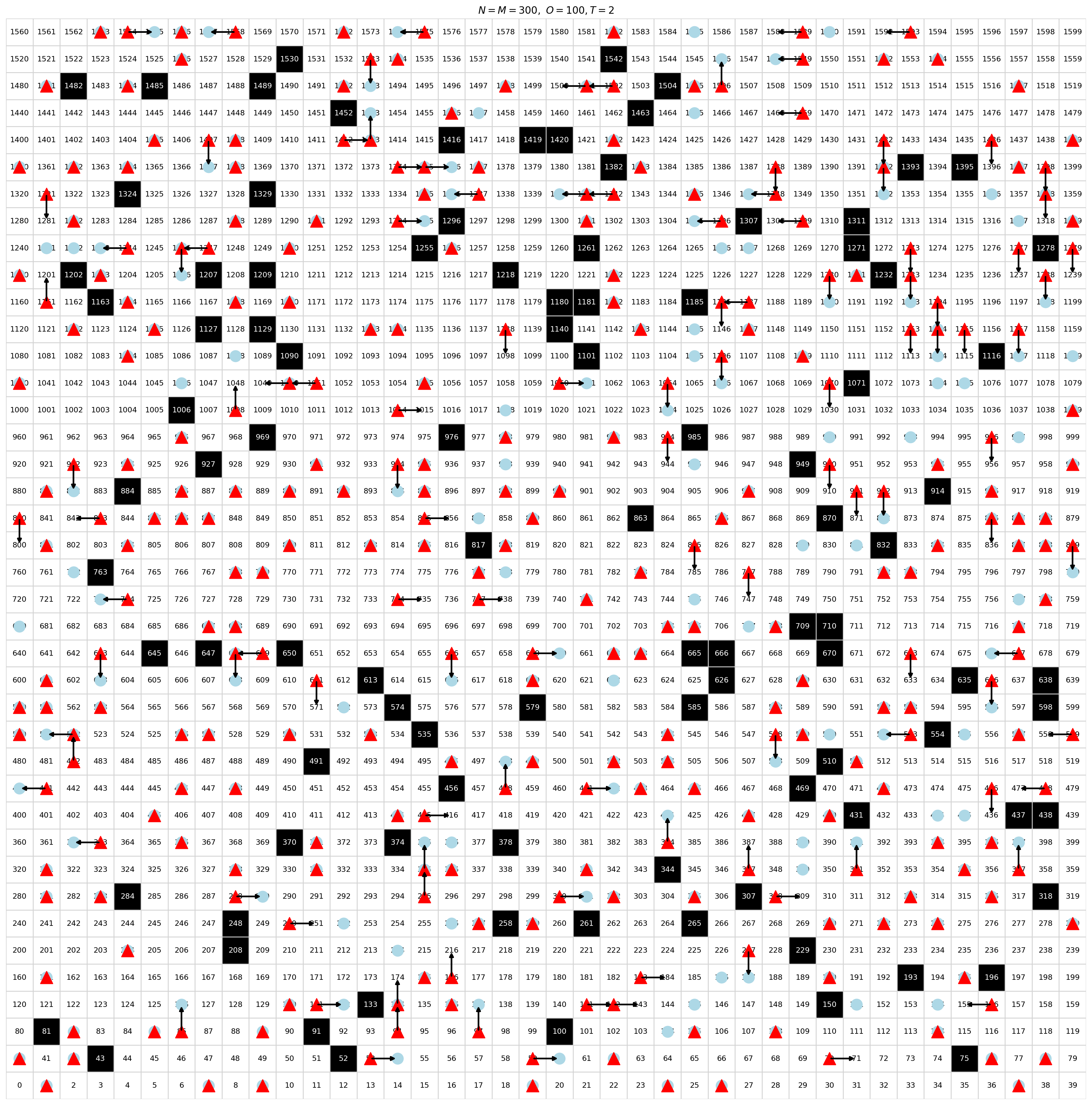}}
    \hspace{0cm}
    {\includegraphics[width=0.4\textwidth]{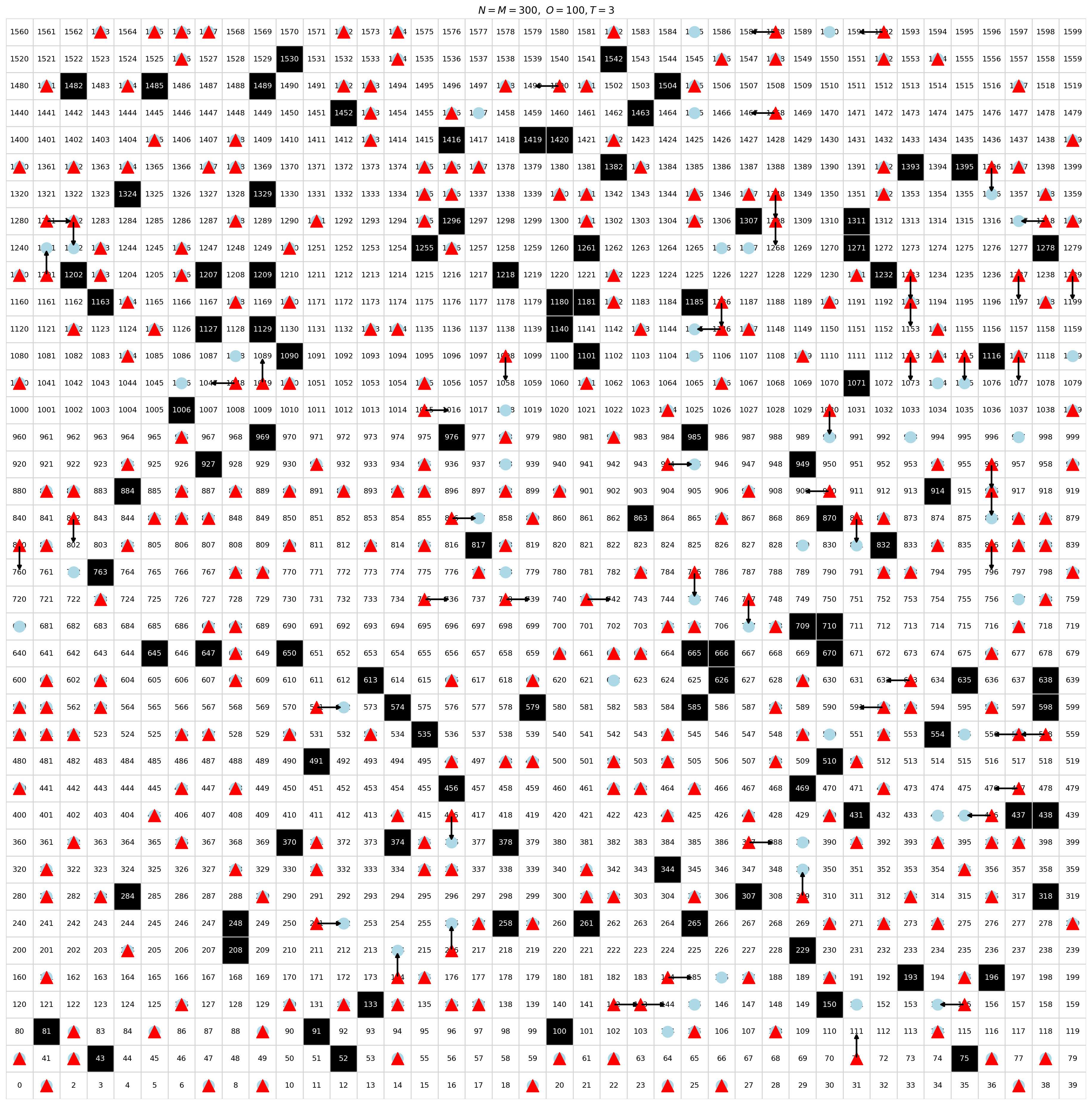}}
    {\includegraphics[width=0.4\textwidth]{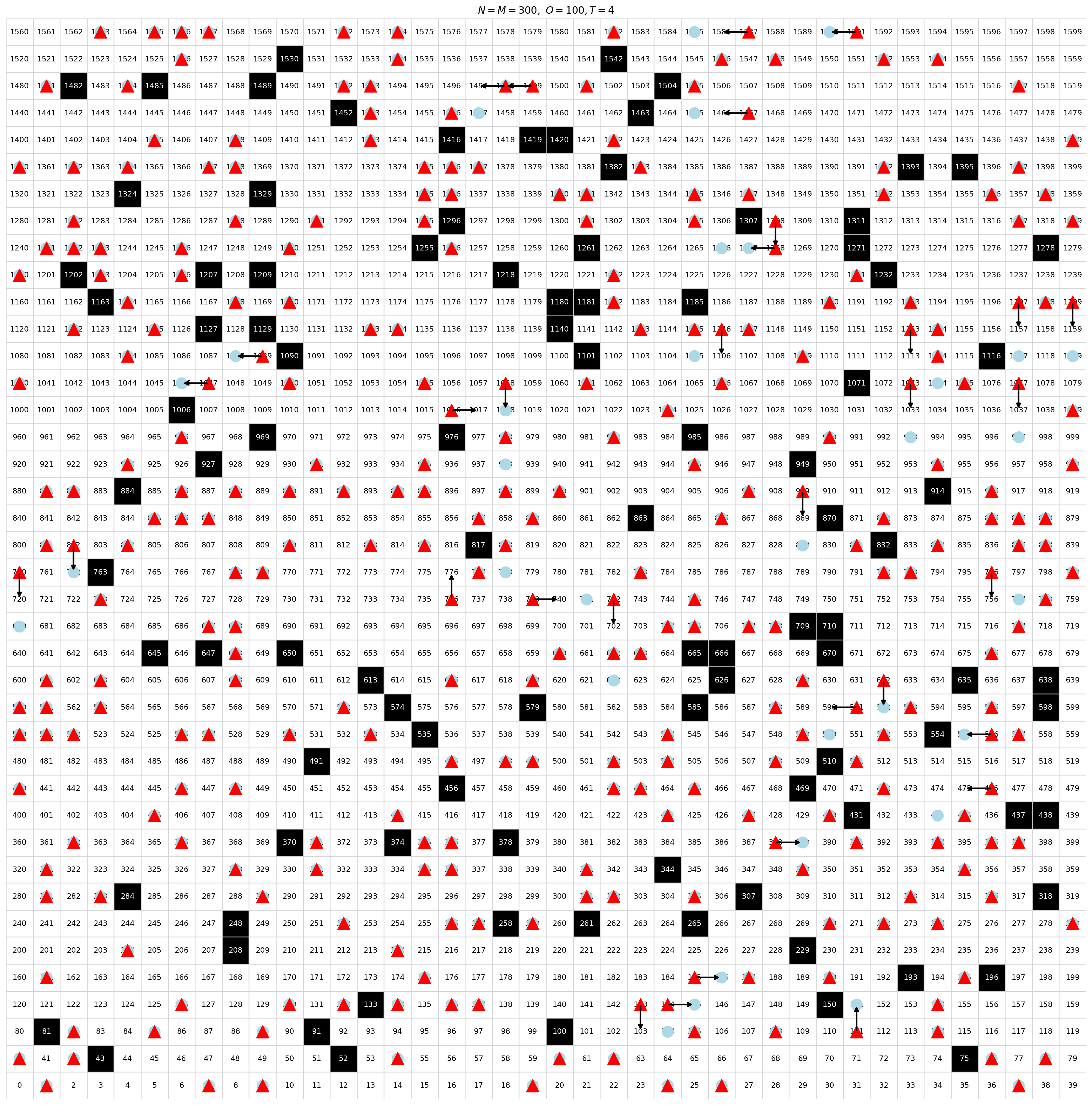}}
    \hspace{0cm}
    {\includegraphics[width=0.4\textwidth]{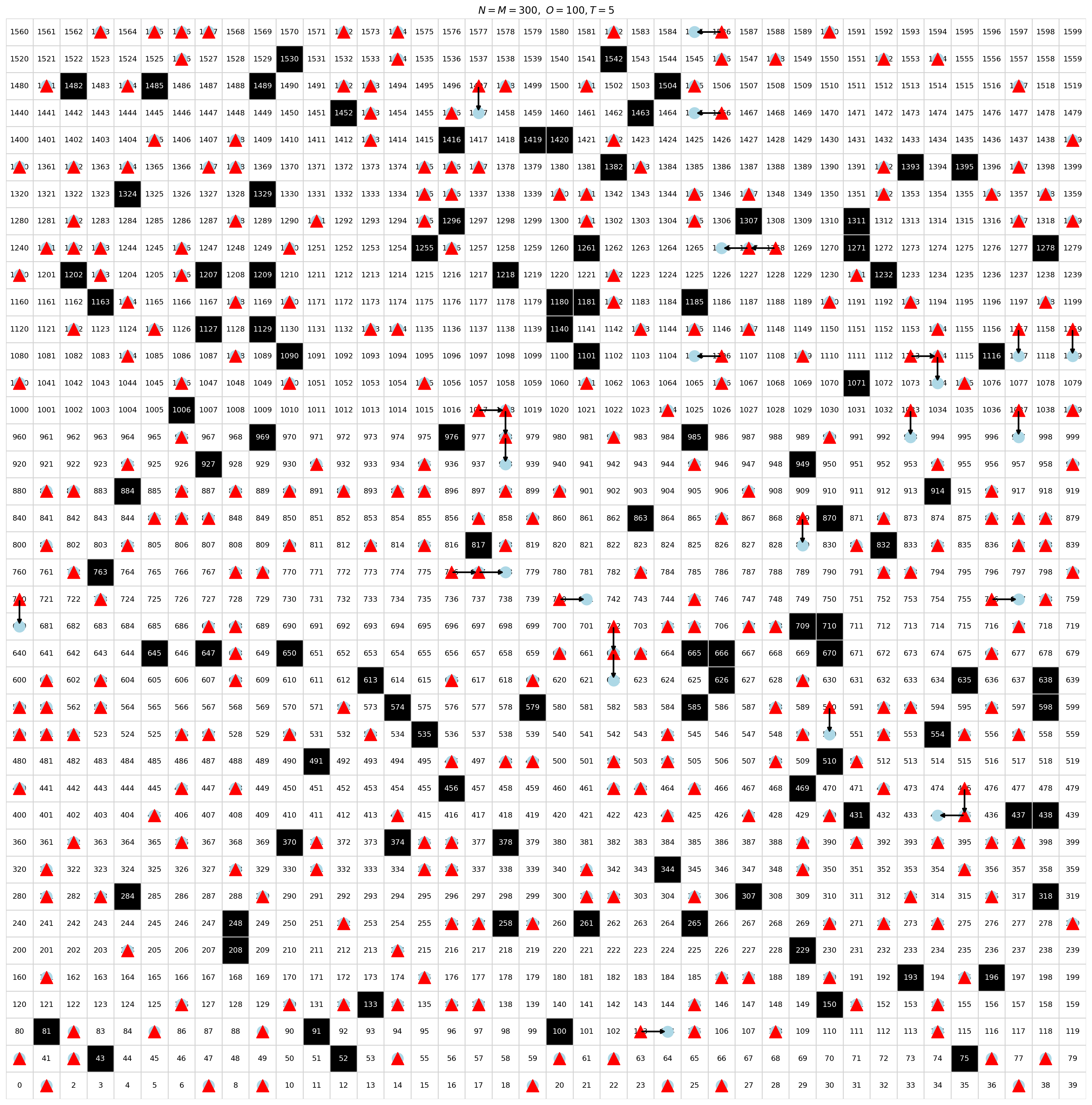}}
    \caption{A larger~${40\times 40}$ grid with~$300$ robots~$\robot$,~$300$ targets~$\target$, and~$100$ obstacles. The robot trajectories come from the min-cost transport obtained by solving~\textbf{P1} over~${T=6}$.}
    \label{a_fig2}
\end{figure}

\newpage
\begin{figure}[!htb]
    \centering
    {\includegraphics[width=0.4\textwidth]{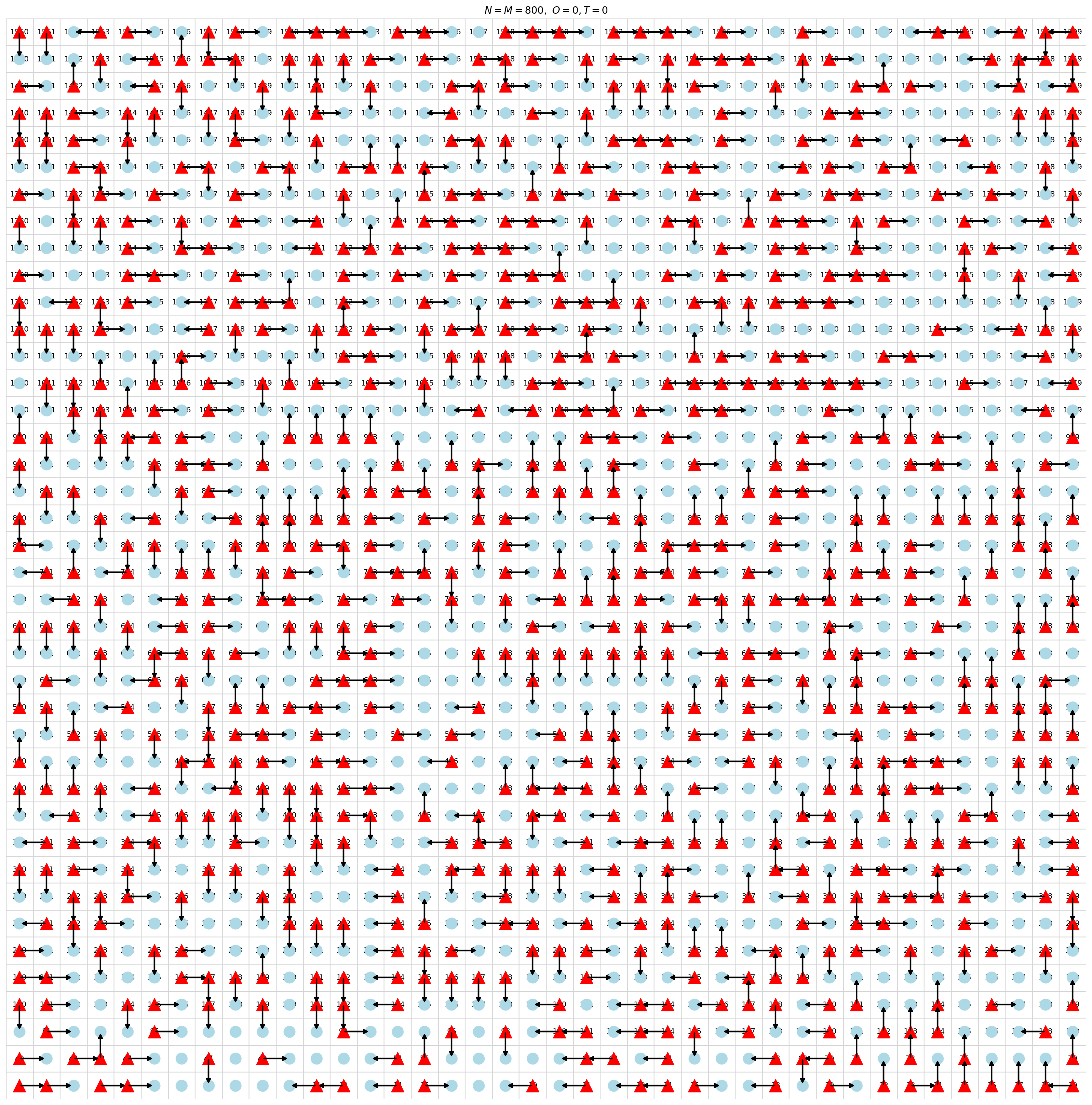}}
    \hspace{0cm}
    {\includegraphics[width=0.4\textwidth]{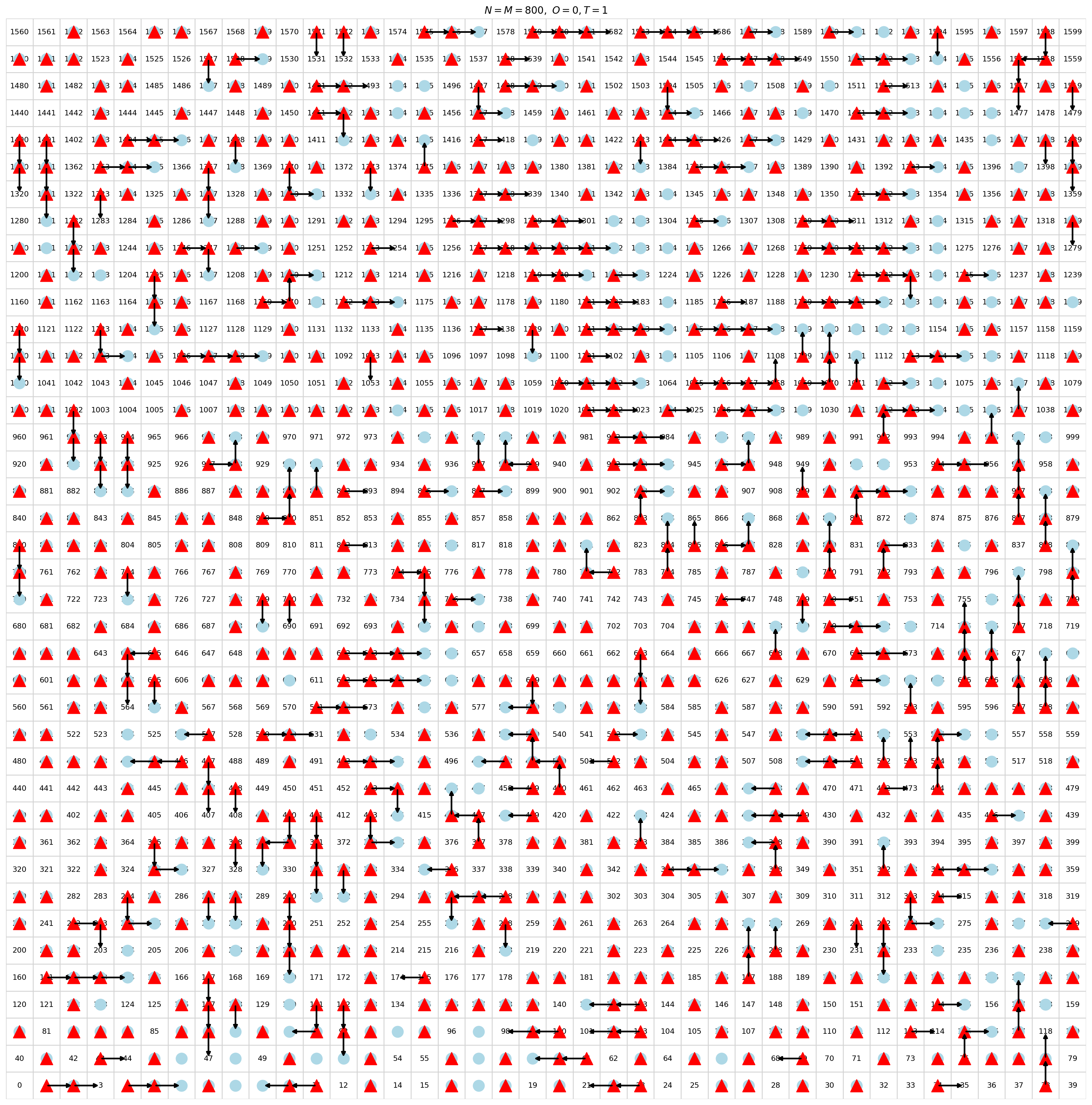}}
    \hspace{0cm}
    {\includegraphics[width=0.4\textwidth]{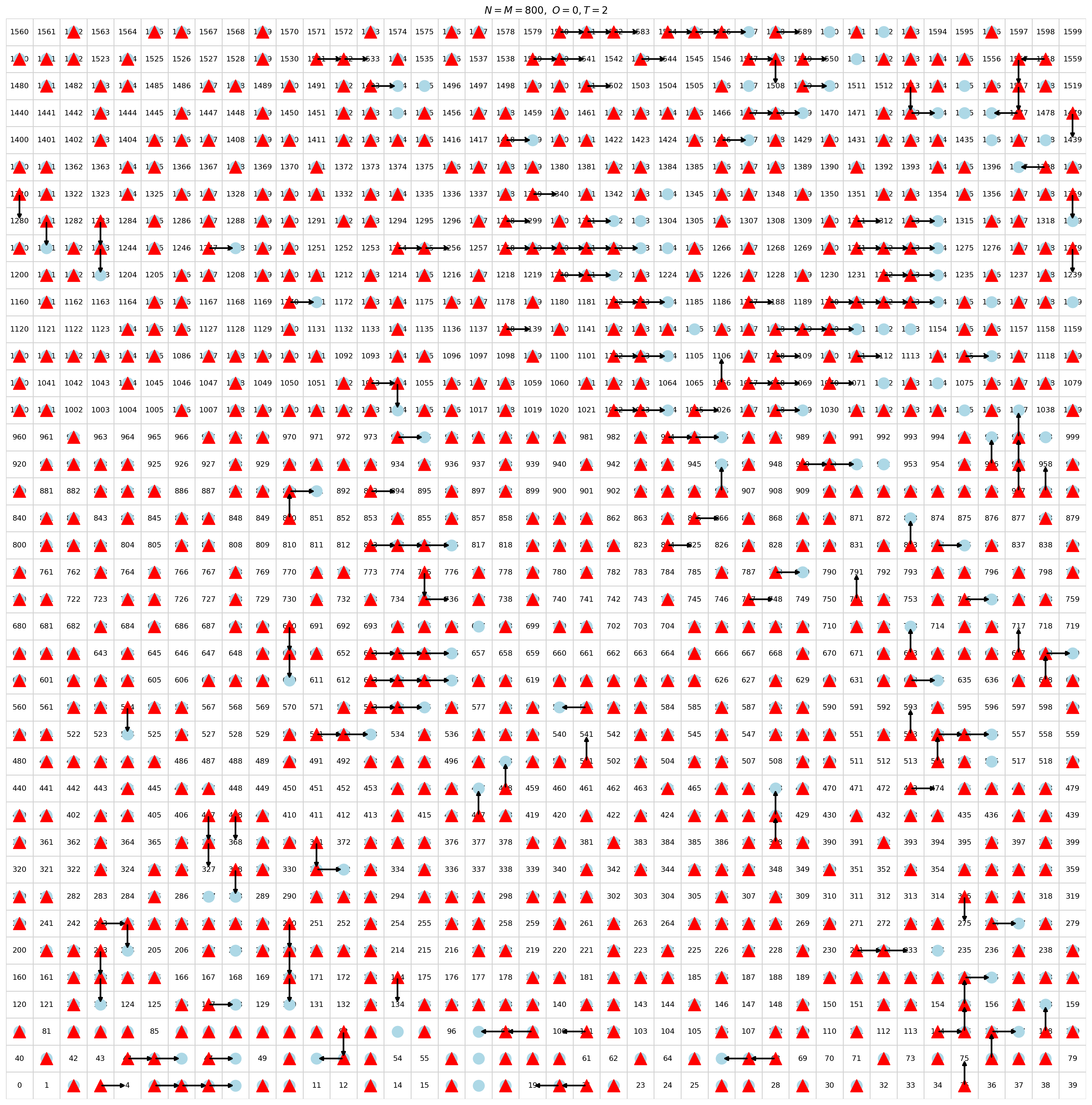}}
    \hspace{0cm}
    {\includegraphics[width=0.4\textwidth]{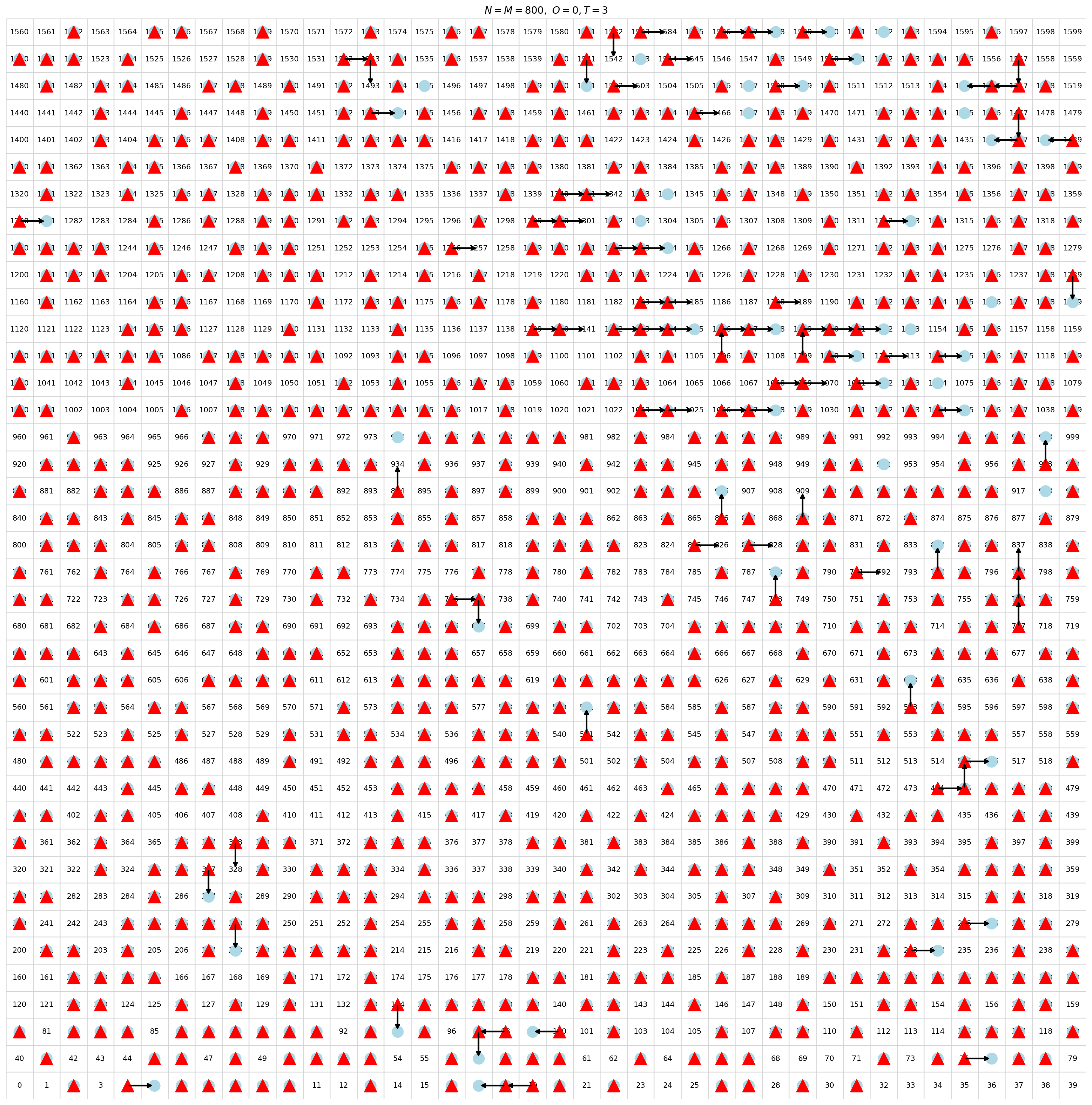}}
    {\includegraphics[width=0.4\textwidth]{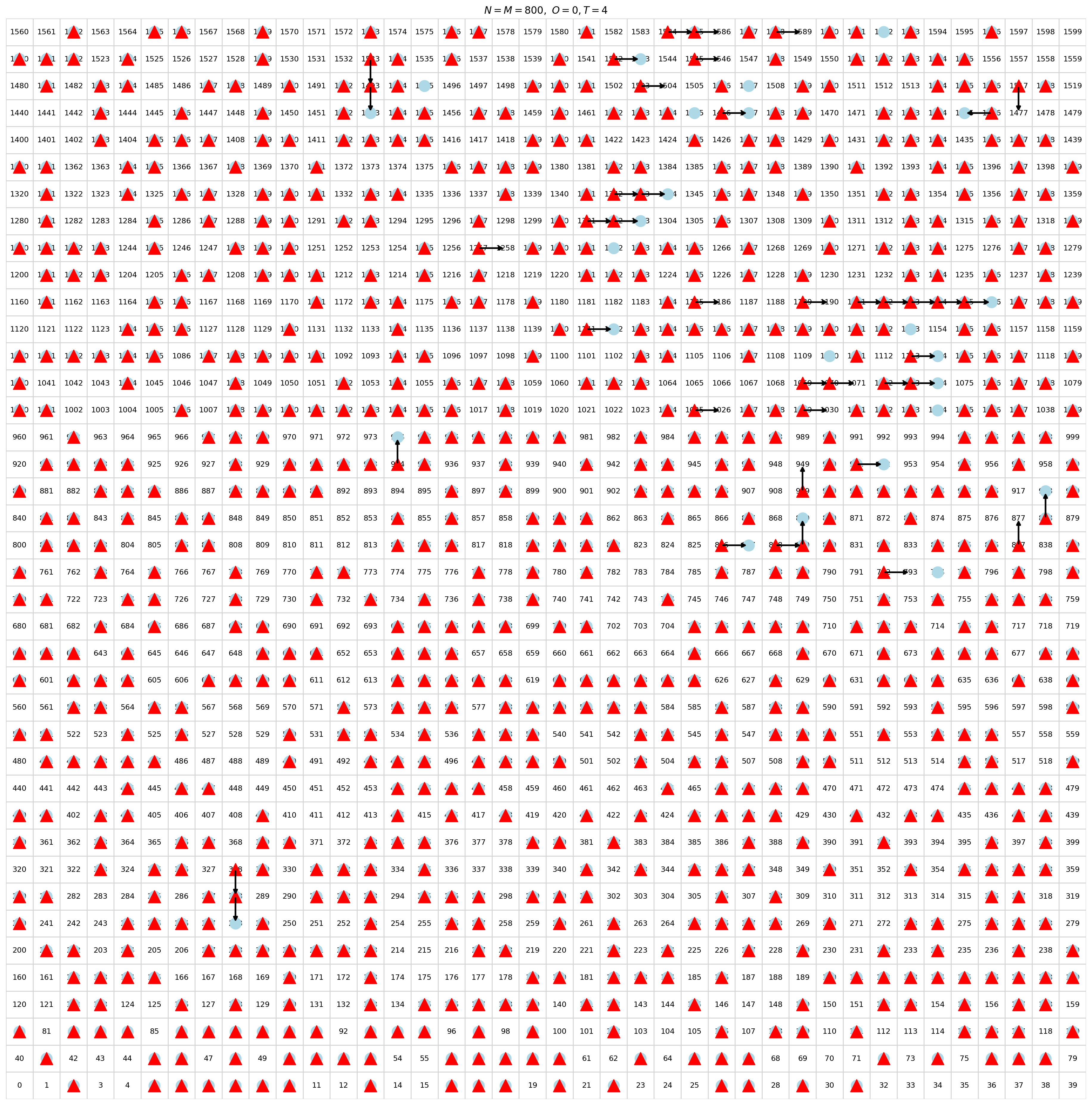}}
    \hspace{0cm}
    {\includegraphics[width=0.4\textwidth]{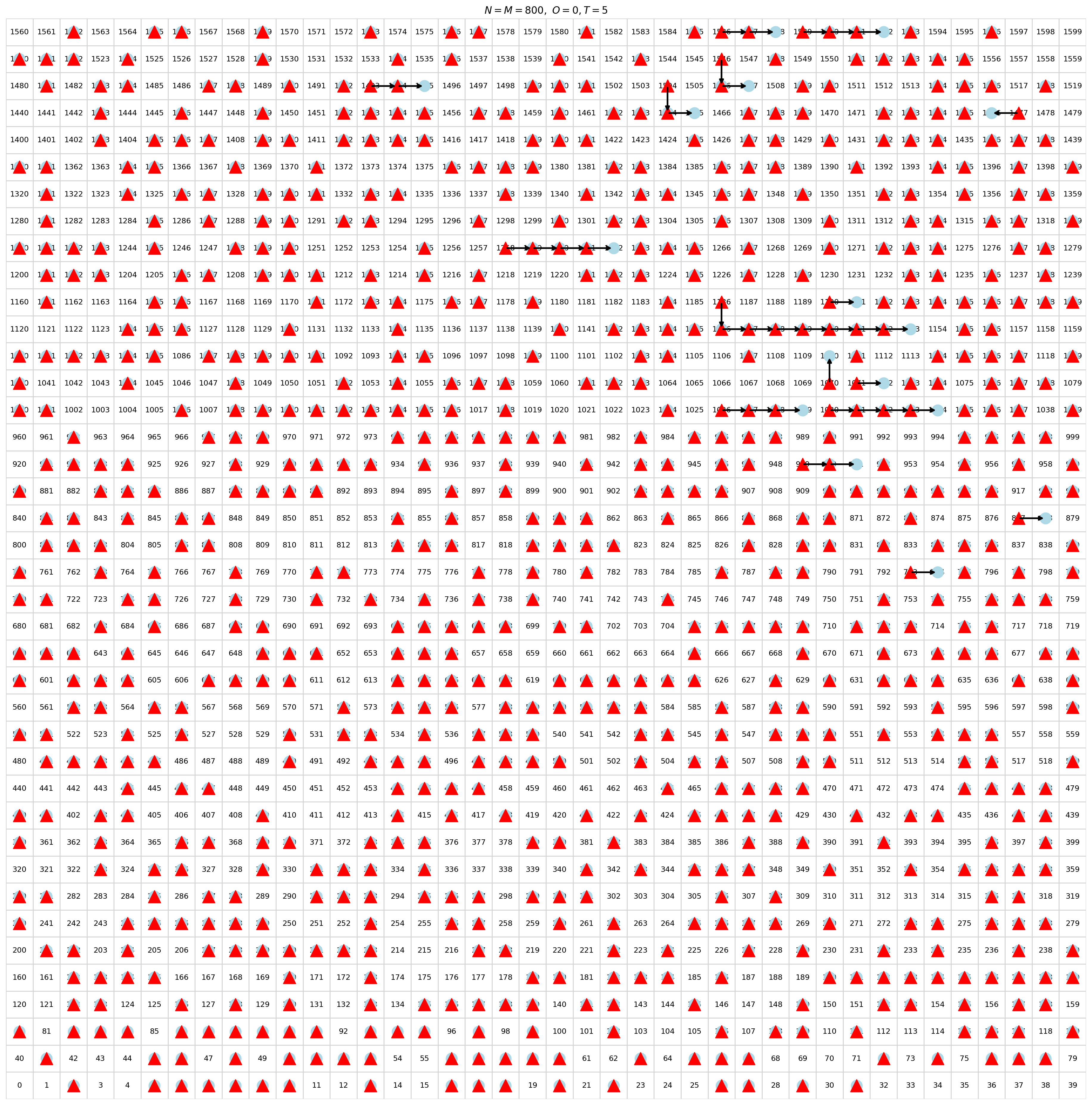}}
    \caption{A larger~${40\times 40}$ grid with~$800$ robots~$\robot$ and~$800$ targets~$\target$; every cell is either occupied by a robot or a target. The robot trajectories come from the min-cost transport obtained by solving~\textbf{P1} over~${T=6}$.}
    \label{a_fig3}
\end{figure}

\newpage
\begin{figure}[!htb]
    \centering
    \includegraphics[
        width=0.9\textwidth,
        trim=2in 0in 2in 0in,
        clip
    ]{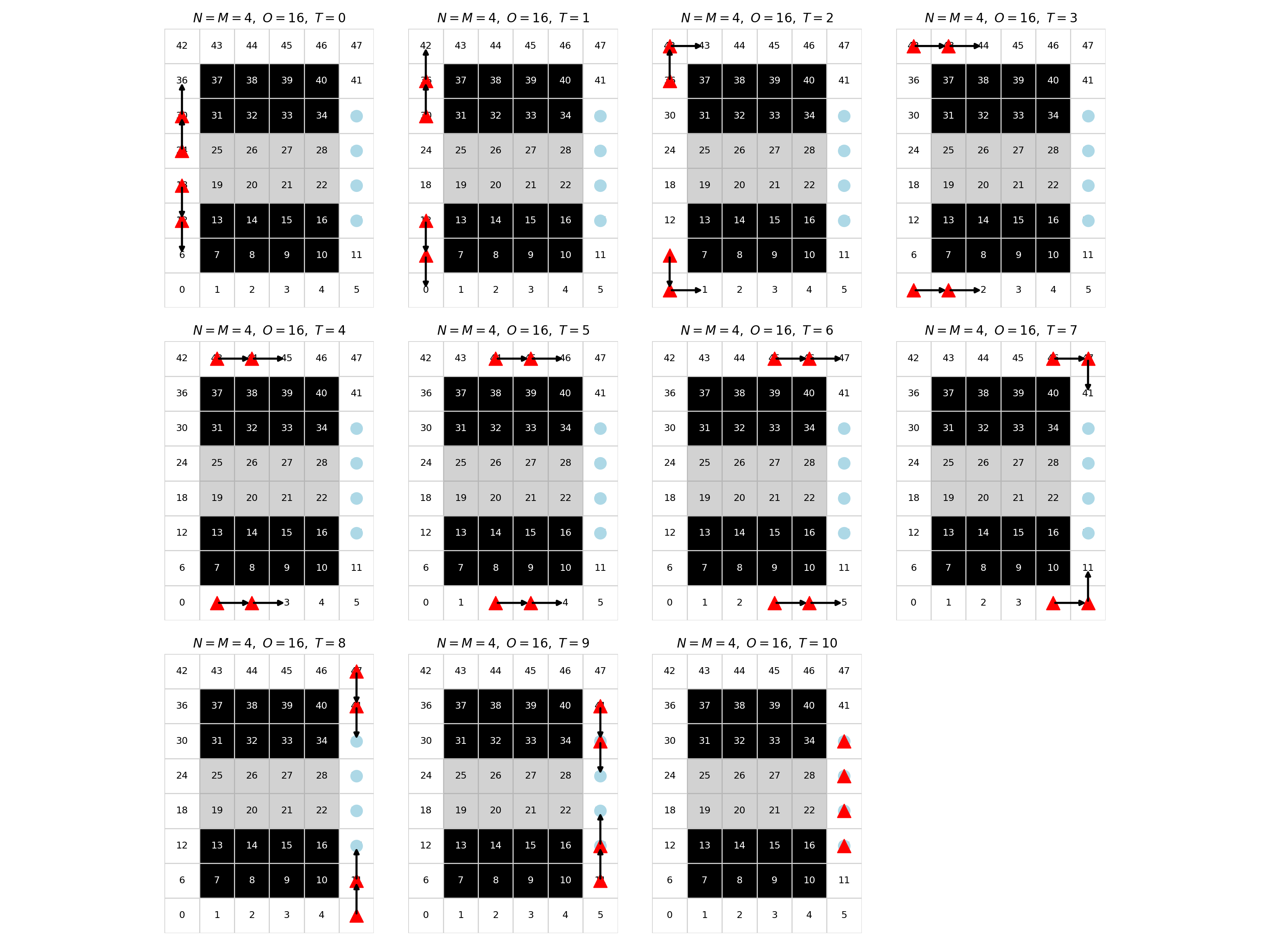}
    \caption{Min-cost vs. Min-makespan: A~${6\times 8}$ grid with~$4$ robots~$\robot$,~$4$ targets~$\target$, and~$16$ obstacles. Edges in the gray shaded region have cost~$10$; rest follow our move-wait cost convention. \\ With~${T=10}$, the min-cost transport avoids the high cost interior and takes the robots from the boundary of the grid, with all one-cost moves for a total of~$40$.}
    \label{ms_fig1}
\end{figure}
\newpage
\begin{figure}[!htb]
    \centering
    \includegraphics[
        width=0.9\textwidth,
        trim=2in 0in 2in 0in,
        clip
    ]{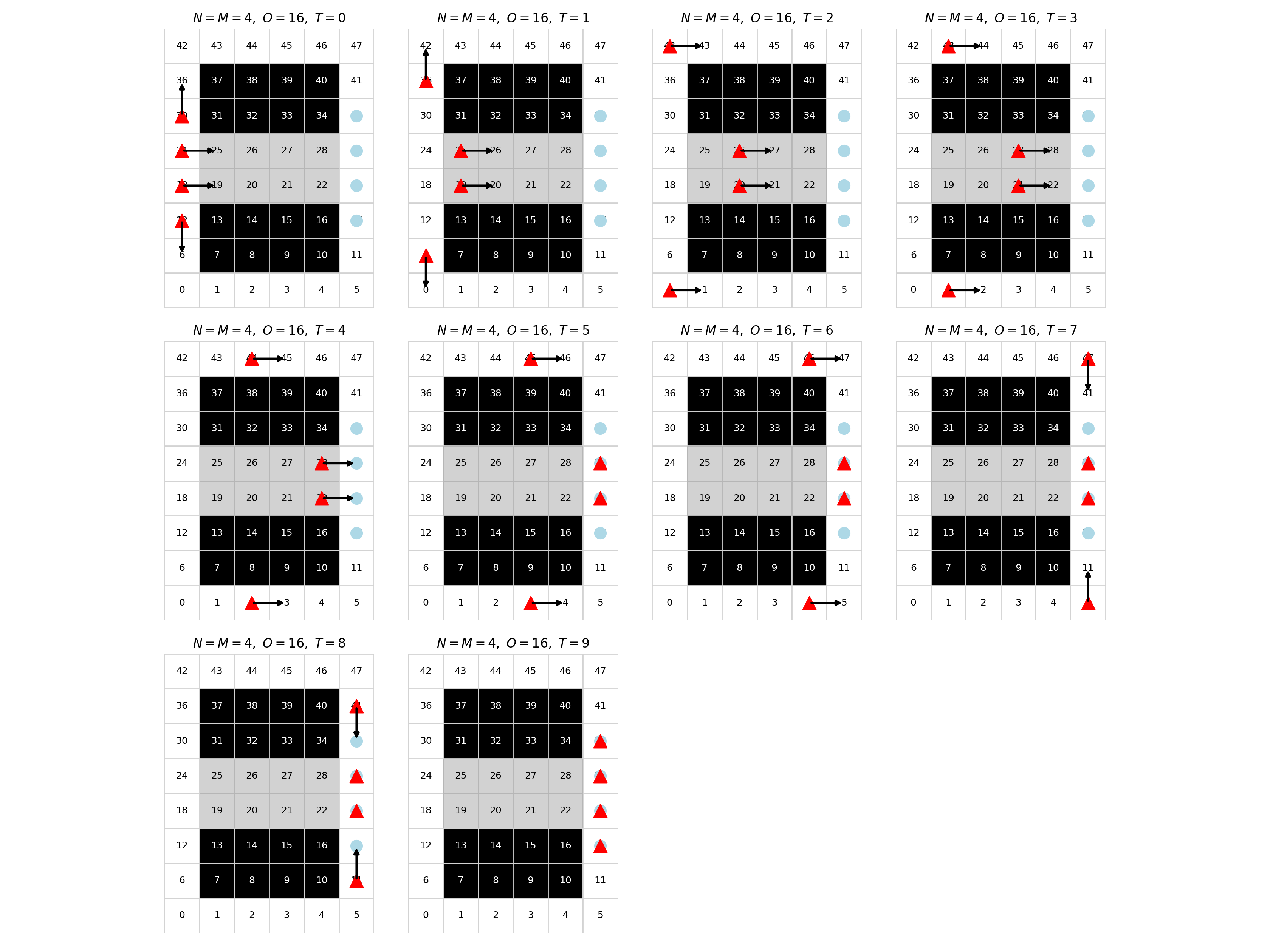}
    \caption{Min-cost vs. Min-makespan: A~${6\times 8}$ grid with~$4$ robots~$\robot$,~$4$ targets~$\target$, and~$16$ obstacles. Edges in the gray shaded region have cost~$10$; rest follow our move-wait cost convention.\\ With~${T=9}$, all robots cannot travel on the boundary as that requires~$10$ moves; the min-cost transport therefore takes two robots from the boundary in~$9$ steps each, and two from the higher cost interior edges for a total cost of~$82$.}
    \label{ms_fig2}
\end{figure}

\newpage
\begin{figure}[!htb]
    \centering
    \includegraphics[
        width=0.9\textwidth,
        trim=2in 0in 2in 0in,
        clip
    ]{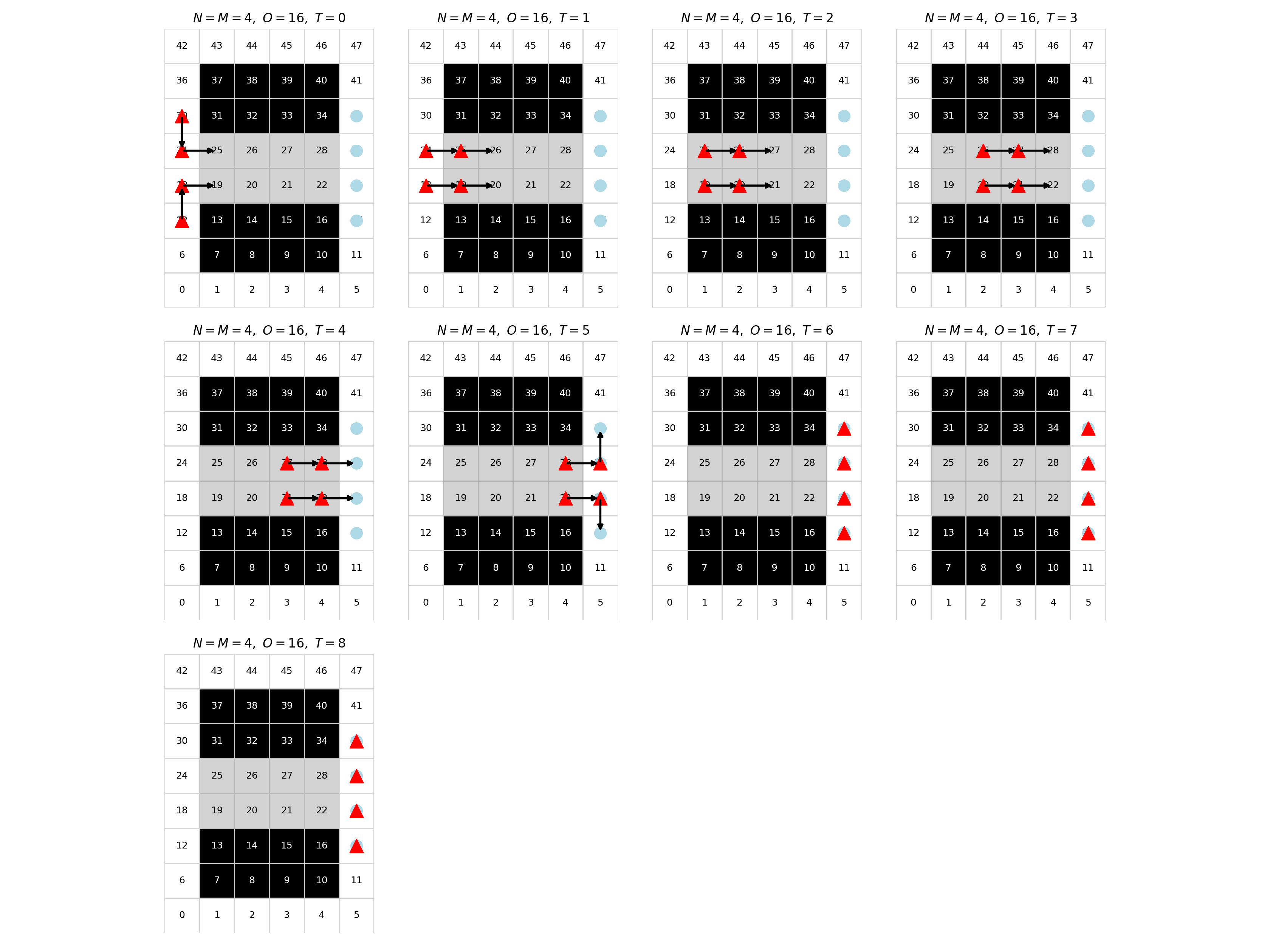}
    \caption{Min-cost vs. Min-makespan: A~${6\times 8}$ grid with~$4$ robots~$\robot$,~$4$ targets~$\target$, and~$16$ obstacles. Edges in the gray shaded region have cost~$10$; rest follow our move-wait cost convention. \\ With~${T=8}$, boundary paths are no longer feasible within the given time horizon. All robots must travel through the interior edges to reach the targets for a total cost of~$132$.}
    \label{ms_fig3}
\end{figure}
\newpage
\begin{figure}[!htb]
    \centering
    \includegraphics[
        width=0.9\textwidth,
        trim=2in 0in 2in 0in,
        clip
    ]{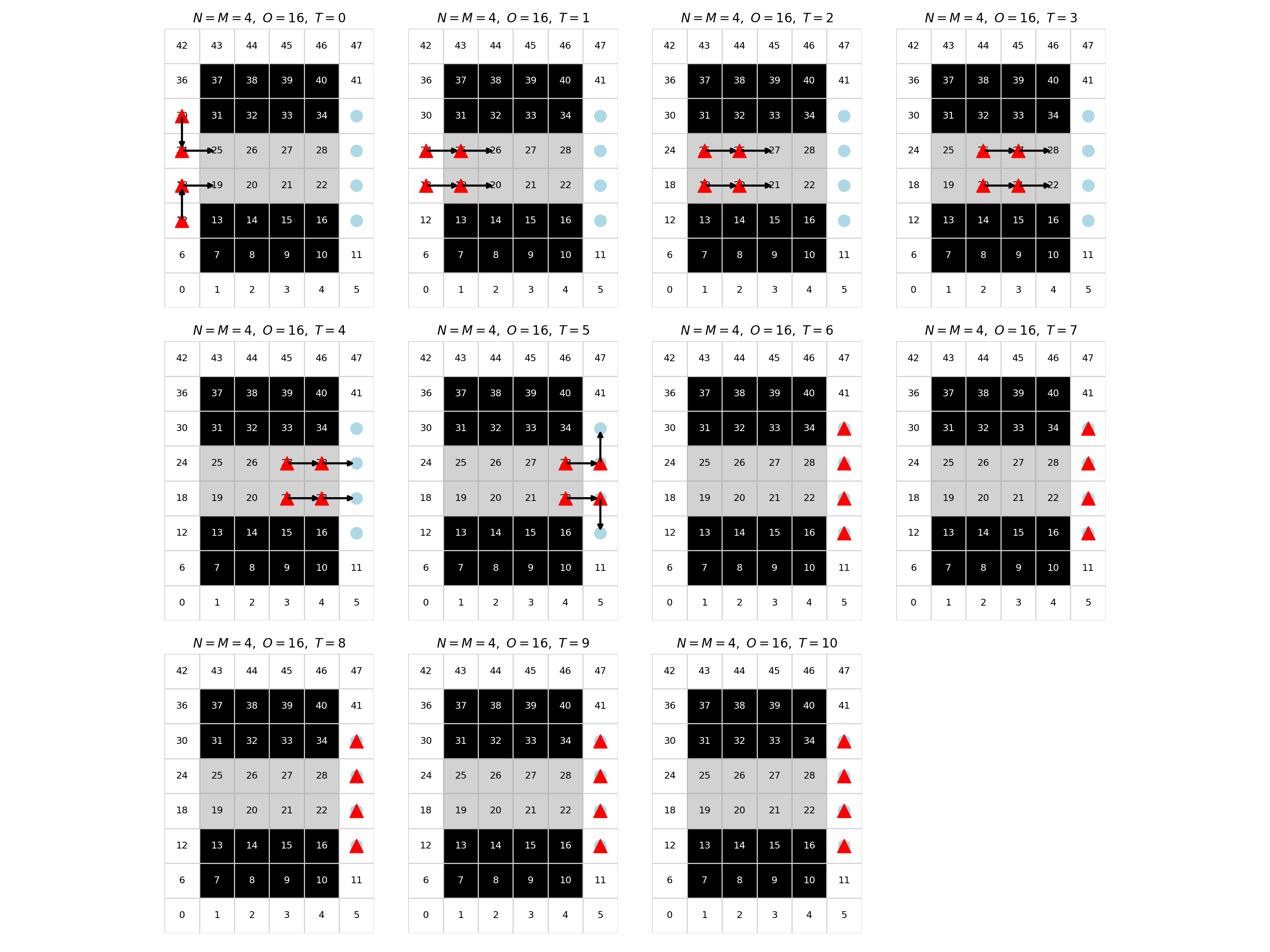}
    \caption{Min-cost vs. Min-makespan: A~${6\times 8}$ grid with~$4$ robots~$\robot$,~$4$ targets~$\target$, and~$16$ obstacles. Edges in the gray shaded region have cost~$10$; rest follow our move-wait cost convention. \\ We choose the cost structure described in Assumption~\ref{assump:time-sep}. \textbf{P1} consequently provides the minimum makespan solution that terminates the robot motion in~$5$ steps, i.e., achieves the minimum makespan, when solved over a longer~${T=10}$ horizon. The total cost of this transport increases exponentially, of the order of~$10^{13}$. For larger problems, the corresponding solvers may run into numerical instabilities and the search procedure described in Lemma~\ref{lem:T*} may be more tractable to compute the minimum makespan.}
    \label{ms_fig4}
\end{figure}

\newpage
\begin{figure}[!htb]
    \centering
    \includegraphics[width=0.75\textwidth]{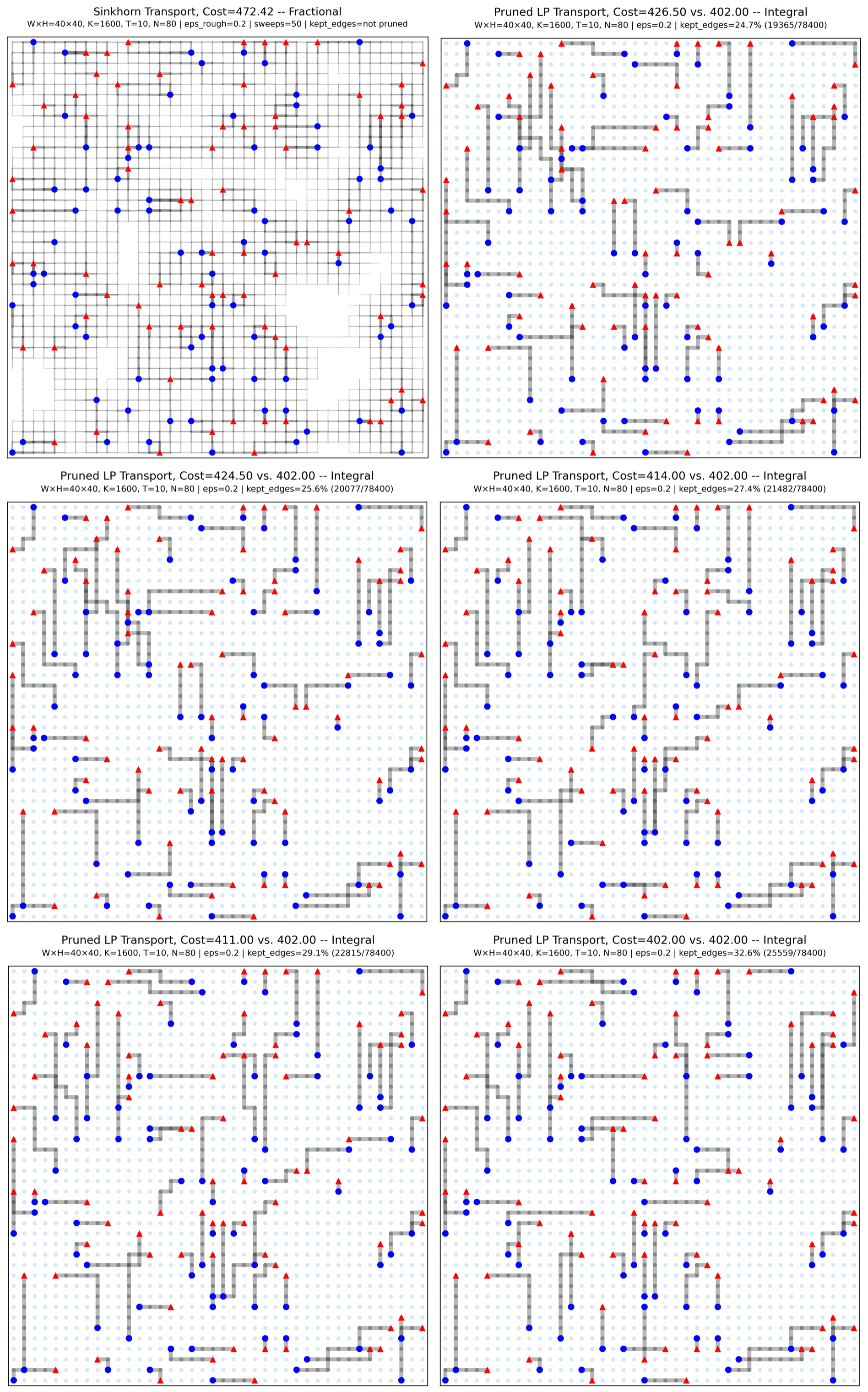}

    \caption{A~${40\times 40}$ grid with~$80$ robots~$\robot$,~$80$ targets~$\target$, and~${T=10}$. The robot paths are superimposed over the horizon~$T$. The first figure is the Schrödinger shadow that shows the likely mass transport obtained by solving~\textbf{P2} after~$\bar\tau$ Sinkhorn iterations in Appendix~\ref{sinkMAPF};~${{\varepsilon=0.2,\bar\tau=50}}$. The next figures show the integral projection~\textbf{P3} by keeping the highest-valued~$X\%$ of edges from the shadow. Cost is compared with the optimal in the figure title, and the optimal min-cost transport can be achieved with~$32.6\%$ edges. It is interesting to note the difference in the robot trajectories over these solutions.}
    \label{a_fig4}
\end{figure}

\newpage
\begin{figure}[!htb]
    \centering
    \includegraphics[width=0.75\textwidth]{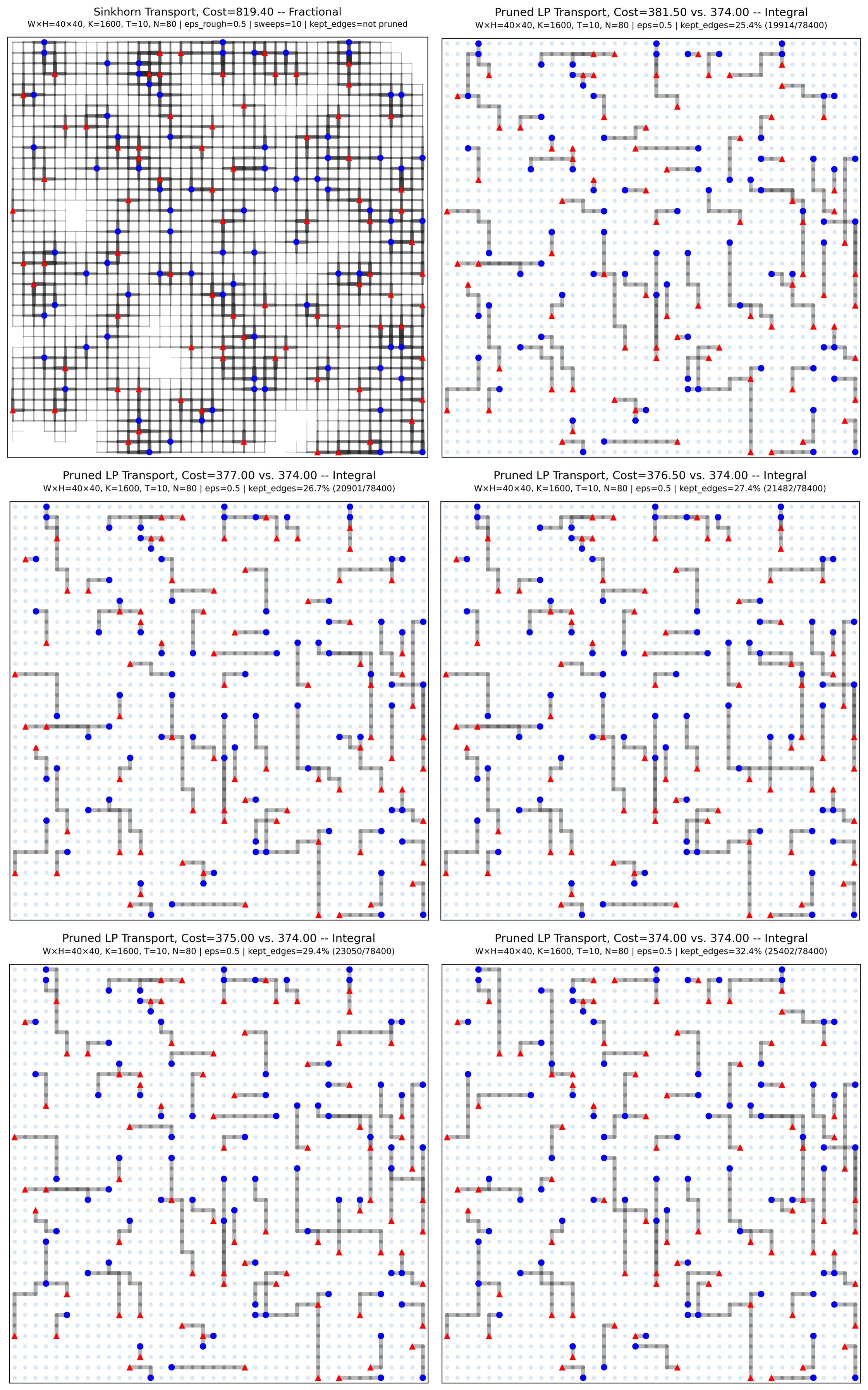}

    \caption{A~${40\times 40}$ grid with~$80$ robots~$\robot$,~$80$ targets~$\target$, and~${T=10}$. The robot paths are superimposed over the horizon~$T$. The first figure is the Schrödinger shadow that shows the likely mass transport obtained by solving~\textbf{P2} after~$\bar\tau$ Sinkhorn iterations in Appendix~\ref{sinkMAPF};~${{\varepsilon=0.5,\bar\tau=10}}$. The next figures show the integral projection~\textbf{P3} by keeping the highest-valued~$X\%$ of edges from the shadow. Cost is compared with the optimal in the figure title, which can be achieved with~$32.4\%$ edges.}
    \label{a_fig5}
\end{figure}

\newpage
\begin{figure}[!htb]
    \centering
    \includegraphics[width=0.75\textwidth]{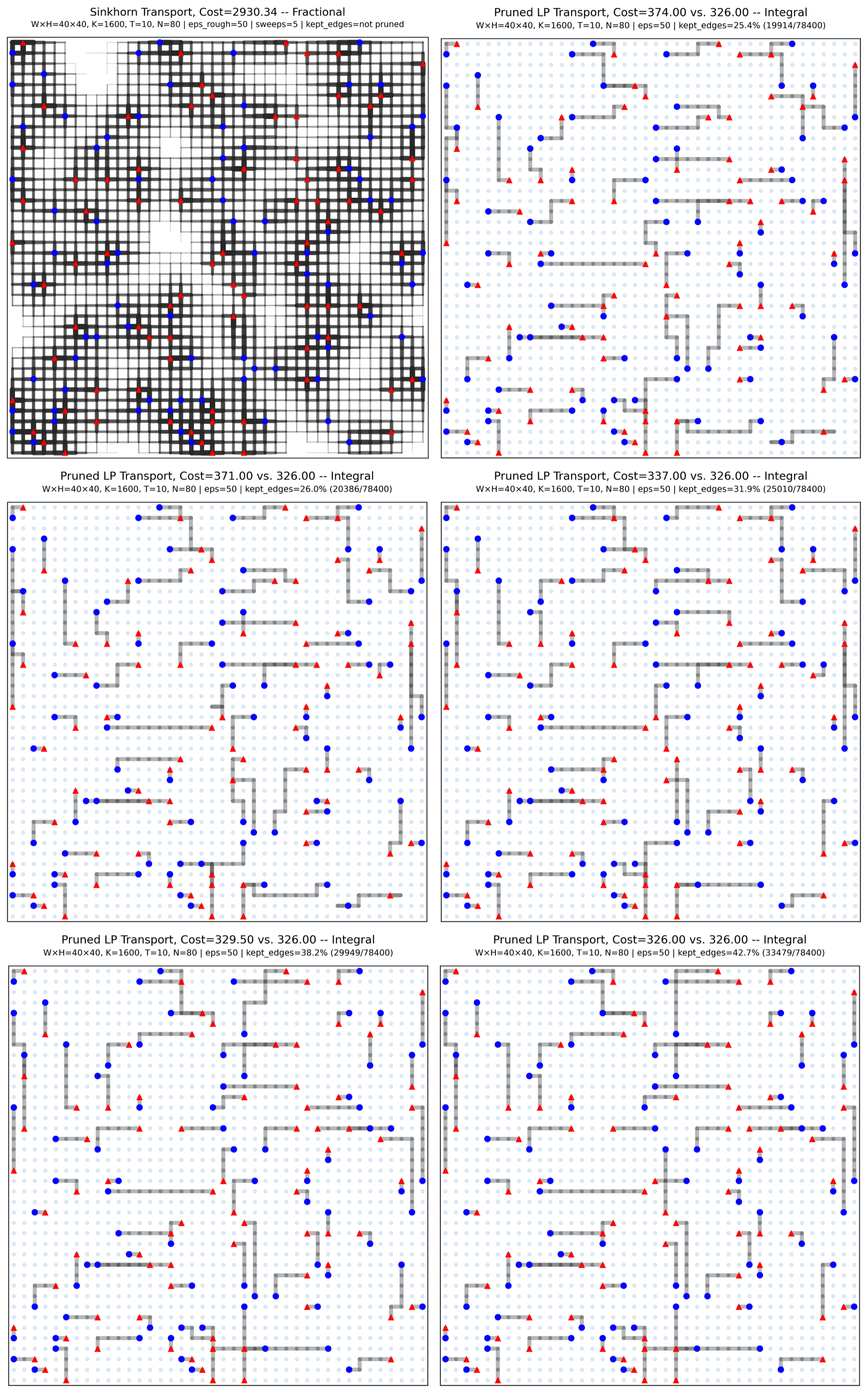}

    \caption{A~${40\times 40}$ grid with~$80$ robots~$\robot$,~$80$ targets~$\target$, and~${T=10}$. The robot paths are superimposed over the horizon~$T$. The first figure is the Schrödinger shadow that shows the likely mass transport obtained by solving~\textbf{P2} after~$\bar\tau$ Sinkhorn iterations in Appendix~\ref{sinkMAPF};~${{\varepsilon=50,\bar\tau=5}}$. The next figures show the integral projection~\textbf{P3} by keeping the highest-valued~$X\%$ of edges from the shadow. Cost is compared with the optimal in the figure title, which can be achieved with~$42.7\%$ edges.}
    \label{a_fig6}
\end{figure}

\end{document}